\documentclass[11pt,final]{amsart}
\RequirePackage{amsthm,amsmath,amsfonts,amssymb}
\RequirePackage[numbers]{natbib}
\RequirePackage[colorlinks,citecolor=blue,urlcolor=blue]{hyperref}
\RequirePackage{graphicx}

\usepackage{algorithm}
\usepackage[noend]{algpseudocode}
\usepackage{epsfig, epsf, graphicx, float}
\usepackage{pstricks, psfrag, wrapfig, latexsym, pst-solides3d}
\usepackage{amssymb, amsmath, amsfonts, exscale,mathtools,amsthm}
\usepackage{verbatim,enumerate,hyperref}
\usepackage{mathtools}
\usepackage{longtable, booktabs,paralist}
\usepackage{steinmetz}
\usepackage[shortlabels]{enumitem}
\usepackage{smartdiagram}
\usepackage{pgfgantt}
\usepackage{arydshln,subfigure,url}
\usepackage{pgfplots}
\usepackage[margin=1in]{geometry}
\usepackage[foot]{amsaddr}

\usepackage{tikz-qtree,tikz-qtree-compat}
\usepackage{tikz,pgf,hyperref}
\usetikzlibrary{positioning,patterns}
\usetikzlibrary{calc,fadings,decorations.pathreplacing,arrows,
decorations.markings,
datavisualization.formats.functions,shapes.geometric}
\usetikzlibrary{hobby}
\usesmartdiagramlibrary{additions}

\def\boxit#1{\vbox{\hrule\hbox{\vrule\kern6pt
          \vbox{\kern6pt#1\kern6pt}\kern6pt\vrule}\hrule}}

\def\bse{\begin{eqnarray*}}
\def\ese{\end{eqnarray*}}
\def\be{\begin{eqnarray}}
\def\ee{\end{eqnarray}}
\def\bq{\begin{equation}}
\def\eq{\end{equation}}
\def\bse{\begin{eqnarray*}}
\def\ese{\end{eqnarray*}}

\def\T{^{\rm T}}

\def\imagetop#1{\vtop{\null\hbox{#1}}}

\newcommand{\corb}[1]{\textcolor{black}{#1}}
\newcommand{\corbt}[1]{\textcolor{black}{#1}}

\newcommand{\bbN}{\mathbb{N}}

\newcommand{\bbS}{\mathbb{S}}

\newcommand{\bbP}{\mathbb{P}}

\newcommand{\bD}{\mathbf{D}}

\newcommand{\bW}{\mathbf{W}}

\newcommand{\bU}{\mathbf{U}}

\newcommand{\bA}{\mathbf{A}}
\newcommand{\bB}{\mathbf{B}}

\newcommand{\bF}{\mathbf{F}}

\newcommand{\bM}{\mathbf{M}}
\newcommand{\bN}{\mathbf{N}}

\newcommand{\bS}{\mathbf{S}}
\newcommand{\bT}{\mathbf{T}}
\newcommand{\bV}{\mathbf{V}}

\newcommand{\bx}{\mathbf{x}}
\newcommand{\by}{\mathbf{y}}
\newcommand{\bv}{\mathbf{v}}

\newcommand{\bu}{\mathbf{u}}
\newcommand{\bw}{\mathbf{w}}

\newcommand{\bpsi}{\boldsymbol{\psi}}
\newcommand{\bphi}{\boldsymbol{\phi}}
\newcommand{\bchi}{\boldsymbol{\chi}}

\newcommand{\bSigma}{\boldsymbol{\Sigma}}

\newcommand{\0}{\mathbf{0}}

\newcommand{\mcB}{{\mathcal B}}
\newcommand{\mcC}{{\mathcal C}}
\newcommand{\mcD}{{\mathcal D}}
\newcommand{\mcE}{{\mathcal E}}
\newcommand{\mcF}{\mathcal{F}}

\newcommand{\mcM}{{\mathcal M}}

\newcommand{\mcO}{{\mathcal O}}
\newcommand{\mcP}{{\mathcal P}}

\newcommand{\mcS}{\mathcal{S}}
\newcommand{\mcT}{{\mathcal T}}

\newcommand{\oone}{\bold{1}}

\newcommand{\nset}{{\mathbb N}}

\newcommand{\eset}[1]{{\mathbb E} \left[ #1 \right] }

\newcommand{\ip}[2]{\langle #1,#2\rangle}
\newcommand{\Union}{\bigcup}
\newcommand{\R}{\mathbb{R}}
\newcommand{\spanof}{\operatorname{span}}
\newcommand{\rank}{\operatorname{rank}}

\newcommand{\Cov}{{\rm Cov}}

\def\R{\Bbb{R}}

\definecolor{darkgreen}{rgb}{0, 0.6, 0}
\definecolor{airforceblue}{rgb}{0.36, 0.54, 0.66}
\definecolor{applegreen}{rgb}{0.55, 0.71, 0.0}
\definecolor{asparagus}{rgb}{0.53, 0.66, 0.42}
\definecolor{cadetblue}{rgb}{0.37, 0.62, 0.63}
\definecolor{cambridgeblue}{rgb}{0.64, 0.76, 0.68}
\definecolor{olivine}{rgb}{0.6, 0.73, 0.45}
\definecolor{rufous}{rgb}{0.66, 0.11, 0.03}
\definecolor{sangria}{rgb}{0.57, 0.0, 0.04}
\definecolor{neworange}{rgb}{1, 0.64, 0}
\definecolor{flowblue}{rgb}{0.4471,    0.6235,    0.8118}
\definecolor{babyblue}{RGB}{153,204,255}
\definecolor{tan}{RGB}{210,180,140}

\tikzstyle{block} = [draw, shade, drop shadow, rounded corners=1ex,
top color = babyblue!100!white,
bottom color = babyblue!100!white,
minimum width=2.5cm]

\tikzstyle{newblock} = [draw, shade, drop shadow,rounded
  corners=1ex,top color=orange!70!, minimum width=4cm] 
  
\tikzstyle{obs} =
          [draw, shade, drop shadow,rounded corners=1ex,top color=orange!70,minimum width=2em]
          
          \tikzstyle{proc} = [draw,rectangle,rounded
            corners=1ex,fill=green!20!blue!60!, minimum width=2em]
          
          \tikzstyle{emptyblock} = [draw,minimum width=2em]

          \tikzstyle{branch}=[fill,shape=circle,minimum size=3pt,inner
            sep=0pt] 
            
            \tikzstyle{vecArrow} = [thick, blue,
            decoration={markings,mark=at position 1 with
              {\arrow[semithick,blue]{triangle 60}}}, double
            distance=1.4pt, shorten >= 5.5pt, preaction = {decorate},
            postaction = {draw,line width=1.4pt, blue,shorten >=
              4.5pt}]

\makeatletter
\def\thm@space@setup{%
  \thm@preskip=0\topsep \thm@postskip=\thm@preskip
}
\makeatother

\theoremstyle{plain}

\newtheorem{theorem}{Theorem}[section]
\newtheorem{lemma}[theorem]{Lemma}
\newtheorem{prop}[theorem]{Proposition}

\theoremstyle{remark}
\newtheorem{definition}[theorem]{Definition}
\newtheorem{asum}[theorem]{Assumption}

\newtheorem*{remark}{Remark}
\begin{document}

\title[Stochastic Analysis Vector Field Anomaly Detection]{Distribution-Free Stochastic Analysis and Robust Multilevel Vector Field Anomaly Detection}

\author{Julio E. Castrill\'{o}n-Cand\'{a}s$^{\ddagger}$, Michael Rosenbaum$^{\ddagger}$, Mark Kon$^{\ddagger}$}
\email{jcandas@bu.edu, mrbaum@bu.edu, mkon@bu.edu}

 \address{
   ${\ddagger}$ Department of Mathematics and Statistics, 
  Boston University, Boston, MA 
  }

\maketitle

\begin{abstract}
 Massive vector field datasets are common in multi-spectral optical and radar sensors, among many other emerging areas of application. We develop a novel stochastic functional (data) analysis approach for detecting anomalies based on the covariance structure of nominal stochastic behavior across a domain. An optimal vector field Karhunen-Lo\`{e}ve expansion is applied to such random field data. A series of multilevel orthogonal functional subspaces is constructed from the geometry of the domain, adapted from the KL expansion. Detection is achieved by examining the projection of the random field on the multilevel basis. \emph{A critical feature of this approach is that reliable hypothesis tests are formed, which do not require prior assumptions on probability distributions of the data.} The method is applied to the important problem of degradation in the Amazon forest. Due to the complexity and high dimensionality of satellite imagery, it is not feasible to assume known distributions, nor to estimate them. In addition to providing reliable hypothesis tests, our approach shows the advantage of using multiple bands of data in a vectorized complex, leading to better anomaly detection. Furthermore, using simulated data, our approach is capable of detecting subtle anomalies that are impossible to detect with PCA-based methods.
\end{abstract}

\section{Introduction}
%The term “human dynamics” is used in the social sciences to refer to
%social interactions, population changes, international and domestic
%migrations, human relationships, human mobility, human–computer
%interactions, socioeconomic activities, and many other variants
%pertinent to human activities generating changes
%\cite{yuan2018human}. Increasingly as well, human dynamics research is
%also incorporating human activities and interactions taking place in
%virtual environments \cite{chen2017delineating}. Human dynamics
%constantly impacts and is impacted by the physical, social, economic,
%cultural, and political, as well as virtual environments.

The development of ever more massive datasets, concurrent with
advances in artificial intelligence and machine learning, is
transforming many aspects of society in extensive ways. Remote sensing
and GIS data over various temporal and spatial resolutions provide
foundational data for addressing issues within many of the facets of
human dynamics \cite{shaw2018introduction}. Increasingly, Wi-Fi and
GPS tracking via cell phones enable us to gather data at high
spatio-temporal resolutions at low cost, providing real-time solutions
for dynamic traffic management and accident prevention
\cite{petrovska2015traffic, guo2020gps}.  Geotagged social media and
direct locational information have provided ways of classifying
functional characteristics of urban locations
\cite{chen2017delineating, hasan2013spatiotemporal}. Analysis of
information from daily mobility patterns
\cite{barboza2020identifying}, IoT sensors \cite{paul2016smartbuddy},
satellites \cite{rotem2015role} and drones \cite{zhang2019predicting},
provide extensive data for systematic study of human dynamics in
migration, disease outbreaks \cite{peng2020exploring}, and threat
outbreaks. Many methods have been developed to detect and identify
anomalies. In particular, for syndromic surveillance and signal
processing the scan statistics approach has been widely used
\cite{Arias-Castro2018,Kulldorff1997,Cheung2013,Guerriero2009,
  Neill2004,Arias-Castro2005,Walther2010,Neill2012}. In addition, 
  PCA based methods have been developed for anomaly detection in the 
  context of network traffic \cite{Lakhina2004}.

    \emph{Due to the complexity and high dimensionality of many modern
    datasets, including satellite imagery, it is not feasible to
    assume known parametric representations of distributions, nor is
    it feasible to estimate them.} Most known
    probabilistic/statistical methods require prior or parametric
    knowledge of the distribution.  This makes them unsuitable for
    such datasets, as they can lead to erroneous conclusions. This is
    a main weakness of many current probabilistic/statistical
    methods. There is a need for a novel probabilistic mathematical
    theory that can tackle this problem.  We introduce a new
    perspective based on singular value decompositions of random
    fields and stochastic processes on tensor product Bochner
    spaces. A critical feature of this approach is that reliable
    hypothesis tests can be formed which do not require prior
    assumptions on probability distributions. Only a good estimate for
    data covariance is needed, making for a significantly simpler
    problem.

In this paper we develop a framework to detect anomalies in random
vector fields based on stochastic functional analysis. This approach
was recently introduced in \cite{Castrillon2022} for scalar
data. However, the detection theory is preliminary and no applications
are shown except for a few examples.  In our current paper the
approach is based on optimal vector field Karhuen-Lo\`{e}ve (KL)
tensor product expansions, and the construction of vector field
multilevel functional spaces for the detection of anomalies.  We show
that this method is well suited for vector field data over complex
geometrical domains (or networks) arising from the measurement of
different modalities from the same objects. Applications include satellite data
with multiple spectral bands. Using KL expansions, this approach
allows detection within large classes of random vector fields. The
nested multilevel spaces are natively adapted to tensor product
expansions. Their construction is elaborate and has been used in the
context of solving Partial Differential Equations (PDEs)
(\cite{Castrillon2003,D'Heedene2005,Tausch2001}).

The application of vector field KL expansions to Functional Data
Analysis (FDA) is almost non-existent and not properly understood,
despite its expansive potential application in multiple fields in extensively integrating multimodal information. Furthermore,
applications appear to be restricted to simple closed interval domains
that correspond to temporal data. The expansion from temporal to
spatio and spatio-temporal and further to extensive multimodal data
speaks to the integration that will be needed.  For example, in the
recent paper \cite{Gao2019} the authors present a method to compress
high dimensional temporal data in the form of a vector field. However,
each entry in the vector field is expanded separately using KL. This
is suboptimal as a compact optimal representation can be obtained by
using the vector field KL expansion. In another recent paper
\cite{Happ2018} a proof of the vector field KL expansion can be found.
Although there is a comment on optimality, it is not proven and the
reference they provide is for PCA and not for KL expansions. Our
approach has the following main contributions:

\begin{itemize}%[leftmargin=*]

\item A detailed and rigorous proof of the existence and optimality of
  the vector field KL expansion on general domains. This proof is
  based on tensor product theory from functional analysis and the
  results developed in \cite{Schwab2006}.

\item Detection of anomalous global and local signals described as
  scalar or vectorial random fields on general domains.

\item \emph{A critical contribution of this paper is the construction
of reliable hypothesis tests that do not require assumptions on the
data distribution, but only the covariance structure, a significantly
easier problem.} This is a fully non-parametric probabilistic
  framework, in particular without any Gaussian or other
  distributional assumptions of the data.  This is in stark contrast
  to traditional hypothesis testing where a parametric model of the
  data is assumed. This can lead to erroneous answers if the
  distribution of the data is not close to the assumed model.

\item The distribution free hypothesis test approach leads to a novel approach of detection. Our numerical results show that very subtle anomalies can be detected, where other methods, such as PCA based \cite{Lakhina2004} are not able to.

\item The method is demonstrated in an important remote sensing
  application with Sentinel-2 satellite data to detect degradation in the Amazon forest.

\end{itemize}

\noindent In addition to having the following features:
\begin{inparaenum}[(1)]
     \item Stochastic fusion of the anomalies in multimodal vector
       field data without any loss of information.  This represents an
       optimal fusion of the scalar components of the vectorized data.
           \item Quantification of magnitudes of the anomalies defined on a
      suitable Bochner normed space.
    \item Adaptability of vector field signal domains to be defined on
      complex topologies. This includes geospatial, spatio-temporal,
      manifold and network topologies, among others.      
    \item Multilevel filters can process large quantities
      of data with near-optimal performance.
    %\item Augmentation of current statistical methods such as change
    %  point detection
     % \cite{Aue2013,Jandyala2013,Page1954,Hinkley1971,Moustakides1986,Lai1995,Chu1996, Horvath2004,Fremdt2014,Aue2012,Wied2013,Kirch2018,Pape2016,Dette2020,Shao2010,Shao2015,Zhang2018}.
    \item Applicability of the theory and code to existing Machine
      Learning problems, leading to significant increases to accuracy
      in their solutions (See \cite{Castrillon2025}). 
\end{inparaenum} 
      
      Note that in \cite{Lakhina2004} the authors develop an approach for
detecting network anomalies based on residual subspaces of a Principal
Component Analysis (PCA). This approach is related to the methods
developed in this paper. However, as shown in section \ref{section:performance}, our approach is significantly more robust for anomaly detection.

We introduce a novel application of this framework in the context of a
major environmental problem, deforestation in the Amazon rain forest.
Proper detection of deforestation and forest degradation events can
lead to quantification and potential mitigation of resulting effects
on climate. Climate change is understood to be driven by unmitigated
anthropogenic emissions of greenhouse gases into the atmosphere,
primarily of carbon dioxide from combustion of fossil fuels and
deforestation. 
%Effects of climate change are being witnessed now, are
%believed to be irreversible on the timescale of human lifespan, and
%are projected to increase in the decades to come
%\cite{USGCRP2017}. 
Additionally, deforestation has been identified as
a factor of greater significance than climate change alone
\cite{skole1994}, and as both the single most important variable
affecting ecological systems \cite{chapin2000} and the most
significant threat to biodiversity \cite{sala2000}.

Satellite remote sensing is often the only viable means for gathering
deforestation information, but numerous problems hinder effective
collection of information.  Among the technical problems are
persistent and unpredictable cloud cover in deforestation hot-spots,
which complicates automated processing of very accessible optical
satellite data.

Our mathematical framework is well suited to detecting changes on land
surfaces in optical modalities.  An application of the multilevel
anomaly filter is applied to such data (e.g., collected by the Sentinel-2
satellite) that  record Amazon forest degradation.
%is illustrated in Figures \ref{PR:Fig2} and \ref{PR:Fig3}. 
\corbt{This application will be analyzed in detail in Section
  \ref{section:applications}.  In Figure \ref{PR:Fig1} the general
  pipeline for anomaly detection using scalar satellite or other
  datasets is demonstrated. Here the optical training dataset is used
  to construct the covariance eigenstructure. These data consist of
  measurements from the Sentinel-2 (\cite{Drusch2012}) optical 
  Satellite  Enhanced Vegetation
  Index (EVI).  The eigenstructure conveys the baseline behavior of
  the vegetation land cover. From the visible domain containing novel
  optical information at single time instances, the KL expansion of
  the optical training dataset is formed. Due to the cloud removal
  pre-processing algorithms \cite{Skakun2022}, the domain of the novel
  optical information may have regions with missing data, and so the
  domain may be irregular in shape.  From the truncated KL expansion
  the multilevel representation of the residual space is formed. By projecting
  novel inputs onto this space at fixed time, the residual map is
  obtained, and we can form  hypothesis tests to measure the
  anomalies.}

Our work here augments this approach to vector field data,
i.e. multiple structured features. The detection framework thus
extends multiple satellite detection modalities (e.g., multi-spectral
data) allowing augmented coordinated detection over scalar modalities,
e.g., in EVI data. The deforestation and
degradation detection problem can thus be framed in a precisely stated
and much more general mathematical framework, using high dimensional
probabilistic constructions.

\begin{figure}[t]
\centering

\begin{tikzpicture}[scale = 1.1, every node/.style={scale=0.9},>=latex']

    \node at (3,-1.75) {\includegraphics[scale = 0.25, trim = 5cm 8cm 5cm 7.5cm, 
         clip]{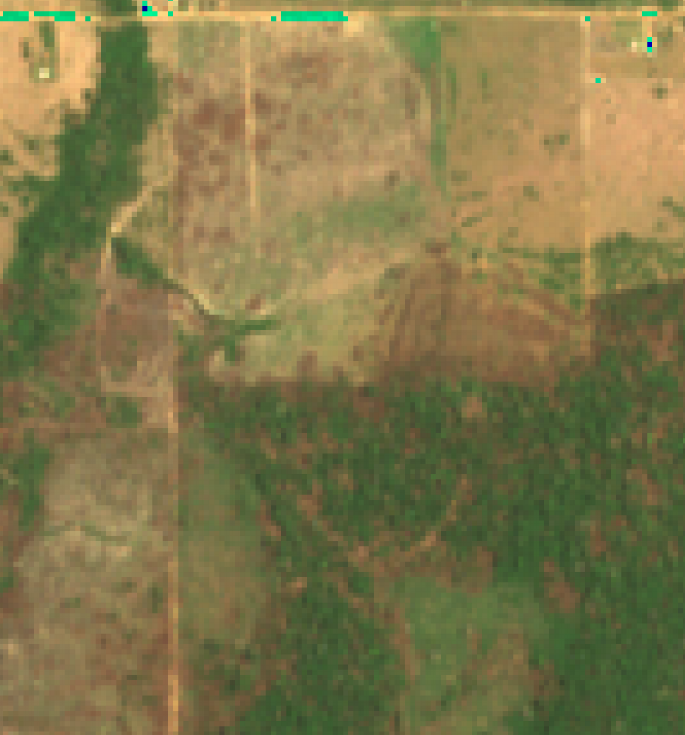}};
    \node at (3,0) {\Large Data};
    \node at (8,0) {\Large Filter};
    \node at (13,0) {\Large Statistics};
           
    \draw[tan,line width=0.5mm,dashed] (5.5,0) -- (5.5,-5.5);
    \draw[tan,line width=0.5mm,dashed] (10.5,0) -- (10.5,-5.5);
     
    \node[block] at (3,-3.65) (sourcenew)
    { \begin{tabular}{c} 
     Novel optical \\
     information 
       \end{tabular}};

       \node[block] at (3,-5) (source)
    { \begin{tabular}{c} 
     %Novel information \\
     Optical training\\  dataset
       \end{tabular}};
            
    \node[block] at (13,-1) (statistics)            
           {
           \begin{tabular}{c}
           Hypothesis \\ 
           tests
           \end{tabular} 
           };
    %       \draw[vecArrow] (filter.north) -- (measures.south);
           %\draw[vecArrow] (radar.east) -- (filter.west);

            \node[block] at (8,-1) (spaces)            
           {\begin{tabular}{c} 
           Anomaly \\
           map (residual)
           \end{tabular}} ;
           \draw[vecArrow] (spaces.east) -- (statistics.west);

            \node[block] at (8,-2.35) (multi)            
           {\begin{tabular}{c} 
           Multilevel Representation\\
           of residual subspace
           \end{tabular}} ;
           \draw[vecArrow] (multi.north) -- (spaces.south);

            \node[block] at (8,-3.65) (expansion)            
           {\begin{tabular}{c} 
           Karhunen-Lo\`{e}ve \\
           Expansion
            \end{tabular}} ;
           \draw[vecArrow] (expansion.north) -- (multi.south);
           \draw[vecArrow] (sourcenew.east) -- (expansion.west);
           \draw[vecArrow] (sourcenew.east) -- (spaces.west);
           
            \node[block] at (8,-5) (spatio)            
           {\begin{tabular}{c} 
           Covariance 
           eigenstructure \\
           $(\lambda_k,\phi_k)$           
           \end{tabular}};
           \draw[vecArrow] (spatio.north) -- (expansion.south);
           \draw[vecArrow] (source.east) -- (spatio.west);

                  \node at (13,-3.5)
              {
                \includegraphics[scale = 0.45, trim = 0cm 0cm 28.2cm 1.44cm, clip
                ]{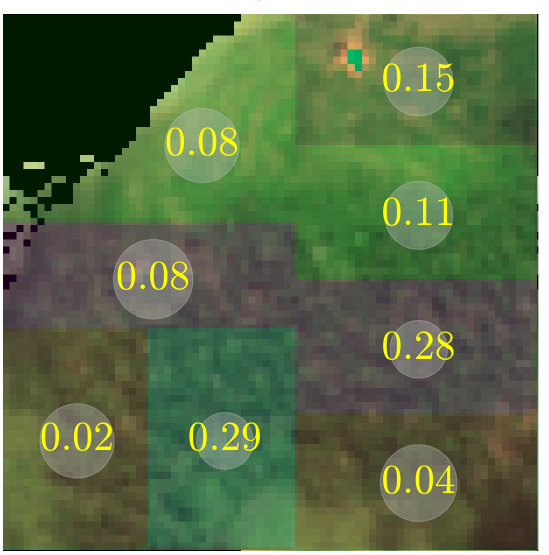}
              };  

\end{tikzpicture}
\caption{Anomaly detection pipeline. Optical training data are
    used to construct the eigenstructure. From the eigenstructure and
    the domain of the novel optical information, the
    Karhunen-Lo\`{e}ve expansion is built. The KL module produces the
    truncation expansion of the random vector field of the training
    dataset. From the truncated KL expansion the multilevel representation of
    the residual space is constructed. The next step is to project the
    novel optical information onto the multilevel basis in the
    residual space, giving rise to the anomaly map. Hypothesis tests
    and anomaly magnitudes can then be computed.}
\label{PR:Fig1}
\end{figure}

\section{Vector field Karhunen-Lo\`{e}ve}
\label{Introduction}

The Karhunen-Lo\`{e}ve expansion is an important methodology that
represents random fields in terms of spatial-stochastic tensor
expansions. It has been shown to be optimal in several ways, making it
attractive for analysis of random fields.  We are interested in data
that can be modeled as random vector fields. For example, satellite
sensors with multiple spectral bands naturally form vectorized
data, among many such examples.  
%for example satellite
%sensors with multiple spectral bands forming naturally vectorized
%ata, among many such examples. 
In this section the mathematical
background for the vector field KL expansion is introduced.

Let $U$ be a domain of $\R^{d}$, $d \in \nset$, and
$(\Omega,\mcF,\bbP)$ be a complete probability space, with a set of
outcomes $\Omega$, and $\mcF$ a $\sigma$-algebra of events equipped
with the probability measure $\bbP$. Let $L^{2}(U;\R^{d})$ be the
space of all square integrable functions $\bv:U \rightarrow \R^{d}$
equipped with the standard inner product $(\bu,\bv) = \int_{U} \bu\T
\bv \, \mbox{d}\bx$, for all $\bu,\bv \in L^{2}(U;\R^{d})$, where
$\bv(\bx) = [v_1(\bx),$ $\dots,v_{d}(\bx)]\T$. 
%and $v_i \in
%L^{2}(U):=L^{2}(U;\R)$ for $i = 1,\dots,q$.

Taking uncertainty into account, suppose that the data can be
described as random vector field $\bv:\Omega \rightarrow
L^{2}(U;\R^{d})$, where $\bv(\bx,\omega) = [v_1(\bx,\omega),$
  $\dots,v_{d}(\bx,\omega)]\T$, $\bx \in U$, $\omega \in \Omega$, and
\corb{$v_i(\bx,\cdot) \in L^{2}(\Omega)$} for $i = 1,\dots,d$. Note
that from context it will be clear when $\bv(\bx)$ and
$\bv(\bx,\omega)$ is referenced.  Let $L^{2}_{\bbP}(\Omega;
L^{2}(U;\R^{d}))$ be the space of all strongly measurable functions
$\bv:\Omega \rightarrow L^{2}(U;\R^{d})$ equipped with the inner
product $(\bu,\bv)_{L^{2}_{\bbP}(\Omega; L^{2}(U))} = \int_{\Omega}
  (\bu,\bv) \, \mbox{d}\bbP$, for all $\bu,\bv \in
  L^{2}_{\bbP}(\Omega;L^{2}(U;\R^d))$.
\begin{definition}~
\begin{enumerate}[(a)]
\item Suppose $\bv \in L^{2}_{\bbP}(\Omega;L^{2}(U;\R^{d}))$, and denote
  $\eset{\bv} := [\eset{v_1},\dots,\eset{v_d}]$ as the mean of
  $\bv(\bx,\omega)$, where $\eset{v_i} : = \int_{\Omega}
  v_i(\bx,\omega) \,\emph{d} \bbP$, for $i = 1,\dots,d$.
\item For all $\bv \in L^{2}_{\bbP}(\Omega;L^{2}(U;\R^{d}))$ let
  $\Cov(v_i(\bx,\omega),v_j(\by,\omega)) := \mathbb{E} [
  (v_i(\bx,\omega)- \eset{v_i(\bx,\omega)})$ $(v_j(\by,\omega)$ $-
  \eset{v_j(\by,\omega)}) ]$ for $i,j = 1,\dots,d$, and denote the
  matrix-valued covariance function of $\bv$ as
  $\Cov(\bv(\bx),\bv(\by)):=$
\[
\begin{bmatrix}
  \Cov(v_1(\bx,\omega),v_1(\by,\omega)) & \Cov(v_1(\bx,\omega),v_2(\by,\omega)) & \dots  & \Cov(v_1(\bx,\omega),v_d(\by,\omega)) \\
  \Cov(v_2(\bx,\omega),v_1(\by,\omega)) & \Cov(v_2(\bx,\omega),v_2(\by,\omega)) & \dots  & \Cov(v_1(\bx,\omega),v_d(\by,\omega)) \\
  \vdots                                & \vdots                               & \ddots & \vdots \\
  \Cov(v_d(\bx,\omega),v_1(\by,\omega)) & \Cov(v_d(\bx,\omega),v_2(\by,\omega)) & \dots  & \Cov(v_d(\bx,\omega),v_d(\by,\omega)) \\
  \end{bmatrix}.
\]
\end{enumerate} 
\end{definition}

From the properties of Bochner integrals (see
\cite{Light1985,Harbrecht2016}) we have that
\corb{$\eset{v(\bx,\omega)} \in L^{2}(U;\R^{d})$} and that the
covariance function $\Cov(v_i(\bx,\omega),$ $v_j(\by,\omega)) \in
L^{2}(U \times U)$ for all $i,j = 1,\dots d$. Thus
$\Cov(\bv(\bx,\omega),\bv(\by,\omega)) \in L^{2}(U \times U;\R^{q
  \times q})$, where the space $L^{2}(U \times U;\R^{d \times d})$ is
equipped with the inner product
\[
\corb{(\bA,\bB)_{L^2(U \times U;\R^{q \times q})}} := \int_U \int_U
\oone\T \bA(\bx,\by) \bullet \bB(\bx,\by) \oone\,\,\mbox{d}\bx
\mbox{d}\by
\]
for all $\bA, \bB \in L^{2}(U \times U;\R^{d \times d})$, where
$\oone$ is a $d$ dimensional vector with all entries equal to one and
$\bA \bullet \bB$ corresponds to the Hadamard product of $\bA$ and
$\bB$.

Although Karhunen-Lo\`{e}ve expansions are well understood for the
scalar case, rigorous proofs and studies of existence and optimality properties for
the vector field case have been somewhat sparse. Despite the
popularity of the KL expansion and its multiple references, a rigorous
existence and optimality proof for the vector field case could not be
found except for the detailed analysis in \cite{Schwab2006}, in the
context of tensor products. However, the application of this approach is
not trivial and requires careful treatment.  We need to construct the explicit 
tensor product on the appropriate Hilbert subspaces and show that 
$L^{2}_{\bbP}(\Omega;L^{2}(U;\R^{d}))$ is isomorphic to  $L^{2}(U;\R^{d}) \otimes L^{2}_{\bbP}(\Omega)$. The details are described in Appendix \ref{appn}.

Consider the operator
\[
  \mcC_{\bv} (\bu)(\bx) := \int_{U} {\rm
      Cov}(v(\bx,\omega),v(\by,\omega)) \bu(\by)\,\mbox{d} \by
\]
for all $\bu \in L^{2}(U;\R^{d})$.  From Lemma 2 in \cite{Harbrecht2016} the
operator $\mcC_{\bv}:L^{2}(U;\R^{d}) \rightarrow L^2(U;\R^{d})$ is a
non-negative symmetric trace class operator. From Theorem 1 in
\cite{Harbrecht2016} there exists an orthonormal set of eigenfunctions
$\{\bphi_k\}_{k \in \bbN}$, where $\bphi_{k} \in L^{2}(U;\R^{d})$, and
eigenvalues $\lambda_1 \geq \lambda_2 \geq \dots\geq 0$ such that
$\mcC_{\bv} \bphi_k$ = $\lambda_k \bphi_k$ for all $k \in
\bbN$. We can now form the vector field KL expansion (See Appendix \ref{appn}).

\begin{theorem}
  Suppose  $\bv \in L^{2}_{\bbP}(\Omega;L^{2}(U;\R^{d}))$ then
  \[
  \bv(\bx,\omega) = \eset{\bv(\bx,\omega)} + \sum_{k \in \bbN} \lambda^{\frac{1}{2}}_{k} \bphi_k(\bx) Y_{k}(\omega),
  \]
where $ Y_k(\omega) = \frac{1}{\sqrt{\lambda_k}} \int_U
(\bv(\bx,\omega) - \eset{\bv(\bx,\omega)}) \T \bphi_{k}(\bx)
\,\mbox{\emph{d}} \bx$, $\eset{Y_k Y_l} = \delta_{kl}$ and $\eset{Y_k}
= 0$ for all $k,l \in \bbN$.
\end{theorem}

The KL expansion has the useful property of being optimal in the set
of all product expansions.  Suppose that $H_M \subset L^{2}(U;\R^d)$ is a
finite dimensional subspace of $L^{2}(U;\R^d)$ such that $\dim H_M = M$ and
$P_{H_M \otimes L^{2}_{\bbP}(\Omega)}: L^{2}(U;\R^{d}) \otimes
L^{2}_{\bbP}(\Omega) \rightarrow H_M \otimes L^{2}_{\bbP}(\Omega)$ is
an orthogonal projection operator. Suppose $\bv \in L^{2}_{\bbP}(\Omega;L^{2}(U;\R^{d}))$, where
$\eset{\bv} = 0$, 
from Proposition \ref{intro:prop3} 
$\bv \in
  L^{2}(U;\R^{d}) \otimes L^{2}_{\bbP}(\Omega)$, 
and
%then from Theorem 2.7 in \cite{Schwab2006} 
\[
\inf_{\begin{array}{c}
H_M \subset L^{2}(U;\R^{d}) \\
\mbox{dim}\, H_M = M
\end{array}
} \|\bv - P_{H_M \otimes
L^{2}_{\bbP}(\Omega)} \bv \|_{L^{2}_{\bbP}(\Omega) \otimes L^{2}(U)} =
\left( \sum_{k \geq M+1} \lambda_k \right)^{\frac{1}{2}}
\]
where the infimum is achieved \corb{only} when $H_M =
\mbox{span}\{\bphi_1,\dots,\bphi_{M}\}$.
\begin{remark}
In practice the KL expansion of a random field is intractable for even
a modest number of terms. Estimating the joint distribution of
$Y_1,\dots,Y_M$ requires massive amounts of data due to the high
dimensionality. Only under certain conditions, such as for Gaussian
processes, can the random variables $Y_1,\dots,Y_M$ be explicitly
known (and shown to be independent).  However, for the anomaly filter
built in this paper, the joint
distribution of
$Y_1,\dots,Y_M$ is not required. Furthermore, for the hypothesis test
derived in section \ref{section:mls} only the eigenpairs
$(\lambda_k,\phi_k)$ for $k = 1,\dots,M$ are needed. This makes for a
significantly easier problem whose quantities can be estimated in
practice from realizations of the random field \corb{$\bv(\bx,\omega)
  \in L^{2}_{\bbP}(\Omega;L^{2}(U;\R^{d}))$} using the method of
snapshots \cite{Castrillon2002}.
\end{remark}

%\begin{remark}
%\end{remark}

\section{Anomaly detection and multilevel orthogonal eigenspaces}
\label{section:mls}

The KL expansion provides a mechanism to represent a vector valued
random field $\bv(\bx,\omega)$ in terms of optimal approximations
based on the first $M$ terms: $\bv_M(\bx,\omega) - \eset{\bv} =
\sum_{k = 1}^{M} \lambda^{\frac{1}{2}}_{k} \phi_k(\bx) Y_{k}(\omega)$.
In the rest of the discussion, without loss of generality, it is
assumed that $\eset{\bv} = \0$.  Suppose that $\bu(\bx,\omega)$ is an
observable random field and assume that the model is given by
$\bu(\bx,\omega) = \bv_M(\bx,\omega) + \bw(\bx,\omega)$.  Given
knowledge of the eigenstructure of $\bv_M(\bx,\omega)$ and the
observations $\bu(\bx,\omega)$, the goal is detection of the anomalous
process $\bw(\bx,\omega)$ and its quantification with respect to a
suitable norm.  Detection is achieved by the construction of
multilevel spaces that are adapted to local and global components of
$\{\phi_1,\dots,\phi_M\}$.

\begin{asum}
We let $\bV_{0} := \mbox{\emph{span}} \{\phi_1,\phi_2, \dots,\phi_M \}$ and
$\bV_0 \subset \bV_1 \dots \subset L^{2}(U;\R^{d})$ be a sequence of
nested subspaces in $L^{2}(U;\R^{d})$ such that
$\overline{\bigcup_{k \in \bbN_{0}} \bV_{k}} =
L^{2}(U;\R^{d})$. Furthermore, for all $k \in \bbN$, we let
$\bW_k \subset L^{2}(U;\R^{d})$ be a subspace such that $\bV_{k+1}
= \bV_{k} \oplus \bW_{k}$, where $\oplus$ is the direct sum, so that
$\overline{
\bV_0 \bigoplus_{k \in \bbN_{0}} \bW_{k}} = L^{2}(U;\R^{d})$.
\label{mls:assum1}
\end{asum}

Although the definition of these spaces is relatively simple, the
construction is elaborate with heavy notation and based on
differential operator-adapted multilevel methods from scientific
computing and computational applied mathematics, for solution of
Partial Differential Equations
\cite{D'Heedene2005,Castrillon2003}. The details of construction of
these multilevel spaces for random fields on complex geometries can be
found our recent publication \cite{Castrillon2022} for the scalar
case. Here we extend this optimal decomposition to multidimensional
vector fields.

We will assume that $U$ can be decomposed into simplices, which can
be thought of as generalizations of triangulations to arbitrary
dimensions.  This allows complex geometric shapes to be simply
approximated. First, we give some definitions.

\begin{definition} A $k$-simplex is defined to be a convex hull of vertices
  $z_0,z_1,\dots,z_k \in \R^d$ that are affinely independent.
 \end{definition}

\begin{definition}
\begin{enumerate}[i)]
\item We denote $\bx_i$ to be the barycenters of simplices $\tau_i \in
  \mcT$ and define $\bbS := \{\bx_1,\dots,\bx_N\}$.
\item The {\rm face} of a $k$-simplex is the convex hull of any $m +
  1$ subset of the points that define a $k$-simplex.
\end{enumerate}
\end{definition}

\begin{definition}
Suppose that $\mcT$ is a collection of simplices in $\R^{d}$.  Then
$\mcT$ is a $k$-simplicial complex if the following properties are
satisfied:
\begin{enumerate}[i)]
        \item Every face of a simplex in $\mcT$ is also in
        $\mcT$.
        \item The non-empty intersection of any two simplices
        $\tau_1,\tau_2 \in \mcT$ is a face of both $\tau_1$ and
        $\tau_2$.
        \item The highest dimension of any simplex in $\mcT$ is
        $k \leq d$.
\end{enumerate}
\end{definition}

The following assumption allows us to construct complex geometrical
shapes from the $k$-simplices and define a space of functions
$\bV_{n+1}$ on them that approximates the vector field
$v(\bx,\omega)$. In Figure \ref{MLRLE:fig1} an example of
triangulation of a surface constructed from $2$-simplices (triangles)
is shown.

\begin{asum}~
\begin{enumerate}[i)]
        \item $U = \cup_{\tau_i \in \mcS} \tau_i$, where $\mcS$ is a
        subset of $\mcT$ and contains $N$ simplices of order $k$.

        \item For any simplex $\tau_i \in \mcS$ and $j = 1,\dots,d$
          let $\bchi^j_{i} := c^j_i
          [0,\dots,0,1^j_{\tau_i},0,\dots,0]$, where $\bchi^j_{i} \in
          L^{2}(U;$ $\R^{d})$ and $1^j_{\tau_i}$ corresponds to the
          indicator function at the $j^{th}$ entry in the vector
          $\bchi^j_i$ on the simplex $\tau_i$.

        \item The coefficients $c^j_i$ for $i = 1, \dots, N$ and $j =
          1,\dots,d$ are chosen such that collection of functions $\bchi^j_i$,
          $\mcE := \{\{\bchi^j_i\}_{i = 1}^{N}\}_{j = 1}^{d}$, forms
        an orthonormal set in $L^{2}(U;\R^{d})$.

        \item Let $\bV_{n+1} = \mcP(\mcE) := span \{\bchi^n_i\}$. We assume
        that Karhunen-Lo\`{e}ve eigenfunctions $\bphi_i \in \mcP(\mcE)$
        for all $i = 1,\dots,M$ where $N > M$.
\end{enumerate}
\label{mls:assum3}
\end{asum}

From the set of indicator functions in $\mcE$, a multilevel basis
representation can be constructed that is adapted to the geometry of
the domain $U$ and the eigenfunctions $\{\phi_1,\phi_2, \dots$ $,\phi_M
\}$. This will allow detection of signals in the vector field in a
local and global sense.  The algorithm to construct a multilevel basis
for the scalar case is described in detail in
\cite{Castrillon2022}. The extension of this basis to the vector field
case can be essentially obtained by replacing the $L^2(U)$ inner
product with that in $L^2(U;\R^{d})$.

For the sake of
completeness the algorithm building the binary tree is described here in
detail.  We will then show how the multilevel basis is constructed for
vector field data. We construct a binary tree to efficiently locate the simplices in
$\mcS$ at different levels of resolution. Furthermore, the binary tree
will serve as a base to construct and locate the multilevel basis
functions of the spaces $\bW_{k}$ for $k = 0, \dots, n$.  The domain
$U$ is initially assumed to be embedded in a square cell of unit
length. For $N$ this can be easily done with a rescaling. We follow
the procedure described in \cite{Dasgupta2008} for the construction
of a kd-tree type decomposition. Other options include Random
Projection (RP) trees, which can be found in \cite{Dasgupta2008}.

Suppose that all the barycenters $\bx \in \bbS$ are embedded in the root cell $B^{0}_{\tt root} (
\mbox{or $B^0_0$})$ $\subset \R^d$, which corresponds to the top of the
binary tree. Without loss of generality it can be assumed that
$B^{0}_{\tt root} = [0,1]^{d}$ and $\bbS \subset B^{0}_{\tt
root}$. Now, the root cell is subdivided according to the rule in
Algorithm \ref{maketree}, thus forming two new cells
$B^{1}_{{\tt left}}$ and $B^{1}_{\tt right}$ at level 1. In general
for any cell $B^l_k$ at level $l$ and index $k$ the collection of
barycenters $\tilde \bbS = \{\bx_j | \bx_j \in B^l_k\}$ is subdivided
as follows by using the following rule \cite{Dasgupta2008} (See Algorithm \ref{chooserule}):

% Add Mesh --------------------------------------------------------------------
\def\addtriangle#1{
    \xdef\trianglesbuffer{\trianglesbuffer #1}
}

\def\calculatecoordinate(#1,#2,#3)[#4]=\f(#5,#6){
        \pgfmathsetmacro#1{#5}
        \pgfmathsetmacro#2{#6}
        \pgfmathsetmacro#3{\f({(#5)},{(#6)})}
        \pgfmathsetmacro#4{\g({(#5)},{(#6)})}
}
\def\calculatetriangle#1{
    % #1 is +            #1 is -
    %         C                  B A  
    %          ◣                  ◥   
    %         A B                  C  
    \calculatecoordinate(\xa,\ya,\za)[\wa]=\f(\x,\y)
    \calculatecoordinate(\xb,\yb,\zb)[\wb]=\f(\x#11,\y)
    \calculatecoordinate(\xc,\yc,\zc)[\wc]=\f(\x,\y#11)
    \addtriangle{(\xa,\ya,\za)[\wa] (\xb,\yb,\zb)[\wb] (\xc,\yc,\zc)[\wc]}
}
\def\calculatetheplot{
    \foreach\x in{0,...,20}{
        \foreach\y in{0,...,20}{
            % check \x + \y + \z =1
            %       ◣
            %       ◣ ◣
            %       ◣ ◣ ◣
            \ifnum\numexpr\x+\y<20
                \calculatetriangle+
            \fi
            %       
            %        ◥
            %        ◥ ◥
            \ifnum\numexpr\x>0 \ifnum\numexpr\y>0 \ifnum\numexpr\x+\y<21
                \calculatetriangle-
            \fi\fi\fi
        }
    }
}

\begin{figure}
  \centering
  %\hspace{9cm}
  \begin{tikzpicture}
\begin{scope}[place/.style={circle,draw=blue!50,fill=blue!20,thick, inner sep=0pt,minimum size=1mm}]
    \begin{axis}[scale = 1.5, axis x line=none, axis y line=none, axis z line=none]
        \def\trianglesbuffer{} % initialise the buffer
        \def\f(#1,#2){#2*#2/20+2*sin(20*#1)} % we want to plot this function
        \def\g(#1,#2){sqrt((#1-20/3)^2+(#2-20/3)^2)} % with this point meta
        \calculatetheplot
        \edef\pgfmarshal{
            \noexpand\addplot3
                [patch,patch type=triangle,point meta=explicit,opacity=1,shader=faceted,colormap/jet]
                coordinates{\trianglesbuffer};
        }
        \pgfmarshal
    \end{axis}
    \node at (-3.5,0) {};

   \node at (4.4,5.21) [place] {};
   \node at (2.23,1.5) {$U$};
   % \draw[<-] (4.23,5.37) to (6,6);
   \draw[<-] (4.53,5.25) to (6,6);

   %\node at (7,7) {\tiny $\bv(\bx,\omega)$};
   \node at (6.7,6) {\footnotesize $\begin{bmatrix} v_1(\bx,\omega) \\ v_2(\bx,\omega) \\ \vdots \\  v_d(\bx,\omega)
   \end{bmatrix}$};
   
\end{scope}
\end{tikzpicture}
\caption{Surface domain $U$ constructed from $2$-simplices (triangles) in $\R^3$. The vector random field
$\bv(\bx,\omega) \in L^{2}_{\bbP}(\Omega;L^{2}(U;\R^{d}))$ is defined
over the domain $U$. As an example this could be satellite multi-spectral data over land.}
\label{MLRLE:fig1}
\end{figure}
%-------------------------------------------------------------------------------following

\begin{enumerate}[(a)]

\item Suppose $v_i$ is the unit vector in the axis coordinate direction for $i = 1,\dots,d$.

\item For each coordinate direction $i = 1,\dots,d$ project every barycenter
$\bx_i \in \tilde \bbS$ onto the unit vector $v_i$

\item Compute the sample variance of these projection coefficients for each coordinate
unit vector $v_i$.

\item Choose the unit coordinate vector $v_i$ in the direction $1\leq j \leq d$
with the maximal sample variance for the above projection
coefficients.

\item Compute the median of the projections along $v$ and split the cell in two
parts ($B^{l-1}_{{\tt left}}$ and $B^{l-1}_{\tt right}$) at this
coordinate position.
  
\end{enumerate}

\begin{algorithm}[H]
\caption{\textsc{ChooseRule}$(\tilde \bS)$ for kd-tree splitting}
\begin{algorithmic}[1]
\Require Point set $\tilde \bbS\subset \R^d$ in a cell
\Ensure Splitting predicate $\textsf{Rule}(x)$, direction $v$, threshold $\tau$
\State For each coordinate direction $e_j$, compute the sample variance of $\{x\cdot e_j: x\in \tilde \bbS\}$
\State Select direction $v \gets e_{j^\star}$ with maximal variance
\State Set threshold $\tau \gets \mathrm{median}\{x\cdot v : x\in \tilde \bbS\}$
\State Define $\textsf{Rule}(x) \equiv (x\cdot v \le \tau)$
\State \Return $(\textsf{Rule},\tau,v)$
\end{algorithmic}
\label{chooserule}
\end{algorithm}

\begin{algorithm}[H]
\caption{\textsc{MakeTree}$(\bbS,n_0)$ (recursive kd-tree construction)}
\begin{algorithmic}[1]
\Require Barycenters $\bbS=\{x_1,\dots,x_N\}$, leaf threshold $n_0$
\Ensure A binary tree $\bT$ whose nodes store cells $B_k^\ell$ and associated point sets

    \State $k$ $\leftarrow$ 0, $l$ $\leftarrow$ 0, $
    \tilde \bbS \gets \{\bx\, | \,\bx \in \bbS\}$
    \State ($\bT$, $k$, $l$) $\leftarrow$ \Call{MakeTreeNode}{$\tilde \bbS$, $n_0$, $k$, $l$}
    %\State $\bT.node$ $\leftarrow$ $node$ 

\Function{MakeTreeNode}{$\tilde \bbS, n_0, k, \ell$}

    \State $\bT$.$B^l_k \gets \{\bx\, | \,\bx \in \tilde \bbS\}$
    \State $\bT$.$k$ $\gets$ $k$, $\bT$.$l$ $\gets$ $l$
    \State $k$ $\leftarrow$ $k$ + 1, $l \leftarrow l + 1$
    \If {$|\tilde \bbS| < n_0$} \State \Return $\bT$, $k$
    \Comment{Leaf node}
    %\hspace*{-\algorithmicindent} 
    \EndIf
    \State (Rule, threshold, $v$) $\leftarrow$ \Call{ChooseRule}{$\tilde \bbS$}
    \State ($\bT$.LeftTree, $k$) $\leftarrow$  \Call{MakeTreeNode}{$\bx \in \tilde \bbS$: Rule($\bx$) = True, $n_0$, $k$, $l$}
    \State ($\bT$.RightTree, $k$) $\leftarrow$ \Call{MakeTreeNode}{$\bx \in \tilde \bbS$: Rule($\bx$) = false, $n_0$
$k$, $l$}
     \State $\bT$.threshold $\gets$ threshold, $\bT$.$v \gets v$
\EndFunction
%\State \Call{MakeTreeNode}{$S,n_0,0,0$}
\State \Return $\bT$, $k$
\end{algorithmic}
\label{maketree}
\end{algorithm}

By applying Algorithm \ref{maketree} we obtain a tree structure $\bT$, which contains
all the cells $B^l_k$ at each level of resolution $l = 0,\dots,n$ for
every level $l$ and associated index $k$. Furthermore, let $\mcB$ be
the collection of all the non-empty cells $B$ in the tree $\bT$ and
$\mcB^{l}:=\{B^j_m \in \mcB \,|\, j = l\}$. From the tree structure $\bT$
and the set $\mcE$ the multilevel basis adapted to the vector field
KL expansion can be constructed (See Figure \ref{MLRLE:fig2}). 

\begin{figure}[t]

\begin{center}
\begin{tikzpicture}[scale=0.8,
   grow                    = down,
    sibling distance        = 0.4em,
    %level 1/.style          = {sibling distance=1cm},
    %level 2/.style          = {sibling distance=0.1cm},
    %level 3/.style          = {sibling distance=0.1cm},
    level distance          = 1.2cm,
    edge from parent/.style = {draw, edge from parent path={(\tikzparentnode) -- (\tikzchildnode)}},
    %sloped
]
    %\node[anchor=center] at (0, -4.5) {$$};
    %\node[anchor=center] at (0,   10) {$$};
\begin{scope}[xshift=5cm, yshift=5cm,
place/.style={circle,draw=blue!50,fill=blue!20,thick,
      inner sep=0pt,minimum size=8mm},
placer/.style={circle,draw=darkgreen!50,
  preaction={fill=olivine,opacity = 0.5}, thick,inner
  sep=0pt,minimum size=4mm}, ]

  %placer/.style={circle,draw=blue!50,
  %preaction={fill=darkgreen!60,fill opacity=0.5}, thick,inner
  %sep=0pt,minimum size=1.5mm}, ]

%\filldraw[fill={rgb:red,143;green,188;blue,143},semitransparent, 
%      thick] (0, 0) rectangle (16, 16);
  
\Tree [.\node[place]{$\bT^{0}_{\tt root}$}; 
             [.\node[place]{$\bT^{1}_{\tt left}$};
                    [.\node[place]{$\bT^{2}_{\tt left}$}; 
                           [.\node[place]{$\bT^{3}_{\tt left}$};]
                           [.\node[place]{$\bT^{3}_{\tt right}$};] 
                    ]       
                    [.\node[place]{$\bT^{2}_{\tt right}$}; 
                           [.\node[place]{};] 
                           [.\node[place]{};] 
                    ] 
             ]                                        
             [.\node[place]{$\bT^{1}_{\tt right}$};
                          [.\node[place]{$\,$};
                                  [.\node[place]{};] 
                                  [.\node[place]{};] 
                           ]
                           [.\node[place]{}; 
                                  [.\node[place]{};] 
                                  [.\node[place]{};] 
                           ] 
                                          ]
]

%\buildMeshBW{(0.3,0.3);(1.5,1);(4,0);(4.5,2.5);(1.81,2.14);(2.5,0.5);(2.8,1.5)}

\end{scope}
\end{tikzpicture}
\end{center}
\caption{Binary tree example from the simplices in $\mcS$  used
to construct the domain $U$.  Algorithm \ref{maketree} is used to decide how the
barycenters in $\tilde \bS$ ($B^k_l$) are split. The tree structure $\bT$
is built from Algorithm \ref{chooserule} and \ref{maketree} and is constructed recursively until
there are at most $n_0 - 1$ barycenters left in the cell $B^l_k$. Once
all the leaves are reached the algorithm stops. Note that it is possible that
not all of the leaves are at the same level. This depends on the location of the barycenters.} 
%Notice that the list
%of barycenters $\tilde \bS$ at any level $l > 0$ is updated with the
%barycenters that satisfy Rule($\bx$) or do not satisfy Rule($\bx$).}
\label{MLRLE:fig3}
\end{figure}

\subsection{Multilevel Basis Construction}

\emph{We first show how to construct the multilevel basis functions at the
finest level (or leaf) $n$.} Suppose that $B^n_k \in \mcB^n$, and
after reordering of the numbering of the barycenters suppose that
$\{\bx_1,\dots,\bx_s\}$ are the barycenters contained in $B^n_k$. For
each simplex $\tau_i$ with corresponding barycenter $\bx_i$ there
exist $\bchi^1_i, \dots,
\bchi^d_i$ as orthonormal functions.  Thus, we form the orthonormal set $\mcE^n_k 
:= \{ \bchi^1_1, \bchi^1_2, \dots, \bchi^1_{s_{n,k}}, \bchi^2_1, \bchi^2_2,$
$\dots, \bchi^2_{s_{n,k}},\dots,
\bchi^{d}_1, \bchi^{d}_2, \dots$  $\bchi^{d}_{s_{n,k}}\}$
where $s_{n,k} = s$. 
%In addition, let $\mcE := \bigcup_{B^n_k \in \mcB^n} {\mcE^n_k}$.  
The objective is to form a linear combination of the elements in
$\mcE^n_k$; we will construct a multilevel grid with the desired
properties. To this end let
\begin{equation*}
%\[
\begin{split}
\bphi^{n,k} _{j}
&:= \sum_{i = 1}^{s_{n,k}} \sum_{h = 1}^{d} 
c^{n,k} _{i,h,j}
  \bchi^h_i, \hspace{2mm} j \in \{1, \dots, a_{n,k}\} \\
\bpsi^{n,k}_{j}
&:= \sum_{i = 1}^{s_{n,k}} \sum_{h = 1}^{d} 
d^{n,k}_{i,h,j} \bchi^h_i, \hspace{2mm}
j \in \{a_{n,k}+1, \dots, s_{n,k}\},
\end{split}
%\]
\end{equation*}
where the coefficients $c^{n,k}_{i,h,j}, d^{n,k}_{i,h,j},
a_{n,k} \in \mathbb{R}$ are still unknown. The goal is to construct
$\bpsi^{n-1,k}_j$ such that it is orthogonal to the subspace $\bV_0$
under the $L^{2}(U;\R^d)$ inner product, i.e. for $i = 1,\dots,M$ and
$j = a_{n,k}+1, \dots, s_{n,k}$,
\begin{equation}
\int_{U}
\phi_i(\bx)\T  \bpsi^{n,k}_{j}(\bx) \, \mbox{d} \bx 
= 0.
\label{hbconstruction:eqn1}
\end{equation}
From the eigenfunctions $\phi_1, \dots, \phi_M$ of the KL expansion
  and $\mcE^{n}_{k}$ we can form the \corb{matrix}
\[
\footnotesize
\bM^{n,k} := 
\begin{bmatrix}
 ( \phi_{1}(\bx),  \bchi^{1}_1(\bx))  & 
\dots & 
 ( \phi_{1}(\bx),  \bchi^{s_{n,k}}_1(\bx))  &
 \dots &
  ( \phi_{1}(\bx),  \bchi^{1}_d(\bx))  & 
\dots & 
 ( \phi_{1}(\bx),  \bchi^{s_{n,k}}_d(\bx))
\\
( \phi_{2}(\bx),  \bchi^{1}_1(\bx))  & 
\dots & 
 ( \phi_{2}(\bx),  \bchi^{s_{n,k}}_1(\bx))  &
 \dots &
  ( \phi_{2}(\bx),  \bchi^{1}_d(\bx))  & 
\dots & 
 ( \phi_{2}(\bx),  \bchi^{s_{n,k}}_d(\bx))
 \\
 \vdots & \vdots & \vdots & \vdots & \vdots & \vdots & \vdots
\\
( \phi_{M}(\bx),  \bchi^{1}_1(\bx))  & 
\dots & 
 ( \phi_{M}(\bx),  \bchi^{s_{n,k}}_1(\bx))  &
 \dots &
  ( \phi_{M}(\bx),  \bchi^{1}_d(\bx))  & 
\dots & 
 ( \phi_{M}(\bx),  \bchi^{s_{n,k}}_d(\bx))
\end{bmatrix},
\]
where $(\cdot,\cdot)$ is the standard $L^{2}(U;\R^d)$ inner
product. From the matrix $\bM^{n,k}$ the coefficients
$c^{n,k}_{i,h,j}, d^{n,k}_{i,h,j},$ $a_{n,k} \in \mathbb{R}$ can be
computed. To this end apply the Singular Value Decomposition (SVD) to
$\bM^{n,k}$
\begin{equation}
\bM^{n,k} = \bU \bD \bV \T,
\label{hbconstruction:eqn2}
\end{equation}
where $\bU \in \R^{ M \times M}$, $\bD \in \R^{M \times s_{n,k}d}$,
$\bV \in \R^{s_{n,k}d \times s_{n,k}d}$, and let $a_{n,k}$ be the rank of the matrix
$\bM^{n,k}$, i.e.  the number of non-zero singular values of the matrix
$\bD$. Our choices of coefficients $c^{n,k}_{i,h,j}$ and $d^{n,k}_{i,h,j}$
are now set to:
\begin{equation}
  \left[ \begin{array}{ccc|ccc}
      c^{n,k}_{1,1,1} & \dots &c^{n,k}_{1,1,a_{n,k}} & d^{n,k}_{1,1,a_{n,k}+1} & \dots &d^{n,k}_{1,1,s_{n,k}d} \\
      c^{n,k}_{2,1,1} & \dots &c^{n,k}_{2,1,a_{n,k}} & d^{n,k}_{2,1,a_{n,k}+1} & \dots &d^{n,k}_{2,1,s_{n,k}d} \\
      \vdots & \vdots & \vdots & \vdots & \vdots & \vdots   \\
      c^{n,k}_{s,d,1} & \dots & c^{n,k}_{s,d,a_{n,k}} & d^{n,k}_{s,d,a_{n,k}+1} & \dots &d^{n,k}_{s,d,s_{n,k}d}
    \end{array}
\right]
 := \bV.
 \label{hbconstruction:eqn3}
\end{equation}

\setlength{\tabcolsep}{0pt}
\tikzset{edge from parent/.style={draw,edge from parent
path={(\tikzparentnode.south)-- +(0,-8pt)-| (\tikzchildnode)}}
}
\begin{figure}[t]
\begin{center}
\begin{tabular}{c c c c} 

\begin{tikzpicture}[scale=.42]

\begin{scope}
[place/.style={circle,draw=blue!50,fill=blue!20,thick, inner sep=0pt,minimum size=1mm}]
%[place/.style={circle,draw=rufous,fill=rufous,thick, inner sep=0pt,minimum size=1mm}]
%[place/.style={regular polygon, regular polygon sides=5, draw=blue!50,fill=blue!20,thick, inner sep=0pt,minimum size=3mm}]
\begin{scope}[xshift = 2cm, yshift = 3.5cm, rotate=45]
\draw (0,0) to (0.5,0.8660);
\draw (0,0) to (-0.5,0.8660);
\draw (0,0) to (-1,0);
\draw (0,0) to (-0.5,-0.8660);
\draw (0,0) to (0.5,-0.8660);
\draw (0.5,0.8660) to (-0.5,0.8660); 
\draw (-0.5,0.8660) to (-1,0);
\draw (-1,0) to (-0.5,-0.8660);
\draw (-0.5,-0.8660) to (0.5,-0.8660);

\node at (0,0.5774) [place] {};
\node at (-0.50,0.2887) [place] {};
\node at (-0.50,-0.2887) [place] {};
\node at (0,-0.5774) [place] {};
\end{scope}

\begin{scope}[xshift = 2.5cm, yshift = 1.25cm, rotate = 145]
\draw (0,0) to (0.5,0.8660);
\draw (0,0) to (-0.5,0.8660);
\draw (0,0) to (-1,0);
\draw (0,0) to (-0.5,-0.8660);
\draw (0,0) to (0.5,-0.8660);
\draw (0.5,0.8660) to (-0.5,0.8660); 
\draw (-0.5,0.8660) to (-1,0);
\draw (-1,0) to (-0.5,-0.8660);
\draw (-0.5,-0.8660) to (0.5,-0.8660);

\node at (0,0.5774) [place] {};
\node at (-0.50,0.2887) [place] {};
\node at (-0.50,-0.2887) [place] {};
\node at (0,-0.5774) [place] {};
\end{scope}

\begin{scope}[xshift = 2cm, yshift = 7cm]
\draw (0,0) to (0.5,0.8660);
\draw (0,0) to (-0.5,0.8660);
\draw (0,0) to (-1,0);
\draw (0,0) to (-0.5,-0.8660);
\draw (0,0) to (0.5,-0.8660);
\draw (0.5,0.8660) to (-0.5,0.8660); 
\draw (-0.5,0.8660) to (-1,0);
\draw (-1,0) to (-0.5,-0.8660);
\draw (-0.5,-0.8660) to (0.5,-0.8660);

\node at (0,0.5774) [place] {};
\node at (-0.50,0.2887) [place] {};
\node at (-0.50,-0.2887) [place] {};
\node at (0,-0.5774) [place] {};
\end{scope}

\begin{scope}[xshift = 3cm, yshift = 7cm, rotate= 0]
\draw (0,0) to (0.5,0.8660);
\draw (0,0) to (-0.5,0.8660);
\draw (0,0) to (-1,0);
\draw (0,0) to (-0.5,-0.8660);
\draw (0,0) to (0.5,-0.8660);
\draw (0.5,0.8660) to (-0.5,0.8660); 
\draw (-0.5,0.8660) to (-1,0);
\draw (-1,0) to (-0.5,-0.8660);
\draw (-0.5,-0.8660) to (0.5,-0.8660);

\node at (0,0.5774) [place] {};
\node at (-0.50,0.2887) [place] {};
\node at (-0.50,-0.2887) [place] {};
\node at (0,-0.5774) [place] {};
\end{scope}

\begin{scope}[xshift = 5.943cm, yshift = 6.73cm, rotate = 180]
\draw (0,0) to (0.5,0.8660);
\draw (0,0) to (-0.5,0.8660);
\draw (0,0) to (-1,0);
\draw (0,0) to (-0.5,-0.8660);
\draw (0,0) to (0.5,-0.8660);
\draw (0.5,0.8660) to (-0.5,0.8660); 
\draw (-0.5,0.8660) to (-1,0);
\draw (-1,0) to (-0.5,-0.8660);
\draw (-0.5,-0.8660) to (0.5,-0.8660);

\node at (0,0.5774) [place] {};
\node at (-0.50,0.2887) [place] {};
\node at (-0.50,-0.2887) [place] {};
\node at (0,-0.5774) [place] {};
\end{scope}

\begin{scope}[xshift = 5.95cm, yshift = 5.01cm]
\draw (0,0) to (0.5,0.8660);
\draw (0,0) to (-0.5,0.8660);
\draw (0,0) to (-1,0);
\draw (0,0) to (-0.5,-0.8660);
\draw (0,0) to (0.5,-0.8660);
\draw (0.5,0.8660) to (-0.5,0.8660); 
\draw (-0.5,0.8660) to (-1,0);
\draw (-1,0) to (-0.5,-0.8660);
\draw (-0.5,-0.8660) to (0.5,-0.8660);

\node at (0,0.5774) [place] {};
\node at (-0.50,0.2887) [place] {};
\node at (-0.50,-0.2887) [place] {};
\node at (0,-0.5774) [place] {};
\end{scope}

\begin{scope}[xshift = 6.95cm, yshift = 3.28cm, rotate = 180]
\draw (0,0) to (0.5,0.8660);
\draw (0,0) to (-0.5,0.8660);
\draw (0,0) to (-1,0);
\draw (0,0) to (-0.5,-0.8660);
\draw (0,0) to (0.5,-0.8660);
\draw (0.5,0.8660) to (-0.5,0.8660); 
\draw (-0.5,0.8660) to (-1,0);
\draw (-1,0) to (-0.5,-0.8660);
\draw (-0.5,-0.8660) to (0.5,-0.8660);

\node at (0,0.5774) [place] {};
\node at (-0.50,0.2887) [place] {};
\node at (-0.50,-0.2887) [place] {};
\node at (0,-0.5774) [place] {};
\end{scope}

\begin{scope}[xshift = 5.95cm, yshift = 3.28cm, rotate = 180]
\draw (0,0) to (0.5,0.8660);
\draw (0,0) to (-0.5,0.8660);
\draw (0,0) to (-1,0);
\draw (0,0) to (-0.5,-0.8660);
\draw (0,0) to (0.5,-0.8660);
\draw (0.5,0.8660) to (-0.5,0.8660); 
\draw (-0.5,0.8660) to (-1,0);
\draw (-1,0) to (-0.5,-0.8660);
\draw (-0.5,-0.8660) to (0.5,-0.8660);

\node at (0,0.5774) [place] {};
\node at (-0.50,0.2887) [place] {};
\node at (-0.50,-0.2887) [place] {};
\node at (0,-0.5774) [place] {};
\end{scope}

    %\node at (7.5,7.5) [] {$B^{0}_0$};
    %\node at (1,5.5) [] {$B^{3}_{7}$};
    %\node at (2.75,5.5) [] {$B^{3}_{8}$};
    
    \draw[step=8,gray,very thin] (0, 0) grid (8, 8);

    %\draw (4.25,0) to (4.25,8);

    %\draw (0,5) to (4.25,5);
    %\draw (2.2,5) to (2.2,8);
    %\draw (0,2.4) to (4.25,2.4);

    %\draw (4.25,4.15) to (8,4.15);
    %\draw (4.25,5.87) to (8,5.87);
    %\draw (6.75,0) to (6.75,4.15);
    
  \end{scope}

\end{tikzpicture}  &
\begin{tikzpicture}[scale=.42]

\begin{scope}
[place/.style={circle,draw=blue!50,fill=blue!20,thick, inner sep=0pt,minimum size=1mm}]
%[place/.style={circle,draw=rufous,fill=rufous,thick, inner sep=0pt,minimum size=1mm}]
%[place/.style={regular polygon, regular polygon sides=5, draw=blue!50,fill=blue!20,thick, inner sep=0pt,minimum size=3mm}]
\begin{scope}[xshift = 2cm, yshift = 3.5cm, rotate=45]
\draw (0,0) to (0.5,0.8660);
\draw (0,0) to (-0.5,0.8660);
\draw (0,0) to (-1,0);
\draw (0,0) to (-0.5,-0.8660);
\draw (0,0) to (0.5,-0.8660);
\draw (0.5,0.8660) to (-0.5,0.8660); 
\draw (-0.5,0.8660) to (-1,0);
\draw (-1,0) to (-0.5,-0.8660);
\draw (-0.5,-0.8660) to (0.5,-0.8660);

\node at (0,0.5774) [place] {};
\node at (-0.50,0.2887) [place] {};
\node at (-0.50,-0.2887) [place] {};
\node at (0,-0.5774) [place] {};
\end{scope}

\begin{scope}[xshift = 2.5cm, yshift = 1.25cm, rotate = 145]
\draw (0,0) to (0.5,0.8660);
\draw (0,0) to (-0.5,0.8660);
\draw (0,0) to (-1,0);
\draw (0,0) to (-0.5,-0.8660);
\draw (0,0) to (0.5,-0.8660);
\draw (0.5,0.8660) to (-0.5,0.8660); 
\draw (-0.5,0.8660) to (-1,0);
\draw (-1,0) to (-0.5,-0.8660);
\draw (-0.5,-0.8660) to (0.5,-0.8660);

\node at (0,0.5774) [place] {};
\node at (-0.50,0.2887) [place] {};
\node at (-0.50,-0.2887) [place] {};
\node at (0,-0.5774) [place] {};
\end{scope}

\begin{scope}[xshift = 2cm, yshift = 7cm]
\draw (0,0) to (0.5,0.8660);
\draw (0,0) to (-0.5,0.8660);
\draw (0,0) to (-1,0);
\draw (0,0) to (-0.5,-0.8660);
\draw (0,0) to (0.5,-0.8660);
\draw (0.5,0.8660) to (-0.5,0.8660); 
\draw (-0.5,0.8660) to (-1,0);
\draw (-1,0) to (-0.5,-0.8660);
\draw (-0.5,-0.8660) to (0.5,-0.8660);

\node at (0,0.5774) [place] {};
\node at (-0.50,0.2887) [place] {};
\node at (-0.50,-0.2887) [place] {};
\node at (0,-0.5774) [place] {};
\end{scope}

\begin{scope}[xshift = 3cm, yshift = 7cm, rotate= 0]
\draw (0,0) to (0.5,0.8660);
\draw (0,0) to (-0.5,0.8660);
\draw (0,0) to (-1,0);
\draw (0,0) to (-0.5,-0.8660);
\draw (0,0) to (0.5,-0.8660);
\draw (0.5,0.8660) to (-0.5,0.8660); 
\draw (-0.5,0.8660) to (-1,0);
\draw (-1,0) to (-0.5,-0.8660);
\draw (-0.5,-0.8660) to (0.5,-0.8660);

\node at (0,0.5774) [place] {};
\node at (-0.50,0.2887) [place] {};
\node at (-0.50,-0.2887) [place] {};
\node at (0,-0.5774) [place] {};
\end{scope}

\begin{scope}[xshift = 5.943cm, yshift = 6.73cm, rotate = 180]
\draw (0,0) to (0.5,0.8660);
\draw (0,0) to (-0.5,0.8660);
\draw (0,0) to (-1,0);
\draw (0,0) to (-0.5,-0.8660);
\draw (0,0) to (0.5,-0.8660);
\draw (0.5,0.8660) to (-0.5,0.8660); 
\draw (-0.5,0.8660) to (-1,0);
\draw (-1,0) to (-0.5,-0.8660);
\draw (-0.5,-0.8660) to (0.5,-0.8660);

\node at (0,0.5774) [place] {};
\node at (-0.50,0.2887) [place] {};
\node at (-0.50,-0.2887) [place] {};
\node at (0,-0.5774) [place] {};
\end{scope}

\begin{scope}[xshift = 5.95cm, yshift = 5.01cm]
\draw (0,0) to (0.5,0.8660);
\draw (0,0) to (-0.5,0.8660);
\draw (0,0) to (-1,0);
\draw (0,0) to (-0.5,-0.8660);
\draw (0,0) to (0.5,-0.8660);
\draw (0.5,0.8660) to (-0.5,0.8660); 
\draw (-0.5,0.8660) to (-1,0);
\draw (-1,0) to (-0.5,-0.8660);
\draw (-0.5,-0.8660) to (0.5,-0.8660);

\node at (0,0.5774) [place] {};
\node at (-0.50,0.2887) [place] {};
\node at (-0.50,-0.2887) [place] {};
\node at (0,-0.5774) [place] {};
\end{scope}

\begin{scope}[xshift = 6.95cm, yshift = 3.28cm, rotate = 180]
\draw (0,0) to (0.5,0.8660);
\draw (0,0) to (-0.5,0.8660);
\draw (0,0) to (-1,0);
\draw (0,0) to (-0.5,-0.8660);
\draw (0,0) to (0.5,-0.8660);
\draw (0.5,0.8660) to (-0.5,0.8660); 
\draw (-0.5,0.8660) to (-1,0);
\draw (-1,0) to (-0.5,-0.8660);
\draw (-0.5,-0.8660) to (0.5,-0.8660);

\node at (0,0.5774) [place] {};
\node at (-0.50,0.2887) [place] {};
\node at (-0.50,-0.2887) [place] {};
\node at (0,-0.5774) [place] {};
\end{scope}

\begin{scope}[xshift = 5.95cm, yshift = 3.28cm, rotate = 180]
\draw (0,0) to (0.5,0.8660);
\draw (0,0) to (-0.5,0.8660);
\draw (0,0) to (-1,0);
\draw (0,0) to (-0.5,-0.8660);
\draw (0,0) to (0.5,-0.8660);
\draw (0.5,0.8660) to (-0.5,0.8660); 
\draw (-0.5,0.8660) to (-1,0);
\draw (-1,0) to (-0.5,-0.8660);
\draw (-0.5,-0.8660) to (0.5,-0.8660);

\node at (0,0.5774) [place] {};
\node at (-0.50,0.2887) [place] {};
\node at (-0.50,-0.2887) [place] {};
\node at (0,-0.5774) [place] {};
\end{scope}

    %\node at (7.5,7.5) [] {$B^{0}_0$};
    %\node at (1,5.5) [] {$B^{3}_{7}$};
    %\node at (2.75,5.5) [] {$B^{3}_{8}$};
    
    \draw[step=8,gray,very thin] (0, 0) grid (8, 8);

    \draw (4.25,0) to (4.25,8);

    %\draw (0,5) to (4.25,5);
    %\draw (2.2,5) to (2.2,8);
    %\draw (0,2.4) to (4.25,2.4);

    %\draw (4.25,4.15) to (8,4.15);
    %\draw (4.25,5.87) to (8,5.87);
    %\draw (6.75,0) to (6.75,4.15);
    
  \end{scope}

\end{tikzpicture}  &
\begin{tikzpicture}[scale=.42]

\begin{scope}
[place/.style={circle,draw=blue!50,fill=blue!20,thick, inner sep=0pt,minimum size=1mm}]
%[place/.style={circle,draw=rufous,fill=rufous,thick, inner sep=0pt,minimum size=1mm}]
%[place/.style={regular polygon, regular polygon sides=5, draw=blue!50,fill=blue!20,thick, inner sep=0pt,minimum size=3mm}]
\begin{scope}[xshift = 2cm, yshift = 3.5cm, rotate=45]
\draw (0,0) to (0.5,0.8660);
\draw (0,0) to (-0.5,0.8660);
\draw (0,0) to (-1,0);
\draw (0,0) to (-0.5,-0.8660);
\draw (0,0) to (0.5,-0.8660);
\draw (0.5,0.8660) to (-0.5,0.8660); 
\draw (-0.5,0.8660) to (-1,0);
\draw (-1,0) to (-0.5,-0.8660);
\draw (-0.5,-0.8660) to (0.5,-0.8660);

\node at (0,0.5774) [place] {};
\node at (-0.50,0.2887) [place] {};
\node at (-0.50,-0.2887) [place] {};
\node at (0,-0.5774) [place] {};
\end{scope}

\begin{scope}[xshift = 2.5cm, yshift = 1.25cm, rotate = 145]
\draw (0,0) to (0.5,0.8660);
\draw (0,0) to (-0.5,0.8660);
\draw (0,0) to (-1,0);
\draw (0,0) to (-0.5,-0.8660);
\draw (0,0) to (0.5,-0.8660);
\draw (0.5,0.8660) to (-0.5,0.8660); 
\draw (-0.5,0.8660) to (-1,0);
\draw (-1,0) to (-0.5,-0.8660);
\draw (-0.5,-0.8660) to (0.5,-0.8660);

\node at (0,0.5774) [place] {};
\node at (-0.50,0.2887) [place] {};
\node at (-0.50,-0.2887) [place] {};
\node at (0,-0.5774) [place] {};
\end{scope}

\begin{scope}[xshift = 2cm, yshift = 7cm]
\draw (0,0) to (0.5,0.8660);
\draw (0,0) to (-0.5,0.8660);
\draw (0,0) to (-1,0);
\draw (0,0) to (-0.5,-0.8660);
\draw (0,0) to (0.5,-0.8660);
\draw (0.5,0.8660) to (-0.5,0.8660); 
\draw (-0.5,0.8660) to (-1,0);
\draw (-1,0) to (-0.5,-0.8660);
\draw (-0.5,-0.8660) to (0.5,-0.8660);

\node at (0,0.5774) [place] {};
\node at (-0.50,0.2887) [place] {};
\node at (-0.50,-0.2887) [place] {};
\node at (0,-0.5774) [place] {};
\end{scope}

\begin{scope}[xshift = 3cm, yshift = 7cm, rotate= 0]
\draw (0,0) to (0.5,0.8660);
\draw (0,0) to (-0.5,0.8660);
\draw (0,0) to (-1,0);
\draw (0,0) to (-0.5,-0.8660);
\draw (0,0) to (0.5,-0.8660);
\draw (0.5,0.8660) to (-0.5,0.8660); 
\draw (-0.5,0.8660) to (-1,0);
\draw (-1,0) to (-0.5,-0.8660);
\draw (-0.5,-0.8660) to (0.5,-0.8660);

\node at (0,0.5774) [place] {};
\node at (-0.50,0.2887) [place] {};
\node at (-0.50,-0.2887) [place] {};
\node at (0,-0.5774) [place] {};
\end{scope}

\begin{scope}[xshift = 5.943cm, yshift = 6.73cm, rotate = 180]
\draw (0,0) to (0.5,0.8660);
\draw (0,0) to (-0.5,0.8660);
\draw (0,0) to (-1,0);
\draw (0,0) to (-0.5,-0.8660);
\draw (0,0) to (0.5,-0.8660);
\draw (0.5,0.8660) to (-0.5,0.8660); 
\draw (-0.5,0.8660) to (-1,0);
\draw (-1,0) to (-0.5,-0.8660);
\draw (-0.5,-0.8660) to (0.5,-0.8660);

\node at (0,0.5774) [place] {};
\node at (-0.50,0.2887) [place] {};
\node at (-0.50,-0.2887) [place] {};
\node at (0,-0.5774) [place] {};
\end{scope}

\begin{scope}[xshift = 5.95cm, yshift = 5.01cm]
\draw (0,0) to (0.5,0.8660);
\draw (0,0) to (-0.5,0.8660);
\draw (0,0) to (-1,0);
\draw (0,0) to (-0.5,-0.8660);
\draw (0,0) to (0.5,-0.8660);
\draw (0.5,0.8660) to (-0.5,0.8660); 
\draw (-0.5,0.8660) to (-1,0);
\draw (-1,0) to (-0.5,-0.8660);
\draw (-0.5,-0.8660) to (0.5,-0.8660);

\node at (0,0.5774) [place] {};
\node at (-0.50,0.2887) [place] {};
\node at (-0.50,-0.2887) [place] {};
\node at (0,-0.5774) [place] {};
\end{scope}

\begin{scope}[xshift = 6.95cm, yshift = 3.28cm, rotate = 180]
\draw (0,0) to (0.5,0.8660);
\draw (0,0) to (-0.5,0.8660);
\draw (0,0) to (-1,0);
\draw (0,0) to (-0.5,-0.8660);
\draw (0,0) to (0.5,-0.8660);
\draw (0.5,0.8660) to (-0.5,0.8660); 
\draw (-0.5,0.8660) to (-1,0);
\draw (-1,0) to (-0.5,-0.8660);
\draw (-0.5,-0.8660) to (0.5,-0.8660);

\node at (0,0.5774) [place] {};
\node at (-0.50,0.2887) [place] {};
\node at (-0.50,-0.2887) [place] {};
\node at (0,-0.5774) [place] {};
\end{scope}

\begin{scope}[xshift = 5.95cm, yshift = 3.28cm, rotate = 180]
\draw (0,0) to (0.5,0.8660);
\draw (0,0) to (-0.5,0.8660);
\draw (0,0) to (-1,0);
\draw (0,0) to (-0.5,-0.8660);
\draw (0,0) to (0.5,-0.8660);
\draw (0.5,0.8660) to (-0.5,0.8660); 
\draw (-0.5,0.8660) to (-1,0);
\draw (-1,0) to (-0.5,-0.8660);
\draw (-0.5,-0.8660) to (0.5,-0.8660);

\node at (0,0.5774) [place] {};
\node at (-0.50,0.2887) [place] {};
\node at (-0.50,-0.2887) [place] {};
\node at (0,-0.5774) [place] {};
\end{scope}

    %\node at (7.5,7.5) [] {$B^{0}_0$};
    %\node at (1,5.5) [] {$B^{3}_{7}$};
    %\node at (2.75,5.5) [] {$B^{3}_{8}$};
    
    \draw[step=8,gray,very thin] (0, 0) grid (8, 8);

    \draw[gray,dashed] (4.25,0) to (4.25,8);

    \draw (0,5) to (4.25,5);
    %\draw (2.2,5) to (2.2,8);
    %\draw (0,2.4) to (4.25,2.4);

    \draw (4.25,4.15) to (8,4.15);
    %\draw (4.25,5.87) to (8,5.87);
    %\draw (6.75,0) to (6.75,4.15);
    
  \end{scope}

\end{tikzpicture}  &
\begin{tikzpicture}[scale=.42]

\begin{scope}
[place/.style={circle,draw=blue!50,fill=blue!20,thick, inner sep=0pt,minimum size=1mm}]
%[place/.style={circle,draw=rufous,fill=rufous,thick, inner sep=0pt,minimum size=1mm}]
%[place/.style={regular polygon, regular polygon sides=5, draw=blue!50,fill=blue!20,thick, inner sep=0pt,minimum size=3mm}]
\begin{scope}[xshift = 2cm, yshift = 3.5cm, rotate=45]
\draw (0,0) to (0.5,0.8660);
\draw (0,0) to (-0.5,0.8660);
\draw (0,0) to (-1,0);
\draw (0,0) to (-0.5,-0.8660);
\draw (0,0) to (0.5,-0.8660);
\draw (0.5,0.8660) to (-0.5,0.8660); 
\draw (-0.5,0.8660) to (-1,0);
\draw (-1,0) to (-0.5,-0.8660);
\draw (-0.5,-0.8660) to (0.5,-0.8660);

\node at (0,0.5774) [place] {};
\node at (-0.50,0.2887) [place] {};
\node at (-0.50,-0.2887) [place] {};
\node at (0,-0.5774) [place] {};
\end{scope}

\begin{scope}[xshift = 2.5cm, yshift = 1.25cm, rotate = 145]
\draw (0,0) to (0.5,0.8660);
\draw (0,0) to (-0.5,0.8660);
\draw (0,0) to (-1,0);
\draw (0,0) to (-0.5,-0.8660);
\draw (0,0) to (0.5,-0.8660);
\draw (0.5,0.8660) to (-0.5,0.8660); 
\draw (-0.5,0.8660) to (-1,0);
\draw (-1,0) to (-0.5,-0.8660);
\draw (-0.5,-0.8660) to (0.5,-0.8660);

\node at (0,0.5774) [place] {};
\node at (-0.50,0.2887) [place] {};
\node at (-0.50,-0.2887) [place] {};
\node at (0,-0.5774) [place] {};
\end{scope}

\begin{scope}[xshift = 2cm, yshift = 7cm]
\draw (0,0) to (0.5,0.8660);
\draw (0,0) to (-0.5,0.8660);
\draw (0,0) to (-1,0);
\draw (0,0) to (-0.5,-0.8660);
\draw (0,0) to (0.5,-0.8660);
\draw (0.5,0.8660) to (-0.5,0.8660); 
\draw (-0.5,0.8660) to (-1,0);
\draw (-1,0) to (-0.5,-0.8660);
\draw (-0.5,-0.8660) to (0.5,-0.8660);

\node at (0,0.5774) [place] {};
\node at (-0.50,0.2887) [place] {};
\node at (-0.50,-0.2887) [place] {};
\node at (0,-0.5774) [place] {};
\end{scope}

\begin{scope}[xshift = 3cm, yshift = 7cm, rotate= 0]
\draw (0,0) to (0.5,0.8660);
\draw (0,0) to (-0.5,0.8660);
\draw (0,0) to (-1,0);
\draw (0,0) to (-0.5,-0.8660);
\draw (0,0) to (0.5,-0.8660);
\draw (0.5,0.8660) to (-0.5,0.8660); 
\draw (-0.5,0.8660) to (-1,0);
\draw (-1,0) to (-0.5,-0.8660);
\draw (-0.5,-0.8660) to (0.5,-0.8660);

\node at (0,0.5774) [place] {};
\node at (-0.50,0.2887) [place] {};
\node at (-0.50,-0.2887) [place] {};
\node at (0,-0.5774) [place] {};
\end{scope}

\begin{scope}[xshift = 5.943cm, yshift = 6.73cm, rotate = 180]
\draw (0,0) to (0.5,0.8660);
\draw (0,0) to (-0.5,0.8660);
\draw (0,0) to (-1,0);
\draw (0,0) to (-0.5,-0.8660);
\draw (0,0) to (0.5,-0.8660);
\draw (0.5,0.8660) to (-0.5,0.8660); 
\draw (-0.5,0.8660) to (-1,0);
\draw (-1,0) to (-0.5,-0.8660);
\draw (-0.5,-0.8660) to (0.5,-0.8660);

\node at (0,0.5774) [place] {};
\node at (-0.50,0.2887) [place] {};
\node at (-0.50,-0.2887) [place] {};
\node at (0,-0.5774) [place] {};
\end{scope}

\begin{scope}[xshift = 5.95cm, yshift = 5.01cm]
\draw (0,0) to (0.5,0.8660);
\draw (0,0) to (-0.5,0.8660);
\draw (0,0) to (-1,0);
\draw (0,0) to (-0.5,-0.8660);
\draw (0,0) to (0.5,-0.8660);
\draw (0.5,0.8660) to (-0.5,0.8660); 
\draw (-0.5,0.8660) to (-1,0);
\draw (-1,0) to (-0.5,-0.8660);
\draw (-0.5,-0.8660) to (0.5,-0.8660);

\node at (0,0.5774) [place] {};
\node at (-0.50,0.2887) [place] {};
\node at (-0.50,-0.2887) [place] {};
\node at (0,-0.5774) [place] {};
\end{scope}

\begin{scope}[xshift = 6.95cm, yshift = 3.28cm, rotate = 180]
\draw (0,0) to (0.5,0.8660);
\draw (0,0) to (-0.5,0.8660);
\draw (0,0) to (-1,0);
\draw (0,0) to (-0.5,-0.8660);
\draw (0,0) to (0.5,-0.8660);
\draw (0.5,0.8660) to (-0.5,0.8660); 
\draw (-0.5,0.8660) to (-1,0);
\draw (-1,0) to (-0.5,-0.8660);
\draw (-0.5,-0.8660) to (0.5,-0.8660);

\node at (0,0.5774) [place] {};
\node at (-0.50,0.2887) [place] {};
\node at (-0.50,-0.2887) [place] {};
\node at (0,-0.5774) [place] {};
\end{scope}

\begin{scope}[xshift = 5.95cm, yshift = 3.28cm, rotate = 180]
\draw (0,0) to (0.5,0.8660);
\draw (0,0) to (-0.5,0.8660);
\draw (0,0) to (-1,0);
\draw (0,0) to (-0.5,-0.8660);
\draw (0,0) to (0.5,-0.8660);
\draw (0.5,0.8660) to (-0.5,0.8660); 
\draw (-0.5,0.8660) to (-1,0);
\draw (-1,0) to (-0.5,-0.8660);
\draw (-0.5,-0.8660) to (0.5,-0.8660);

\node at (0,0.5774) [place] {};
\node at (-0.50,0.2887) [place] {};
\node at (-0.50,-0.2887) [place] {};
\node at (0,-0.5774) [place] {};
\end{scope}

    %\node at (7.5,7.5) [] {$B^{0}_0$};
    %\node at (1,5.5) [] {$B^{3}_{7}$};
    %\node at (2.75,5.5) [] {$B^{3}_{8}$};
    
    \draw[step=8,gray,very thin] (0, 0) grid (8, 8);

    \draw[gray,dashed] (4.25,0) to (4.25,8);

    \draw[gray,dashed] (0,5) to (4.25,5);
    \draw (2.2,5) to (2.2,8);
    \draw (0,2.4) to (4.25,2.4);

    %\draw (0,5) to (3.25,5);

    \draw[gray,dashed] (4.25,4.15) to (8,4.15);
    \draw (4.25,5.87) to (8,5.87);
    %\draw (5,4.15) to (5,8);
    \draw (6.75,0) to (6.75,4.15);
    
  \end{scope}

\end{tikzpicture} 
\\ 
\imagetop{\begin{tikzpicture}[scale=0.6]
    %\node[anchor=center] at (0, -4.5) {$$};
    %\node[anchor=center] at (0,   10) {$$};
\begin{scope}[xshift=5cm, yshift=5cm,
place/.style={circle,draw=blue!50,fill=blue!20,thick,
      inner sep=0pt,minimum size=1.5mm},
placer/.style={circle,draw=darkgreen!50,
  preaction={fill=olivine,opacity = 0.5}, thick,inner
  sep=0pt,minimum size=1.5mm}, ]

  %placer/.style={circle,draw=blue!50,
  %preaction={fill=darkgreen!60,fill opacity=0.5}, thick,inner
  %sep=0pt,minimum size=1.5mm}, ]

%\filldraw[fill={rgb:red,143;green,188;blue,143},semitransparent, 
%      thick] (0, 0) rectangle (16, 16);
  
\Tree [.\node[place]{$B^{0}_{0}$}; 
]

%\buildMeshBW{(0.3,0.3);(1.5,1);(4,0);(4.5,2.5);(1.81,2.14);(2.5,0.5);(2.8,1.5)}

\end{scope}
\end{tikzpicture}}  &
\imagetop{\begin{tikzpicture}[scale=0.6,
    grow                    = down,
    sibling distance        = 6em,
    %level 1/.style          = {sibling distance=0.5cm},
    %level 2/.style          = {sibling distance=0.1cm},
    %level 3/.style          = {sibling distance=0.1cm},
    level distance          = 1.2cm,
    edge from parent/.style = {draw, edge from parent path={(\tikzparentnode) -- (\tikzchildnode)}},
    %sloped
    ]
    %\node[anchor=center] at (0, -4.5) {$$};
    %\node[anchor=center] at (0,   10) {$$};
\begin{scope}[xshift=5cm, yshift=5cm,
place/.style={circle,draw=blue!50,fill=blue!20,thick,
      inner sep=0pt,minimum size=6mm},
placer/.style={circle,draw=darkgreen!50,
  preaction={fill=olivine,opacity = 0.5}, thick,inner
  sep=0pt,minimum size=1.5mm}, ]

  %placer/.style={circle,draw=blue!50,
  %preaction={fill=darkgreen!60,fill opacity=0.5}, thick,inner
  %sep=0pt,minimum size=1.5mm}, ]

%\filldraw[fill={rgb:red,143;green,188;blue,143},semitransparent, 
%      thick] (0, 0) rectangle (16, 16);
  
\Tree [.\node[place]{$B^{0}_{0}$}; 
             [.\node[placer]{$B^{1}_{1}$};
             ]                                        
             [.\node[placer]{$B^{1}_{2}$};
             ]      
]

%\buildMeshBW{(0.3,0.3);(1.5,1);(4,0);(4.5,2.5);(1.81,2.14);(2.5,0.5);(2.8,1.5)}

\end{scope}
\end{tikzpicture}}  &
\imagetop{\begin{tikzpicture}[scale=0.6,
    grow                    = down,
    sibling distance        = 2em,
    %level 1/.style          = {sibling distance=2cm},
    %level 2/.style          = {sibling distance=0.4cm},
    %level 3/.style          = {sibling distance=0.1cm},
    level distance          = 1.2cm,
    edge from parent/.style = {draw, edge from parent path={(\tikzparentnode) -- (\tikzchildnode)}}
    %sloped
    ]
  
    %\node[anchor=center] at (0, -4.5) {$$};
    %\node[anchor=center] at (0,   10) {$$};
\begin{scope}[xshift=5cm, yshift=5cm,
place/.style={circle,draw=blue!50,fill=blue!20,thick,
      inner sep=0pt,minimum size=1.5mm},
placer/.style={circle,draw=darkgreen!50,
  preaction={fill=olivine,opacity = 0.5}, thick,inner
  sep=0pt,minimum size=1.5mm}, ]

  %placer/.style={circle,draw=blue!50,
  %preaction={fill=darkgreen!60,fill opacity=0.5}, thick,inner
  %sep=0pt,minimum size=1.5mm}, ]

%\filldraw[fill={rgb:red,143;green,188;blue,143},semitransparent, 
%      thick] (0, 0) rectangle (16, 16);
  
\Tree [.\node[place]{$B^{0}_{0}$}; 
             [.\node[place]{$B^{1}_{1}$};
                    [.\node[placer]{$B^{2}_{3}$}; 
                    ]       
                    [.\node[placer]{$B^{2}_{4}$}; 
                    ] 
             ]                                        
             [.\node[place]{$B^{1}_{2}$};
                          [.\node[placer]{$B^{2}_{5}$};
                           ]
                           [.\node[placer]{$B^{2}_{6}$}; 
                           ] 
                                          ]
]

%\buildMeshBW{(0.3,0.3);(1.5,1);(4,0);(4.5,2.5);(1.81,2.14);(2.5,0.5);(2.8,1.5)}

\end{scope}
\end{tikzpicture}}  &
\imagetop{\begin{tikzpicture}[scale=0.6,
   grow                    = down,
    sibling distance        = 0.4em,
    %level 1/.style          = {sibling distance=1cm},
    %level 2/.style          = {sibling distance=0.1cm},
    %level 3/.style          = {sibling distance=0.1cm},
    level distance          = 1.2cm,
    edge from parent/.style = {draw, edge from parent path={(\tikzparentnode) -- (\tikzchildnode)}},
    %sloped
]
    %\node[anchor=center] at (0, -4.5) {$$};
    %\node[anchor=center] at (0,   10) {$$};
\begin{scope}[xshift=5cm, yshift=5cm,
place/.style={circle,draw=blue!50,fill=blue!20,thick,
      inner sep=0pt,minimum size=1.5mm},
placer/.style={circle,draw=darkgreen!50,
  preaction={fill=olivine,opacity = 0.5}, thick,inner
  sep=0pt,minimum size=1.5mm}, ]

  %placer/.style={circle,draw=blue!50,
  %preaction={fill=darkgreen!60,fill opacity=0.5}, thick,inner
  %sep=0pt,minimum size=1.5mm}, ]

%\filldraw[fill={rgb:red,143;green,188;blue,143},semitransparent, 
%      thick] (0, 0) rectangle (16, 16);
  
\Tree [.\node[place]{$B^{0}_{0}$}; 
             [.\node[place]{$B^{1}_{1}$};
                    [.\node[place]{$B^{2}_{3}$}; 
                           [.\node[placer]{$B^{3}_{7}$};]
                           [.\node[placer]{$B^{3}_{8}$};] 
                    ]       
                    [.\node[place]{$B^{2}_{4}$}; 
                           [.\node[placer]{$B^{3}_{9}$};] 
                           [.\node[placer]{$B^{3}_{10}$};] 
                    ] 
             ]                                        
             [.\node[place]{$B^{1}_{2}$};
                          [.\node[place]{$B^{2}_{5}$};
                                  [.\node[placer]{$B^{3}_{11}$};] 
                                  [.\node[placer]{$B^{3}_{12}$};] 
                           ]
                           [.\node[place]{$B^{2}_{6}$}; 
                                  [.\node[placer]{$B^{3}_{13}$};] 
                                  [.\node[placer]{$B^{3}_{14}$};] 
                           ] 
                                          ]
]

%\buildMeshBW{(0.3,0.3);(1.5,1);(4,0);(4.5,2.5);(1.81,2.14);(2.5,0.5);(2.8,1.5)}

\end{scope}
\end{tikzpicture}}

\end{tabular}
\end{center}
\caption{Cartoon example of the construction of a kd-tree from the triangular simplices in
$\mcT$.}
\label{MLRLE:fig2}
\end{figure}

\begin{lemma}
Under this choice (equation \eqref{hbconstruction:eqn3}),
$\bpsi^{n,k}_{a_{n,k}+1}, \dots, \bpsi^{n,k}_{s}$ satisfy equation
\eqref{hbconstruction:eqn1}.
\end{lemma}

\begin{proof}
Following the argument in \cite{Tausch2001, Castrillon2016a}, let

%\[
\begin{equation*}
\begin{split}
\bN^{n,k} := 
%\begin{array}{c}
&\left[
\begin{matrix}
 ( \phi_{1},  \bphi^{n,k}_{1} )  & 
\dots & 
 ( \phi_{1},  \bphi^{n,k}_{a_{n,k}} )   \\
\vdots & \ddots & \vdots \\
 ( \phi_{M},  \bphi^{n,k}_{1} )  & 
\dots & 
( \phi_{M} , \bphi^{n,k}_{a_{n,k}} )   \\
\end{matrix} 
\right|
\left.
\begin{matrix}
 ( \phi_{1},  \bpsi^{n,k}_{a_{n,k} + 1} )  & 
\dots & 
 ( \phi_{1},  \bpsi^{n,k}_{s_{n,k}d} )   \\
\vdots & \ddots & \vdots \\
 ( \phi_{M},  \bpsi^{n,k}_{a_{n,k} + 1} )  & 
\dots & 
 ( \phi_{M} , \bpsi^{n,k}_{s_{n,k}d} )   \\
\end{matrix}
\right].
\end{split}
%\end{array}
%\]
\end{equation*}
Thus from the choice of coefficients $c^{n,k}_{i,h,j}$ and
$d^{n,k}_{i,h,j}$ we have that
$\bN^{n,k} = \bM^{n,k} \bV$.
From equation \eqref{hbconstruction:eqn2} we conclude that
$\bN^{n,k} = \bM^{n,k} \bV = \bU \bD$.
Since $M^{n,k}$ is of rank $a_{n,k}$ we have 
$\bD = [ \bSigma \, |\, \0 ]$,
where $\bSigma \in \R^{M \times a_{n,k}}$ is a diagonal matrix with
the non-zero singular values of $\bM^{n,k}$ and $\0 \in \R^{M \times
(s_{n,k}d - a_{n,k}) }$ is the zero matrix.  Thus $\bU \bD =
[\bU \bSigma\,|\,\0]$ and $\bN^{n,k} = [\bU \bSigma\,|\,\0]$. It
follows that columns $a_{n,k}+1,\dots,s$ of $\bV$ form an
orthonormal basis of the nullspace of $\bM^{n,k}$ and therefore
$\bpsi^{n,k}_{a_{n,k}+1}, \dots, \bpsi^{n,k}_s$ satisfy
equation \eqref{hbconstruction:eqn1}.
\end{proof}

\begin{lemma}
Let $D_k^{n} :=\{\bpsi^{n,k}_{a_{n,k}+1}, \dots, 
\bpsi^{n,k}_{s} \}$ and  $C_k^{n}:=\{ \bphi^{n,k}_{1}, 
\dots, \bphi^{n,k}_{a_{n,k}}\}$.  Then $D_k^{n} \cup C^{n}_k$ form an orthonormal set.
\end{lemma}

\begin{proof}
This follows from the fact that $\bV$ is a unitary matrix and from
the choice of coefficients from equation \eqref{hbconstruction:eqn3}.
\end{proof}

For every cell $B^n_{k} \in \mcB^n$ in the tree $\bT$ at level $n$, the
SVD orthogonalisation process is repeated. Let $\mcD^{n}$ be the
collection of orthonormal basis functions such that $\mcD^{n}
= \cup_{B^n_k \in \mcB^n} D_k^{n}$, and the multilevel space be defined
as $\bW_{n} := \mbox{\emph{span}}_{D_k^{n} \in \mcD^{n}}\{ D_{k}^{n}\}$
and $\bW_{n,k} := \mbox{\emph{span}}_{\bpsi^{n,k}_{g} \in D_k^{n} }
\{ \bpsi^{n,k}_{g} \}$ for any $B^n_k \in \mcB^n$.

\begin{algorithm}[htbp]
\caption{\textsc{LeafBasis}$(B_k^l, \{\varphi_i\}_{i=1}^M, \mcE^l_k )$}
\begin{algorithmic}[1]
\Require Leaf cell $B_k^l$, KL modes $\{\varphi_i\}_{i=1}^M$, $\mcE^l_k$ 
\Ensure Local sets $D_k^n$ (detail/nullspace basis, contributes to $\bW_n$) and $C_k^n$ (coarse part, passed upward)
\State Gather indices of simplices/barycenters in $B_k^n$ 
%and build $E^{d}_{n,k}\subset E$
\State Assemble $\bM^{n,k}\in \R^{M\times (s_{n,k}d)}$ with entries $(\bM^{n,k})_{i,(h,r)}= \langle \varphi_i, \bchi_r^h \rangle$
\State Compute SVD: $\bM^{n,k}=UDV^\top$
\State $a_{n,k}\gets \mbox{rank} (\bM^{n,k})$
\State Form linear combinations of $\mcE^{d}_{n}$ using columns of $V$:
\State \hspace{0.75cm} $\{\phi^{n,k}_j\}_{j=1}^{a_{n,k}}$ and $\{\psi^{n,k}_j\}_{j=a_{n,k}+1}^{s_{n,k}d}$
\State $D_k^n \gets \{\psi^{n,k}_{a_{n,k}+1},\dots,\psi^{n,k}_{s_{n,k}d}\}$
\State $C_k^n \gets \{\phi^{n,k}_{1},\dots,\phi^{n,k}_{a_{n,k}}\}$
\State \Return $(D_k^n, C_k^n)$
\end{algorithmic}
\label{leafbasis}
\end{algorithm}

\begin{remark}
Note that a leaf is not necessarily at the highest level $n$. Depending on the distribution of the barycenters, a leaf can be located at a lower level $\ell \leq n$. In Algorithm \ref{leafbasis} the pseudocode for the construction of the multilevel basis is shown for any leaf at level $\ell \leq n$.
\end{remark}

\emph{Construction of non-leaf multilevel basis.}
Although $C_k^{n}$ forms an orthonormal set, these functions are not
in general orthogonal to $V_0$. However, it is clear that
$\cup_{B^n_k \in \mcB^n} C^{n}_k$ form an orthonormal set.  The next
step is to work up the tree.  For any two sibling cells denoted as
$B^{n}_{\tt{left}}$ and $B^{n}_{\tt{right}}$ and corresponding basis
functions $C^{n}_{\tt{left}}$ and $C^{n}_{\tt{right}}$ at level $n$,
let $\mcE^{n-1}_{k} := C^{n}_{\tt{left}} \cup C^{n}_{\tt{right}}$ for
some index $k$ and let $B^{n-1}_{k}$ be the corresponding cell at
level $\mcB^{n-1}$. The orthogonalisation process is repeated for the
functions in $\mcE^{n-1}_{k}$. Rewrite the elements in
$\mcE^{n-1}_{k}$ as
$\{\bchi^{n-1}_{1}, \dots, \bchi^{n-1}_{s_{n-1,k}} \}$ and form the set
of equations
%\[
\begin{equation*}
\begin{split}
\bphi^{n-1,k} _{j}
&:= \sum_{i = 1}^{s_{n-1,k}} 
c^{n-1,k} _{i,j}
  \bchi^{n-1}_i, j \in \{1, \dots, a_{n-1,k}\}; \\
\bpsi^{n-1,k}_{j}
&:= \sum_{i = 1}^{s_{n-1,k}} 
d^{n-1,k}_{i,j} \bchi^{n-1}_{i},
j \in \{a_{n-1,k}+1, \dots, s_{n-1,k}\}.
\end{split}
%\]
\end{equation*}
We can form the matrix
\[
\bM^{n-1,k} := 
\begin{bmatrix}
 ( \phi_{1}(\bx),  \bchi^{n-1}_1(\bx))  & 
\dots & 
 ( \phi_{1}(\bx),  \bchi^{n-1}_{s_{n-1,k}}(\bx))  
\\
( \phi_{2}(\bx),  \bchi^{n-1}_1(\bx))  & 
\dots & 
 ( \phi_{2}(\bx),  \bchi^{n-1}_{s_{n-1,k}}(\bx))  &
 \\
 \vdots & \vdots & \vdots
\\
( \phi_{M}(\bx),  \bchi^{n-1}_1(\bx))  & 
\dots & 
 ( \phi_{M}(\bx),  \bchi^{n-1}_{s_{n-1,k}}(\bx))  
 \end{bmatrix}
\]
and apply the SVD $\bM^{n-1,k} = \bU \bD \bV\T$. Suppose that
$a_{n,k}$ is the rank of the matrix $\bM^{n-1,k}$.  Then under the
choice
\[
  \left[ \begin{array}{ccc|ccc}
      c^{n-1,k}_{1,1} & \dots &c^{n-1,k}_{1,a_{n,k}} & d^{n-1,k}_{1,a_{n-1,k}+1} & \dots &d^{n-1,k}_{1,s_{n-1,k}} \\
      c^{n-1,k}_{2,1} & \dots &c^{n-1,k}_{2,a_{n,k}} & d^{n-1,k}_{2,a_{n-1,k}+1} & \dots &d^{n-1,k}_{2,s_{n-1,k}} \\
      \vdots & \vdots & \vdots & \vdots & \vdots & \vdots   \\
      c^{n-1,k}_{s_{n-1,k},1} & \dots &c^{n-1,k}_{s_{n-1,k},a_{n,k}} & d^{n-1,k}_{s_{n-1,k},a_{n-1,k}+1} & \dots &d^{n-1,k}_{s_{n-1,k},s_{n-1,k}}
    \end{array}
\right]
 := \bV,
\]
for $i = 1,\dots,M$ and $j = a_{n-1,k}+1, \dots, s_{n-1,k}$ we have
that $\int_{U}
\phi_i(\bx)\T  \bpsi^{n-1,k}_{j}(\bx) \, \mbox{d} \bx 
= 0$.

For every cell $B^{n-1}_{k} \in \mcB^n$ in the tree $\bT$ at level
$n-1$ the SVD orthogonalisation process is repeated. Let $\mcD^{n-1}$
be the collection of orthonormal basis functions such that $\mcD^{n-1}
= \cup_{B^{n-1}_k \in \mcB^{n-1}} D_k^{n-1}$ and the multilevel space
be defined as $\bW_{n-1} := \mbox{\emph{span}}_{D_k^{n-1} \in
  \mcD^{n-1}}\{ D_{k}^{n-1}\}$ and $\bW_{n-1,k} :=
\mbox{\emph{span}}_{\bpsi^{n-1,k}_g \in D_k^{n-1}}$ $\{
\bpsi^{n-1,k}_{g} \}$ for any $B^{n-1}_k \in \mcB^{n-1}$.

For any two sibling cells denoted as $B^{n-1}_{\tt{left}}$ and
$B^{n-1}_{\tt{right}}$ and corresponding basis functions
$C^{n-1}_{\tt{left}}$ and $C^{n-1}_{\tt{right}}$ at level $n$, let
$\mcE^{n-2}_{k} := C^{n-1}_{\tt{left}} \cup C^{n-1}_{\tt{right}}$ for
some index $k$ and let $B^{n-2}_{k}$ be the corresponding cell at
level $\mcB^{n-2}$. It is clear that $\mcE^{n-2}_{k}$ is an
orthonormal set.  The orthogonalisation process is repeated for all
the levels of the tree until the level $0$ is reached.

In Algorithm \ref{multilevelbasis} the process is described.  It is
not hard to show that this process will terminate in at most
$\mcO(nN)$ steps. Thus we have proved
\begin{theorem}~
Decompose $\bV_{n+1}$ as
$\bV_{n+1}
\rightarrow 
  \bV_0 \oplus
  \bW_{0} \oplus \dots \bW_{n}$
  and
\begin{enumerate}[(i)]
\item The complexity cost of the multi-level basis is bounded by $\mcO(nN)$.

\item The multi-level basis vectors of $\bV_0 \oplus \bW_{0} \oplus
\dots \bW_{n}$ form an orthonormal set.

\end{enumerate}  
\end{theorem}
These multilevel basis functions can now be used to detect the anomaly
$\bw(\bx,\omega)$ at the various levels of resolution.
\begin{algorithm}[htbp]
\caption{\textsc{MultilevelBasis}($\bT, \{\varphi_i\}_{i=1}^M, \mcE)$}
\begin{algorithmic}[1]
\Require Tree $\bT$ with levels $\ell=0,\dots,n$, KL modes $\{\varphi_i\}_{i=1}^M$, $\mcE$
\Ensure Bases for multilevel spaces $\bW_0,\dots,\bW_n$ (equivalently sets $D^0,\dots,D^n$)
\State Determine leaf level $n$ from $\bT$ (termination depth of \textsc{MakeTree})
\ForAll{leaf cells $B_k^n$ at level $n$}
    \State $(D_k^n, C_k^n)\gets$ \Call{LeafBasis}{$B_k^n, \{\varphi_i\}_{i=1}^M, \mcE$}
\EndFor
\State $D^n \gets \Union_{B_k^n} D_k^n$;\quad $\bW_n \gets \spanof(D^n)$
\For{$\ell=n-1$ \textbf{down to} $0$}
    \ForAll{cells $B_k^\ell$ with non-empty children $B_{\mathrm{left}}^{\ell+1}, B_{\mathrm{right}}^{\ell+1}$} 
    %\Comment{$B_{\tt{left}}^{\ell+1}$ or $B_{\tt {right}}^{\ell+1}$ could be empty}
        \State $\mcE_k^\ell \gets C_{\tt{left}}^{\ell+1}\cup C_{\tt{right}}^{\ell+1}$ %\Comment{$E_k^\ell$ orthonormal}
        \State Assemble $\bM^{\ell,k}$ with entries $(\bM^{\ell,k})_{i,j}=\ip{\varphi_i}{\eta_j}$ for $\{\eta_j\}$ enumerating $\mcE_k^\ell$
        \State Compute SVD: $\bM^{\ell,k}=UDV^\top$ and set $a_{\ell,k}\gets \rank(\bM^{\ell,k})$
        \State Split $\mcE_k^\ell$ into:
        \State \hspace{0.75cm} $C_k^\ell$ (first $a_{\ell,k}$ combinations) and $D_k^\ell$ (remaining nullspace part)
    \EndFor
     \ForAll{cells $B_k^\ell$ without children (i.e. a leaf)}
     \State $(D_k^\ell, C_k^\ell)\gets$ \Call{LeafBasis}{$B_k^\ell, \{\varphi_i\}_{i=1}^M, \mcE^l_k$}
     \EndFor
    \State $D^\ell \gets \Union_{B_k^\ell} D_k^\ell$;\quad $\bW_\ell \gets \spanof(D^\ell)$ 
\EndFor
\State \Return $\{\bW_\ell\}_{\ell=0}^n$ (combinations$\{D^\ell\}_{\ell=0}^n$)
\end{algorithmic}
\label{multilevelbasis}
\end{algorithm}

%\textbf{else} 

\subsection{Multilevel Detection}

\begin{lemma}
Suppose that $\bv \in L^{2}_{\bbP}(\Omega;L^{2}(U;\R^d))$ with KL
expansion $\bv =
\sum_{i \in \bbN} \lambda^{\frac{1}{2}}_{i} \phi_i(\bx)$  $Y_{i}(\omega)$.
Then for all $l \in \bbN_0$, $B^l_{k} \in \mcB^l$ and for the associated
orthogonal projection coefficients $d^{l,k}_i(\omega)
= \int_U \bv \T \bpsi^{l,k}_i \,\mbox{\emph{d}} \bx$ we have that
$\eset{d^{l,k}_i} = 0$ and $\eset{(d^{l,k}_i)^2} \leq \sum_{j\geq
M+1} \lambda_{j}$.
\label{mls:lemma1}
\end{lemma}
\begin{proof}
The proof is a simple extension of the argument
given in \cite{Castrillon2022}.
\end{proof}

As $M$ increases, not only is the approximation error of the KL expansion
reduced and dominated by the sum of eigenvalues, but the variance
of the coefficients $d^{l,k}_d$ for the corresponding cell $B^{l}_{k}$
is also controlled by the same quantity.  We shall use this property
to construct a reliable hypothesis test for detection of anomalous
signals in any of the cells $B^{l}_{k} \in \mcB^l$ for $l = 0,\dots,n$.

\begin{theorem}[Detection: Hypothesis Test]
  Suppose that $\bu(\bx,\omega) = \bv(\bx,\omega) + \bw(\bx,\omega)$  and 
%\[
\begin{equation*}
\begin{split}
H_0:\bu(\bx,\omega) = \bv(\bx,\omega) \hspace{1cm} H_A:\bu(\bx,\omega) \neq \bv(\bx,\omega).
\end{split}
%\]
\end{equation*}
Let $1 \geq \alpha \geq 0 $ be the significance level. If the
null hypothesis $H_0$ is true:
it follows that (a)
\[
\bbP(|d^{l,k}_p(\omega)| \geq \alpha^{-\frac{1}{2}}\sum_{i \geq M
+ 1}\lambda_i ) \leq \alpha.
\]
and (b) for any cell $B^{l}_{k} \in \mcB$ we have
\[
\bbP\left(\sum_{\bpsi^{l,k}_{p} \in D^{l}_{k}}  
\left(d^{l,k}_p(\omega)\right)^{2} \geq \alpha^{-1}\sum_{i \geq M
+ 1}\lambda_i \sum_{\bpsi^{l,k}_{p} \in D^{l}_{k}}
\left(
b^{l,k}_{i,p}
\right)^2
\right)
\leq \alpha, 
\]
where 
$b^{l,k}_{i,p} := \int_U 
\bphi_{i}^{\top} \bpsi^{l,k}_p\,\emph{d}\bx$.
\label{mls:detection}
\end{theorem}

\begin{proof}
(a) The result follows from Lemma \ref{mls:lemma1} and the Chebyshev
inequality. (b) Recall that $\bW_{l,k} := \mbox{\emph{span}}_{\bpsi^{l,k}_g \in D_k^{n}}\{ \bpsi^{l,k}_g \}$ and suppose $P^{l,k}:L^{2}(U;\R^d) \rightarrow \bW_{n,l}$ is an orthogonal projection, then
\[
\bv^{l,k}(\bx,\omega) :=
P^{l,k} \bv(\bx,\omega) = 
\sum_{\bpsi^{l,k}_{p} \in D^{l}_{k}}  
d^{l,k}_{p} \bpsi^{l,k}_{p} 
\]
where $\{d^{l,k}_{1}, d^{l,k}_{2}, \dots, d^{l,k}_{p}, \dots \}$ 
are the orthogonal projection coefficients. Alternatively,
\[
\bv^{l,k}(\bx,\omega) = P^{l,k} \bv(\bx,\omega) 
= \sum_{i=M+1}^{\infty} \lambda_{i}^{\frac{1}{2}} P^{l,k}\bphi_{i}(\bx)Y_i(\omega)
= \sum_{i=M+1}^{\infty} \lambda_{i}^{\frac{1}{2}} Y_i(\omega)
\sum_{\bpsi^{l,k}_{p} \in D^{l}_{k}} 
b^{l,k}_{i,p}
\bpsi^{l,k}_{p}
\]
and therefore
\[
\sum_{\bpsi^{l,k}_{p} \in D^{l}_{k}}  d^{l,k}_{p} \bpsi^{l,k}_{p} 
= 
\sum_{i=M+1}^{\infty} \lambda_{i}^{\frac{1}{2}} Y_i(\omega) 
\sum_{\bpsi^{l,k}_{p} \in D^{l}_{k}} 
b^{l,k}_{i,p} \bpsi^{l,k}_{p}
\]
It is not hard to show that (e.g, Parseval's Theorem)
\[
\sum_{\bpsi^{l,k}_{p} \in D^{l}_{k}} 
\left( d^{l,k}_{p} \right)^{2}
= 
\int_U
\left( 
\bv^{l,k}(\bx,\omega) 
\right)^{\top}
\bv^{l,k}(\bx,\omega) 
\,\mbox{d}\bx
\] 
and therefore $\eset{
\sum_{\bpsi^{l,k}_{p} \in D^{l}_{k}} 
\left( d^{l,k}_{p} \right)^{2}
} = $
\[
\eset{
\int_U 
\left(
\sum_{i=M+1}^{\infty} \lambda_{i}^{\frac{1}{2}} Y_i(\omega) 
\sum_{\bpsi^{l,k}_{p} \in D^{l}_{k}} 
b^{l,k}_{i,p} \bpsi^{l,k}_{p}
\right)^{\top}
\left(
\sum_{j=M+1}^{\infty} \lambda_{j}^{\frac{1}{2}} Y_j(\omega) 
\sum_{\bpsi^{l,k}_{p} \in D^{l}_{k}} 
b^{l,k}_{j,p} \bpsi^{l,k}_{p}
\right) \,\mbox{d}\bx
}.
\]
Since $Y_1, \dots,Y_M,\dots$ are all uncorrelated and have unit variance,
\[
\eset{
\sum_{\bpsi^{l,k}_{p} \in D^{l}_{k}} 
\left( d^{l,k}_{p} \right)^{2}
}
=
\sum_{i=M+1}^{\infty} \lambda_{i} 
\int_U \left( \sum_{\bpsi^{l,k}_{p} \in D^{l}_{k}} 
b^{l,k}_{i,p} \bpsi^{l,k}_{p}\right)^{\top}
\sum_{\bpsi^{l,k}_{g} \in D^{l}_{k}} 
b^{l,k}_{i,g} \bpsi^{l,k}_{g}
\,\mbox{d}\bx
\]
Furthermore, since all the functions in $D^{l}_{k}$ are orthonormal, we conclude that
\[
\eset{
\sum_{\bpsi^{l,k}_{p} \in D^{l}_{k}} 
\left( d^{l,k}_{p} \right)^{2}
}
=
\sum_{i=M+1}^{\infty} \lambda_{i} 
\sum_{\bpsi^{l,k}_{p} \in D^{l}_{k}}
\left(
b^{l,k}_{i,p}
\right)^2
\]
From Markov's inequality the result follows
\end{proof}

\begin{remark}
(\textbf{Important})
With this hypothesis test the coefficients $d^{l,k}_p$ can be
  used as detectors of anomalous signals in the cell $B^l_k$.  Here
  are the key features of this detector:
\begin{itemize}[noitemsep,topsep=0pt]
%\setlength\itemsep{-0.5em}
%\item \textit{Interpretation:} The larger the coefficients $d^k_l$ the less
%  likely the null hypothesis will be true
\item Decay of the eigenvalues with $M$ controlling the sharpness of the bound.
\item For validity of the hypothesis test only a good estimate of the
  covariance function is needed.
  \item \textbf{No assumptions on independence nor underlying
    distribution (e.g. Normal, Poisson, etc.) of the data}.
\end{itemize}
In Algorithm \ref{algorithm5} we describe how to use multilevel anomaly detection method with distribution free hypothesis test to classify anomalous cells.
\end{remark}

\begin{algorithm}[ht]
\caption{Multilevel anomaly detection}
\label{algorithm5}
\begin{algorithmic}[1]
\Require Training data $\{v^{(m)}\}_{m=1}^{N_{\mathrm{train}}}$, test field $u$, truncation level $M$, significance level $\alpha$, $n_0$
\Ensure Cellwise decisions on anomaly detection
\State Estimate the mean $\mu$ and covariance operator from $\{v^{(m)}\}_{m=1}^{N_{\mathrm{train}}}$
\State Compute the leading KL eigenpairs $\{(\lambda_i,\bphi_i)\}_{i=1}^M$
\State Build the kd-tree $\bT \gets \Call{MakeTree}{S,n_0}$ and multilevel detail spaces $\{D_k^\ell\} \gets \Call{MultilevelBasis}{\bT,\{\bphi_i\}_{i=1}^M,\mcE}$
\State Center the test field: $\widetilde u \gets u-\mu$
\For{each level $\ell$ and each cell $B_k^\ell$}
    \State Compute $d_p^{\ell,k}=\int_U \widetilde u(x)^\top \psi_p^{\ell,k}(x)\,dx$ for all $\psi_p^{\ell,k}\in D_k^\ell$
    \State Form $T_k^\ell=\sum_{\psi_p^{\ell,k}\in D_k^\ell}(d_p^{\ell,k})^2$
    \State Form $\tau_k^\ell(\alpha)=\alpha^{-1}\!\left(\sum_{i\ge M+1}\lambda_i\right)\!\left(\sum_{\psi_p^{\ell,k}\in D_k^\ell}\sum_{i\ge M+1}(b_{i,p}^{\ell,k})^2\right)$
    \Statex \hspace{\algorithmicindent} where $b_{i,p}^{\ell,k}=\int_U \bphi_i(x)^T \psi_p^{\ell,k}(x)\,dx$
    \If{$T_k^\ell \ge \tau_k^\ell(\alpha)$}
        \State declare $B_k^\ell$ anomalous
    \Else
        \State declare $B_k^\ell$ nominal
    \EndIf
\EndFor
\end{algorithmic}
\end{algorithm}

An alternative approach for detecting signals is to measure the size
of the anomaly with respect to a suitable norm.

\begin{theorem}
Suppose that $\bu(\bx,\omega) = \bv_{M}(\bx,\omega) + \bw(\bx,\omega)$
for some $\bw(\bx,\omega) \in$ \\ $L^{2}_{\bbP}(\Omega;L^{2}(U;\R^d))$,
where $\bw(\bx,\cdot) \in \bV_{0}^{\perp} \cap \bV_{n+1}$ almost
surely.  Then
%\[
\begin{equation*}
\begin{split}
&\sum_{l=0}^n
\sum_{\bpsi^{l,k}_{p} \in D^{l}_{k}} \sum_{D^{l}_{k} \in \mcD^l}
(d^{l,k}_{p})^{2} = 
\| \bw(\bx,\omega)\|^{2}_{L^{2}(U;\R^d)}
\,\,\,\mbox{(a.s.) and} \\
&\sum_{l=0}^n  
\sum_{\bpsi^{l,k}_{p} \in D^{l}_{k}} \sum_{D^{l}_{k} \in \mcD^l}
\eset{(d^{l,k}_{p})^{2}} = 
\| \bw \|^{2}_{ L^{2}_{\bbP}(\Omega;L^{2}(U;\R^d)) }.
\end{split}
%\]
\end{equation*}
\end{theorem}

\begin{proof} The result follows from the orthogonality of the multilevel
basis of $\bW_0 \oplus \dots \bW_n$.
\end{proof} 

Thus under the model $\bu(\bx,\omega) = \bv_{M}(\bx,\omega)
+ \bw(\bx,\omega)$ the size of the anomaly for $\bw(\bx,\omega)$ can be
calculated from the projection coefficients of the basis functions
in $\bV_{0}^{\perp} \cap \bV_{n+1} = \bW_0 \oplus \dots \bW_n$.

In many cases the nominal behaviour of the signal cannot be captured
assuming a finite dimensional random field
$\bv_M(\bx,\omega)$. Suppose that $\bu(\bx,\omega) = \bv(\bx,\omega)
+ \bw(\bx,\omega)$, but $\bw(\bx,\cdot) \in
$$\bV_{0}^{\perp} \cap \bV_{n+1}$ almost surely. In this case the tail of
the KL expansion of $\bv(\bx,\omega)$ intersects with the
anomaly. However, the size of the anomaly can still be bounded.
\begin{theorem}
Let $t_{M} := \sum_{j\geq M +1} \lambda_j$, $s_{M} := \sum_{j\geq M +1} \sqrt{\lambda_j}$,
and suppose that
$\bu(\bx,\omega) = \bv(\bx,\omega) + \bw(\bx,\omega) $ for some
$\bw(\bx,\omega) \in L^{2}_{\bbP}(\Omega;L^{2}(U;\R^d))$, where
$\bw(\bx,\cdot) \in \bV_{0}^{\perp} \cap \bV_{n+1}$ almost surely.
Then
%\[
\begin{equation*}
\begin{split}
\| \bw \|^{2}_{ L^{2}_{\bbP}(\Omega;L^{2}(U;\R^d)) } ( 1 - 2
s_M
) 
+
t_M 
&\leq 
\sum_{l=0}^{n} 
\sum_{\bpsi^{l,k}_{p} \in D^{l}_{k}} 
\sum_{D^{l}_{k} \in \mcD^l}
\eset{(d^{l,k}_{p})^{2}} \\
&\leq 
\| \bw \|^{2}_{ L^{2}_{\bbP}(\Omega;L^{2}(U;\R^d)) }4
( 1 + 2
s_M
) 
+
t_M.
\end{split}
%\]
\end{equation*}
\label{mls:theo3}
\end{theorem}
\begin{proof}
The result is a simple extension of the proof of Theorem 3
in \cite{Castrillon2022}.
\end{proof}

\begin{remark} The implementation of the vector field anomaly detection, which includes the multilevel basis construction and hypothesis tests, can be downloaded from
\\ {\tt https://github.com/jcandas/Multimodal-Anomaly-Detection}.
\end{remark}

\section{Performance tests}
\label{section:performance}
In order to obtain a quantitative evaluation of our method, we design a test to mimic performance on noisy, 2-D imagery. The test consists of the attempted detection of a Gaussian anomaly in generated, synthetic images. We start by drawing samples from the stochastic process described in Example 1 of \cite{Castrillon2022}. To this end, we let
%\todo{test is this really working}
$$
v_m(x,\omega)=1+Y_1(\omega) \left(\frac{\sqrt{\pi}L}{2}\right)^{\frac{1}{2}} + \sum_{k=2}^m \lambda_k^{\frac{1}{2}} \phi_k(x)Y_k(\omega)
$$
be our stochastic process defined for $x\in[0,1]$ where
$$
\phi_k(x) := \begin{cases} 
      \sin{\frac{\lfloor \frac{k}{2}\pi x}{L_p}}, \text{ if $k$ is even} \\
      \cos{\frac{\lfloor \frac{k}{2}\pi x}{L_p}}, \text{ if $k$ is odd} \\
   \end{cases}\,\,\,\mbox{and}\,\,\, \sqrt{\lambda_k} := (\sqrt{\pi}L)^{\frac{1}{2}}\exp \left( -\frac{(\lfloor \frac{k}{2} \rfloor \pi L)^2}{8} \right).
$$
%and
%$$
%\sqrt{\lambda_k} := (\sqrt{\pi}L)^{\frac{1}{2}}\exp \left( -\frac{(\lfloor \frac{k}{2} \rfloor \pi L)^2}{8} \right).
%$$
For our test, we set $m=10$, $L=L_p=0.25$, and take $Y_1,...,Y_m$ to be i.i.d. $U(-\sqrt{3}, \sqrt{3})$ distributed random variables. To generate a realization of $v_m(x,\omega)$, we draw samples of $Y_1,...,Y_m$ according to the aforementioned distribution. This produces a function, $v_m(x,\omega)$, within the domain [0,1].

\begin{figure}[t]
\centering
\includegraphics[scale = 0.4]{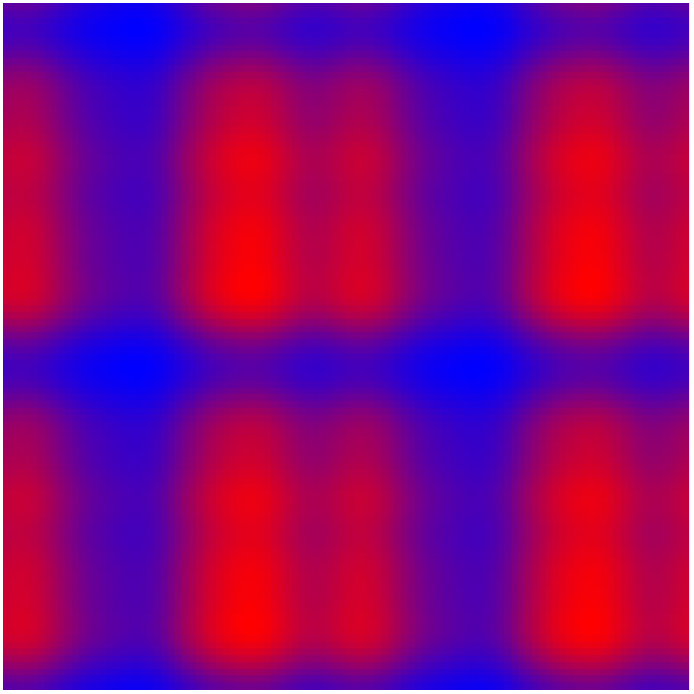}
\includegraphics[scale = 0.4]{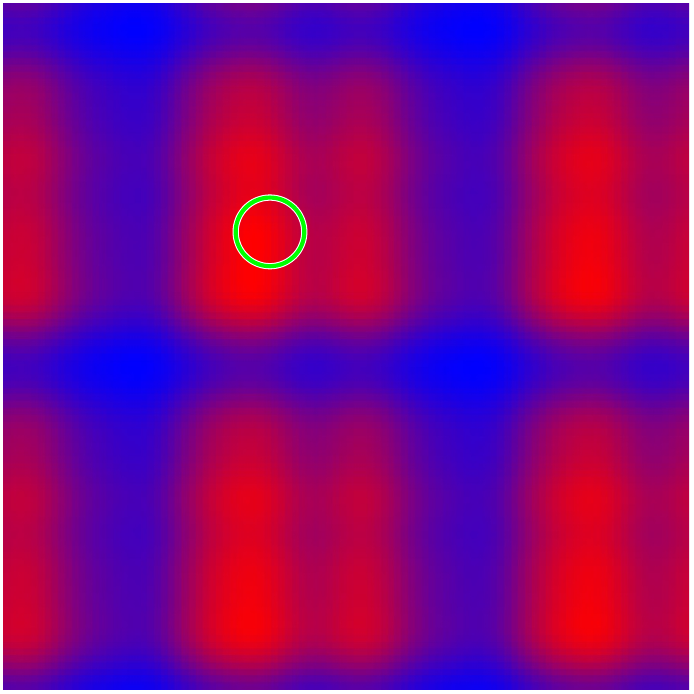}
\includegraphics[scale = 0.4, trim = 2.25cm 1.2cm 2.25cm 1cm, clip]{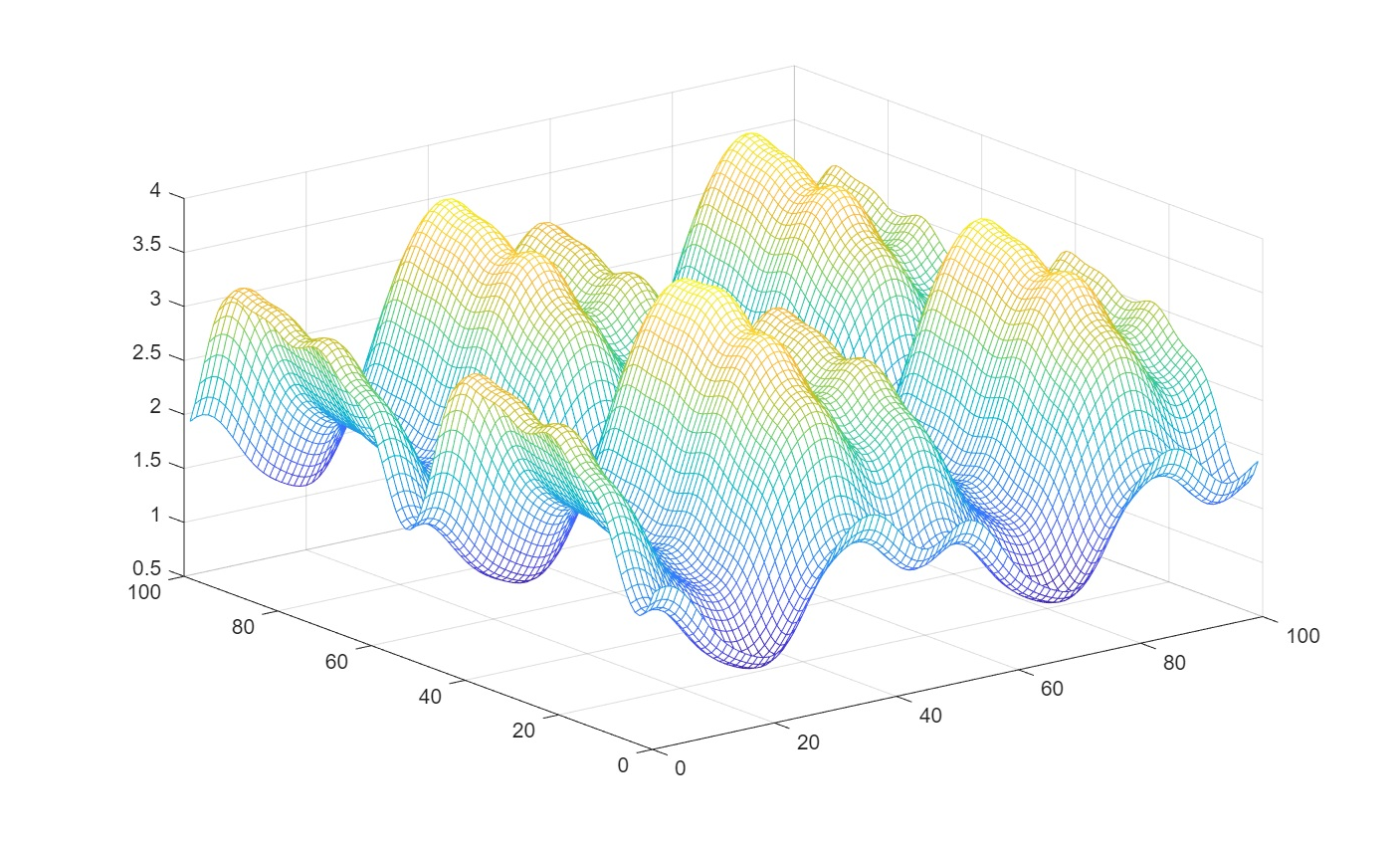}
\caption{Example synthetic image generated based on a stochastic process. Left color plot is a nominal image and right color plot includes an injected Gaussian anomaly, height 0.2 (circled in green). Plots are normalized as blue = minimum value; red = maximum value. Mesh grid below depicts same synthetic image with injected Gaussian.}
\label{performancetest:fig1}
\end{figure}

We then generate our synthetic images as the tensor product of two, independent realizations of this process. The output is a 2-dimensional oscillatory function in the domain $[0,1]\times[0,1]$, whose output ranges roughly from 0 to 2. An example is shown in Figure \ref{performancetest:fig1}.

For our performance test, we generate a training set consisting of 100 independent realizations of the 2-dimensional synthetic images. For a test set, we include both nominal and anomalous images. Nominal images are generated in the same manner as the training images. Anomalous examples are initially generated in the same manner as the training images, but then have a 2-dimensional Gaussian overlaid on top. These Gaussians have covariance matrix $\big[\begin{matrix}
0.05 & 0\\
0 & 0.05
\end{matrix}\big]$ and means that are drawn randomly from $U[0.1,0.9]\times[0.1,0.9]$ (we exclude points near the exterior of the domain in order to ensure the majority of the Gaussian is present within the image). We scale the Gaussians to have heights ranging from 0 to 0.2 and choose them specifically such that we have heights on the order of $10^{-4}$, $10^{-3}$, $10^{-2}$, and $10^{-1}$. An example synthetic image with an injected Gaussian of height 0.2 is shown in Figure \ref{performancetest:fig1}. Overall, the test set includes 200 nominal images, and 200 images at each of 28 different Gaussian heights, for a total of 5800 test images.

We evaluate the detection capabilities of our KL-based method, using the train and test sets described above. The first step is to use the training set to construct a covariance operator, on which we perform our KL expansion. The KL expansion produces a set of 100 eigenfunctions, which we arrange in descending order of associated eigenvalue. We choose $M=85$ as the truncation parameter, and let the last 15 eigenfunctions constitute the residual space. Then, we construct a multilevel basis according to the procedure outlined in Algorithm \ref{multilevelbasis}, creating our tree with levels 0 to 6 (level 0 contains the entire image domain and level 6 provides the finest granularity). We then classify each image as either nominal, or anomalous based on the process described in Algorithm \ref{algorithm5}, using a significance level of $\alpha=0.05$. A statistically significant result at any level of our tree triggers an anomalous classification for the image in question. For an image with an injected Gaussian anomaly, we consider our level of localization to be the highest (or finest) level at which we correctly detect an anomaly in the cell containing the center of the Gaussian.

For a comparison, we also employ a PCA-based residual space anomaly detection method, described in \cite{Lakhina2004}. The first step of the method involves running Principal Component Analysis (PCA) on our covariance operator to obtain the eigenvectors and associated eigenvalues. We then select a truncation parameter as the cutoff point between the principal components and the residual components. In order to maintain a consistent comparison, the truncation parameter is chosen to be $M=85$, the same value used for our KL method. We project the test image onto the principal components, calculating the projection error as the difference between the original image and the projection. One should note that this projection error is the same as the projection image onto the residual space of the last 15 components. We refer to the L2-norm squared of the projection error as the Squared Prediction Error (SPE), which we use as our test statistic. Our threshold for detection is the Q-statistic, discussed in \cite{Jackson1979}; we compute it using a significance level of $\alpha=0.05$, the same value used for our KL method.

\begin{figure}[t]
\centering
\includegraphics[scale = 0.5]{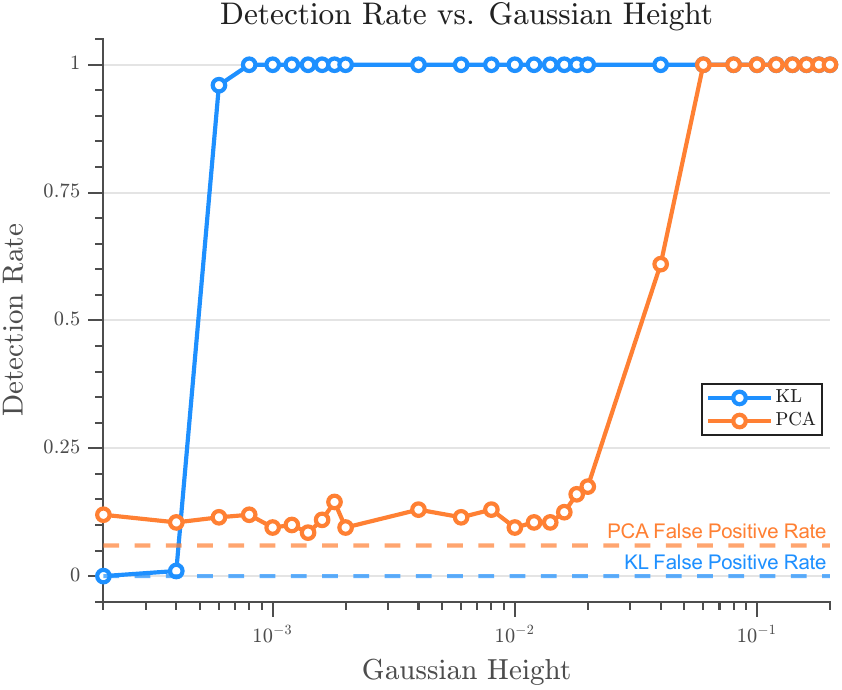}
\includegraphics[scale = 0.5]{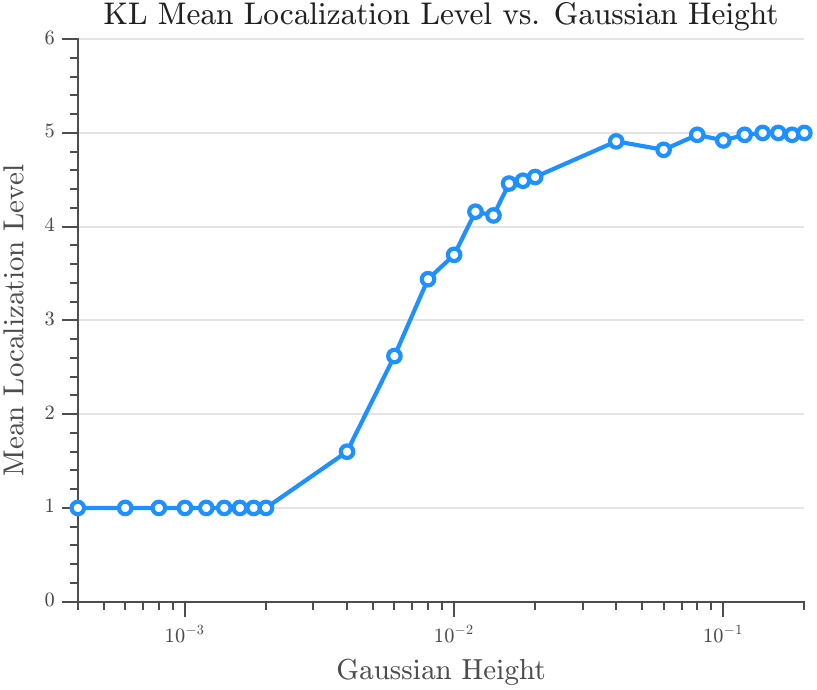}
\includegraphics[scale = 0.5]{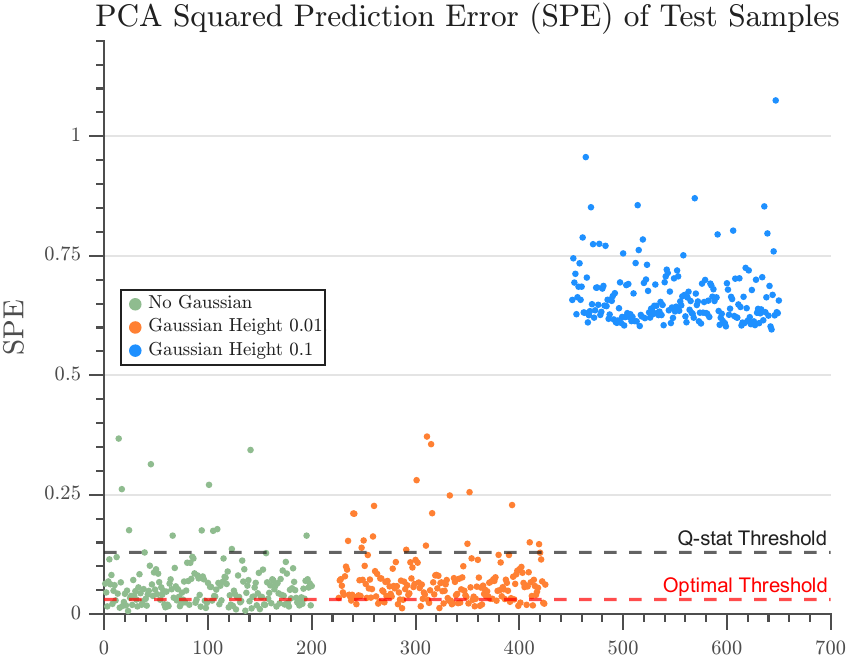}
\caption{Plots showing results from running our KL-based method and a PCA-based method on our synthetic test set. Top left plot shows detection statistics for both methods, plotted against the height of the anomalous Gaussian. False detection rates on the nominal test images are shown as dotted lines for both methods. Top right plot shows the KL method's mean level of localization as a function of Gaussian height. Bottom plot shows Squared Prediction Error (SPE) for each sample in the test set with no Gaussian (green), Gaussian height 0.01 (orange), and Gaussian height 0.1 (blue). Dotted black line shows the Q-statistic threshold at $\alpha=0.05$ confidence level; red dotted line shows the optimal separating threshold between the green and orange points.}
\label{performancetest:fig2}
\end{figure}

One should note that our KL-based detector is conceptually similar to the PCA-based detector; both involve projecting the test image onto the residual space of a centered SVD process, and classifying based on the size of the projection. There are two main differences: 1. the method by which we classify based on the residual projection and 2. the method by which our KL method determines localization through a multilevel basis.

We run both the KL and the PCA methods on our test set of 200 nominal images and 5600 anomalous images. The detection statistics, separated by Gaussian height, are shown in Figure \ref{performancetest:fig2}, top left. Our method maintains perfect detection among anomalous images until the Gaussian heights decrease down to the order of $10^{-4}$. This is in contrast with the PCA method's results, where detection starts to degrade once the Gaussian heights decrease down to the order of $10^{-2}$. Additionally, our method registers zero false positives among the 200 nominal test images (denoted by the dotted blue line), whereas the PCA-based method registers a $6\%$ false positive rate (denoted by the dotted orange line). Our improved performance can be attributed to the difference in the methods' procedures for determining classification from the residual projection.

Using our KL expansion, we leverage the probabilistic properties of the data and the global covariance structure in the test image to directly compute $p$-values connected to regions of the test image. Our direct computation of $p$-values for each separate test image provides an adaptable threshold that keeps our method sensitive to anomalies while maintaining robustness against false detections.

In contrast, the Q-statistic threshold used in the PCA-detector is static across the test set, and assumes the underlying data is normally distributed (recall, our method makes no such assumption). This threshold translates to decreased sensitivity as the anomaly height decreases. In Figure \ref{performancetest:fig2}, bottom, there is clear separation between the SPE's of the nominal test images and those with Gaussian height 0.1. However, when the Gaussian height decreases to 0.01, the clouds of SPE's between the nominal and anomalous images converge to similar levels. As such, even if we had prescient knowledge of the labels in the test set, and were to apply the optimally separating threshold (red dotted line) to the SPE's, we would still be unable to clearly separate the nominal and anomalous images.

Additionally, our KL method provides localization via the multilevel basis. The localization of the KL method is shown as a function of Gaussian height in Figure \ref{performancetest:fig2}, top right. Recall that each level of localization involves splitting the domain of the test image roughly in half, so level 0 includes the entire image, level 1 splits the image into $\frac{1}{2}$'s, level 2 into $\frac{1}{4}$'s... We are able to consistently localize the anomaly to levels 4-5 (within the correct $\frac{1}{16}$ to $\frac{1}{32}$ of the test image) while the Gaussian height is on the order of $10^{-2}$ or greater. Additionally, even when the Gaussian height is extremely small (order of $10^{-4}$ or $10^{-3}$) we are still able to correctly localize the anomaly at level 1, or to the correct half of the image.

Meanwhile, the PCA-based method provides classification, but not localization. There are methods of adapting PCA-based methods to provide localization, such as calculating the residual projection within a sliding window. While these methods provide localization for an anomaly, they typically focus solely on a subset of the image's domain, ignoring the global structure within the rest of the image. Our KL method provides localization within the multilevel basis while still considering the global structure of the test image.

\section{Application: Forest degradation}
\label{section:applications}

This mathematical framework is well suited for detecting changes in
terrestrial land surfaces based on both optical and radar data. Here, we apply it to data collected from the
Sentinel-2 satellite (see \cite{Drusch2012}), 
for the detection of forest
degradation in the Amazon. Detection here is a complex task, as can be
seen from Figure \ref{PR:Fig2}. In particular, detection of changes in
the state of a forest is significantly hindered by the presence of
misleading anomaly artifacts such as cloud cover.  We 
demonstrate the application of the multilevel anomaly filter to
Sentinel-2 satellite optical sensor data. The $p$-values are calculated from
Theorem \ref{mls:detection}. Note that these $p$-values do not require 
the distribution of the data, only the covariance structure, which is 
a significantly easier problem.

\begin{remark}
A cloud masking detection algorithm is applied from \cite{Zhu2012}. These
algorithms are not perfect, and in many instances clouds fail to be
detected or removed (see Figure \ref{PR:Fig2}). The
approach developed above provides a much more viable mechanism for automatically detecting and accounting for these cloud artifacts. The detection of clouds is explored in \cite{Castrillon2026deforest}.
\end{remark}

\begin{figure}[htb]
  \centering
  %\begin{tabular}{c c} 
  %  \includegraphics[scale = 0.4]{./figures/RGBImageFrame3334-crop.pdf} &
  %  \includegraphics[scale = 0.4]{./figures/RGBImageFrame3484-crop.pdf} \\
  %  \includegraphics[scale = 0.4]{./figures/RGBImageFrame3734-crop.pdf} &
  %  \includegraphics[scale = 0.4]{./figures/RGBImageFrame3924-crop.pdf} 
  %  \end{tabular}

    \begin{tabular}{c c}
    {\includegraphics[scale = 0.4, trim = 5cm 8cm 5cm 6.4cm, clip]{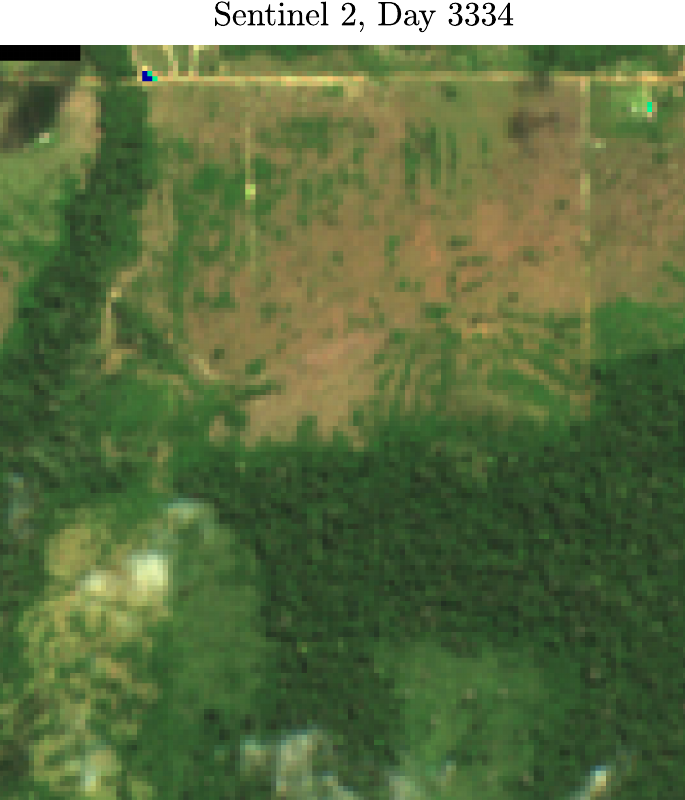}} &
    {\includegraphics[scale = 0.4, trim = 5cm 8cm 5cm 6.4cm, clip]{./figures/RGBImageFrame3484.pdf}} \\
    {\includegraphics[scale = 0.4, trim = 5cm 8cm 5cm 6.4cm, clip]{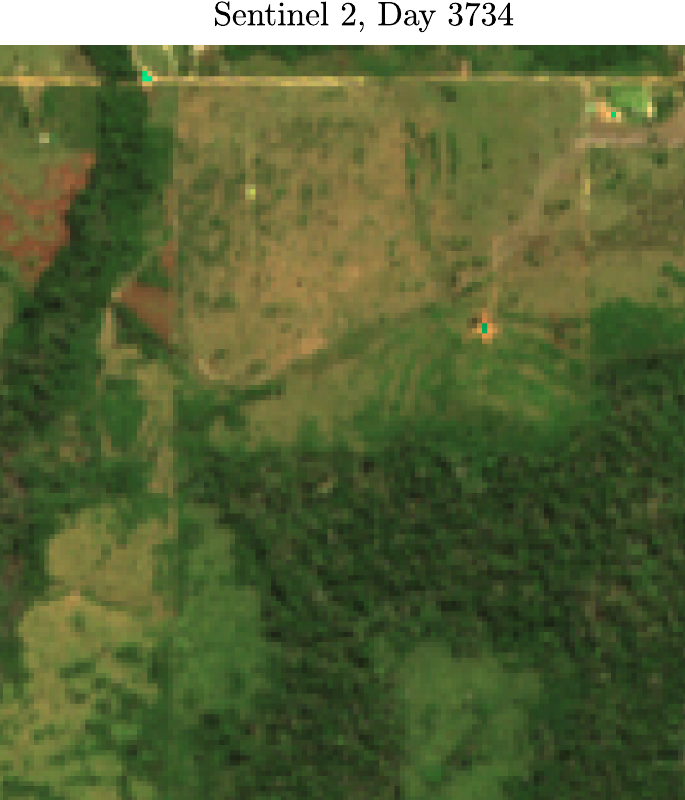}} &
    {\includegraphics[scale = 0.4, trim = 5cm 8cm 5cm 6.4cm, clip]{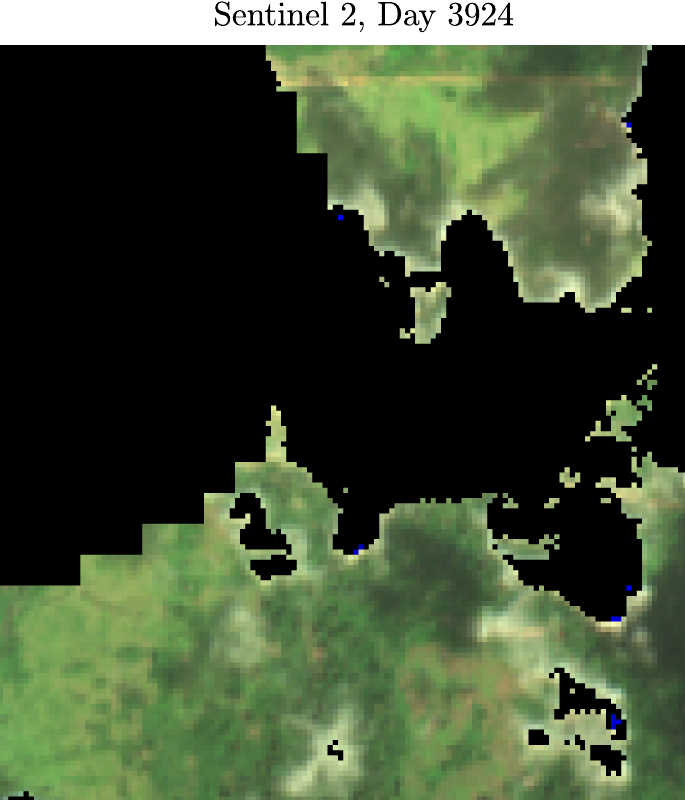}} 
    \end{tabular}
\caption{Deforestation sequence from Sentinel 2 satellite data.  Four
  frames of Amazon forest at \emph{days} 3334, 3484, 3774 and 3924
  showing the clearing and regrowth of forest.  Note that from the brown
  discoloration at day 3484 swaths of the trees are
  cleared. By day 3774 the forest vegetation grows back from 
  nearby trees. However, as we will see, it does not return to the
  earlier state, as this is new forest.  This will be clear when we
  apply the multilevel filter to the EVI data. However, it will be
  significantly more pronounced within the multispectral data.
  Day 3924 corresponds to a cloudy day. The black pixel corresponds to
  data removed by the standard cloud removal
  algorithms.} 
  %The multilevel filter provides a tool to
  %recognise these data as a false anomaly by analysis of the time
  %sequence.}
\label{PR:Fig2}
\end{figure}

\begin{figure}[htb]
  \centering
\begin{tikzpicture}
    \begin{scope}[scale = 1.35, every node/.style={scale=1.35}]
    \node at (0,6)
              {
                \includegraphics[scale = 0.25]{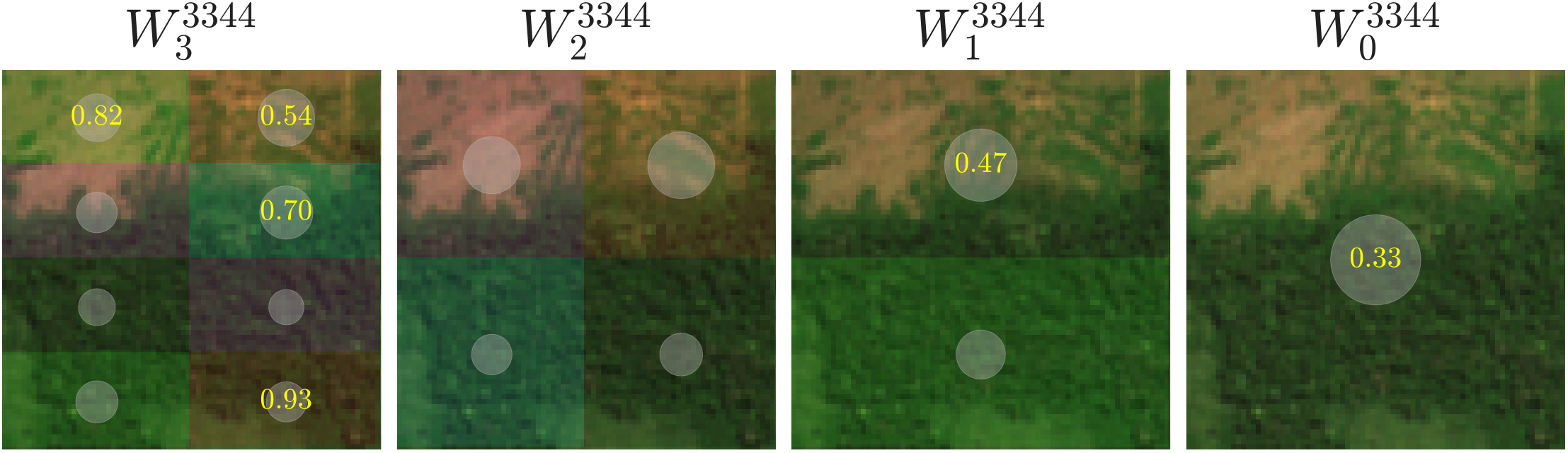}
              };
    \node at (0,3)
              {
                \includegraphics[scale = 0.25]{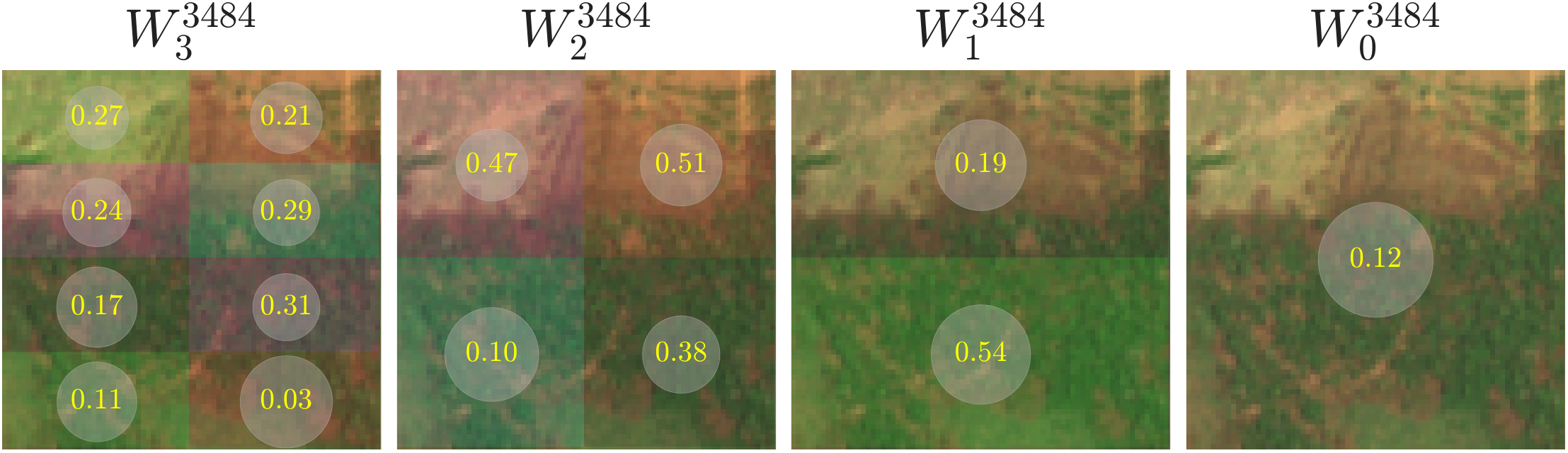}
               };
                  \node at (0,0)
              {
                \includegraphics[scale = 0.25]{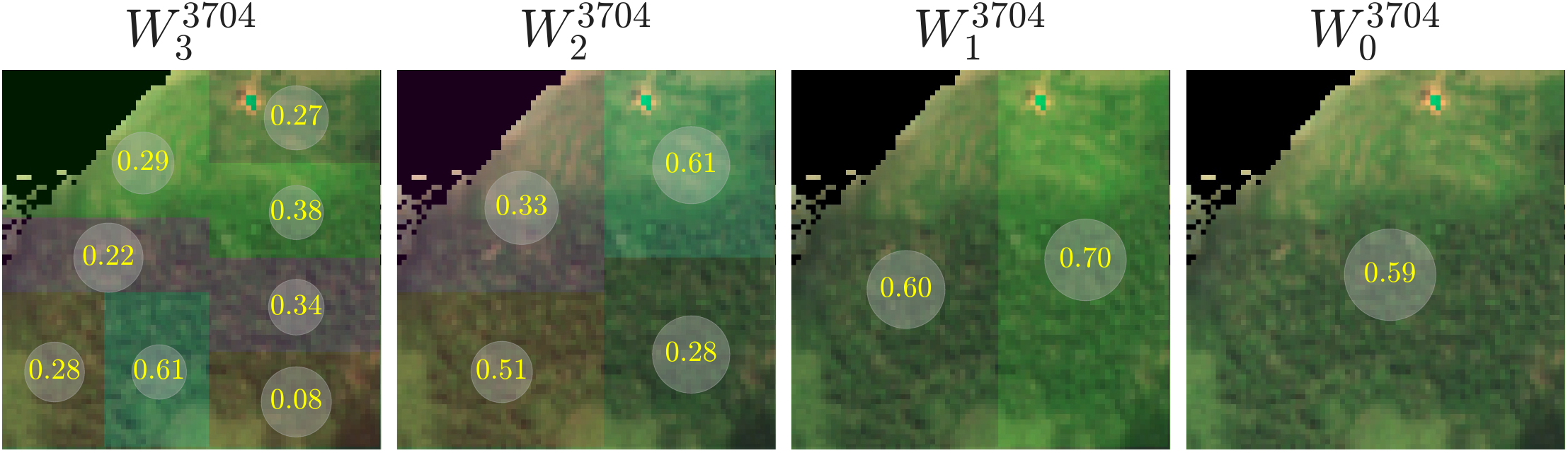};
              };
    \end{scope}  
  \end{tikzpicture}
 \caption{Multilevel anomaly map for days 3344, 3484 and 3704. The top image corresponds to the multilevel cells $B^l_k \in \mcB^l$ on the patch of terrain for each level $l=0,1,\dots,3$ overlayed on the RGB map. The cell supports for each level of resolution are traced by the points of different translucent colours. Notice that the cell size is smaller for higher levels and larger for lower levels. For each cell $B^l_k \in \mcB^l$ and corresponding basis functions $D^l_k \in \mcD^l$, the $p$-value for that cell (Theorem \ref{mls:detection}) is shown, i.e, the probability that the null hypothesis $H_0$ is correct. The $p$-value is printed in yellow when it is smaller than 1.  The grey circle corresponds to the size of $E^l_k$. For day 3344 the sizes of the anomalies $E^l_k$ are shown for each cell $B^l_k$. On day 3484 a part of the tree cover in the forest is removed, so that magnitudes of the anomalies increase. Furthermore the $p$-values decrease significantly for different cells indicating a change in the vegetation. By day 3704 the forest has largely recovered and most of the $p$-values increase, but sizes of the anomalies are still somewhat larger than for day 3344. This is due to the new vegetation. As will be observed, these changes will be significantly more pronounced for multilevel  anomaly detection on the multispectral Sentinel 2 data. Note that the shapes of the cells change due to missing data from the cloud cover.}
\label{PR:Fig4}
\end{figure}

\begin{figure}[ht]
\centering
\begin{tikzpicture}[>=latex',scale = 1, every node/.style={scale=1}]
         \node at (0,0) {\includegraphics[scale = 0.44, trim = 0cm 7.7cm 0cm 0cm,
             clip]{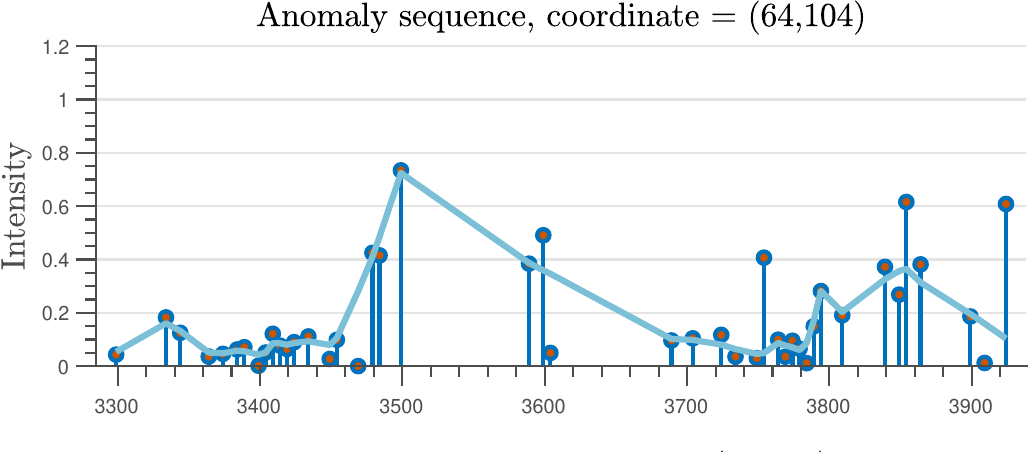}};

         \node at (7.5,0.03) {\includegraphics[scale = 0.44, trim = 0cm 0cm 0cm 7.5cm,
             clip]
             {./figures/AmazonForest/SequenceAnomalyFrame-64-104-crop.pdf}};
\end{tikzpicture}

\caption{Anomaly sequence in a pixel-level time series of satellite data mapping
environmental change illustrated in Figure \ref{PR:Fig2}. The
  multilevel anomaly filter is applied spatially to each frame, with
  the anomaly quantified and plotted against time. A robust LOESS
  is performed on the sequence (blue line).  Logging of the forest
  is detected on day 3484, where anomaly level increases. After
  logging, the forest is allowed to recover, with recovery determined
  on day 3774.}
\label{PR:Fig3}
\end{figure}

From the optical bands, the Enhanced Vegetation Index (EVI) (see \cite{Huete2002}), is created and used primarily to detect vegetation land cover. In Figure \ref{PR:Fig2}) the changes during the test period are shown. This example illustrates the complexity of degradation. On day 3334 the forest is in stable condition. On day 3484 logging of the forest occurs and we notice changes in land.  Furthermore, the forest does not remain in a degraded state and by day 3734  leaves from adjacent trees grow over the bare ground. Note that the anomaly introduced by imperfect cloud masking on day 3909 significantly complicates the task of detecting changes in the forest. 

\subsection{Anomaly detection on scalar data ($d = 1$)}

The multilevel detection approach is first tested on scalar Sentinel-2 data. From 71 Sentinel-2 data frames, starting from day 1 to day 3200, the covariance structure is computed and the multilevel anomaly filter is constructed. Each pixel of the $150 \times 150$ area corresponds to a land patch of $10m \times 10m$. However, the domain $U$ corresponds to the candidate test frame with the black pixels (missing data) removed. We further extract $75 \times 75$ land cover, construct the multilevel filter and apply the multilevel filter to each of the EVI test frames starting from day 3300. 

We choose the truncation parameter $M = 60$. This corresponds to about 4 orders of magnitude decay from the first eigenvalue (See
Figure \ref{truncation}). An alternative choice is to apply a  cross-validation method, but that would depend on the nature of the anomaly. For example, we could use a small area to calibrate the parameter $M$ and then apply detection in the full test area. We are developing alternate algorithms to obtain "good" truncation parameters. These will be described in more detail in a future publication.

Given the assumption that $\bw(\bx) \in \bV_{0}^{\perp} \cap \bV_{n+1}$ an anomaly map is reconstructed as $\bw(\bx) = \sum_{l=0}^n \sum_{\bpsi^{l,k}_{p} \in D^{l}_{k}} \sum_{D^{l}_{k} \in \mcD^l} d^{l,k}_{p} \bpsi^{l,k}_{p}(\bx)$ for each frame (day) of the test data. Note that since our data are finite dimensional, this is equivalent to projecting the test frames on the first $M$ eigenfunctions and subtracting it from the test frame data i.e. the residual component.

\begin{figure}[tb]
\centering
\includegraphics[scale = 0.5]{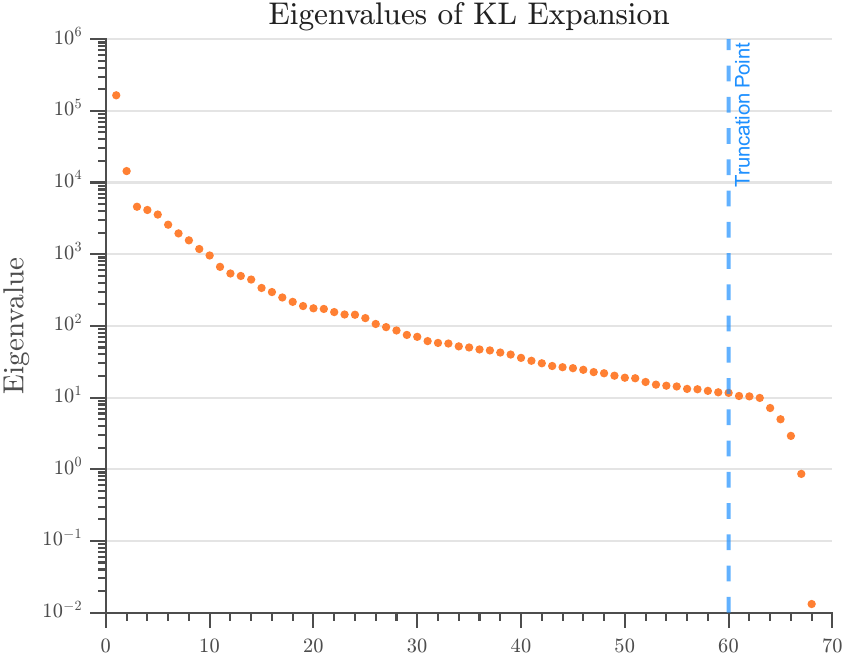}
\caption{Eigenvalue decay of nominal data and truncation parameter $M$}
\label{truncation}
\end{figure}

The projection coefficients $d^{l,k}_p$ are computed for levels $l = 0,\dots,3$ for each available test day e.g. $\bW^{\mbox{day}}_0,\bW^{\mbox{day}}_1,$ $\dots,\bW^{\mbox{day}}_3$. The contribution to the total size of the anomaly $\|\bw(\bx)\|_{L^{2}(U;\R^d)}$ for any fixed realisation on day, within each cell $B^l_k$, is computed as $E^l_k := \sqrt{\sum_{\bpsi^{l,k}_{p} \in D^{l}_{k}} (d^{l,k}_{p})^{2}}$. Note that for any fixed day $\|\bw(\bx)\|^2_{L^{2}(U;\R^d)} := \sum_{l = 0}^{n}$ $ \sum_{\bpsi^{l,k}_{p} \in D^{l}_{k}} \sum_{D^{l}_{k} \in \mcD^l} (d^{l,k}_{p})^{2}$.

In Figure \ref{PR:Fig4} the multilevel anomaly map for days 3344, 3484 and 3704 are shown for levels of resolution $l = 0,\dots,3$. For each cell $B^l_k$ the projection coefficients $d^{l,k}_{p}$ are computed along with the corresponding $p$-values of the cell i.e. the probability the null hypothesis $H_0$ is correct. The translucent rectangle corresponds to the compact support of the functions in $B^l_k$. The grey circle in the middle of the cell $B^l_k$ reflects the size of anomaly given by $E^l_k:=\left(\sum_{\bpsi^{l,k}_{p} \in D^{l}_{k}} (d^{l,k}_{p})^{2}\right)^{\frac{1}{2}}$ for each cell $B^l_k \in \mcB^{l}$. From days 3344, 3484 and 3704 we can see the changes in the forest as indicated by the decrease and then increase of the $p$-values corresponding to the cells $B^l_k \in \mcB^l$.

In Figure \ref{PR:Fig3} (a) the anomaly for pixel coordinate (64,104) is tracked in time. A robust LOESS (\cite{Cleveland1979,Matlab2025b}) is applied to the anomaly sequence. On day 3499 we note the large jump in the anomaly for that pixel coordinate, due to the removal of trees. As time evolves, the anomaly metric grows, reflecting the clearing of trees. However, the vegetation from nearby trees grows into the area (day 3774), and eventually the anomaly level returns to zero. This is expected as leaves in the adjacent trees grow to cover the area.

\subsection{Anomaly detection on scalar data ($d = 6$)}

Results are shown below for application of the multilevel filter to
vector field data.  From the Sentinel 2 dataset used in the previous
experiment, the same land cover of $75 \times 75$ pixels is
extracted. However, here we do not compute the EVI index from
individual optical bands but instead use all of the optical bands 1 -
6 (blue, green, red, near-infrared, shortwave infrared 1 and shortwave
infrared 2) as a vector field defined on the land cover.  The
covariance functional matrix across all bands is estimated and the
eigenspace is constructed for $M = 60$ eigenfunctions.  The anomaly
sequence is computed from the vectorial field multilevel filter
$\bw(\bx) =
\sum_{l=0}^n
\sum_{\bpsi^{l,k}_{p} \in D^{l}_{k}} \sum_{D^{l}_{k} \in \mcD^l}
d^{l,k}_{p} \bpsi^{l,k}_{p}(\bx)$
for each available day.

\begin{figure}[htb]
  \centering
  
  \begin{tikzpicture}
    \begin{scope}[scale = 1.35, every node/.style={scale=1.35}]
    \node at (0,6)
              {
                \includegraphics[scale = 0.25]{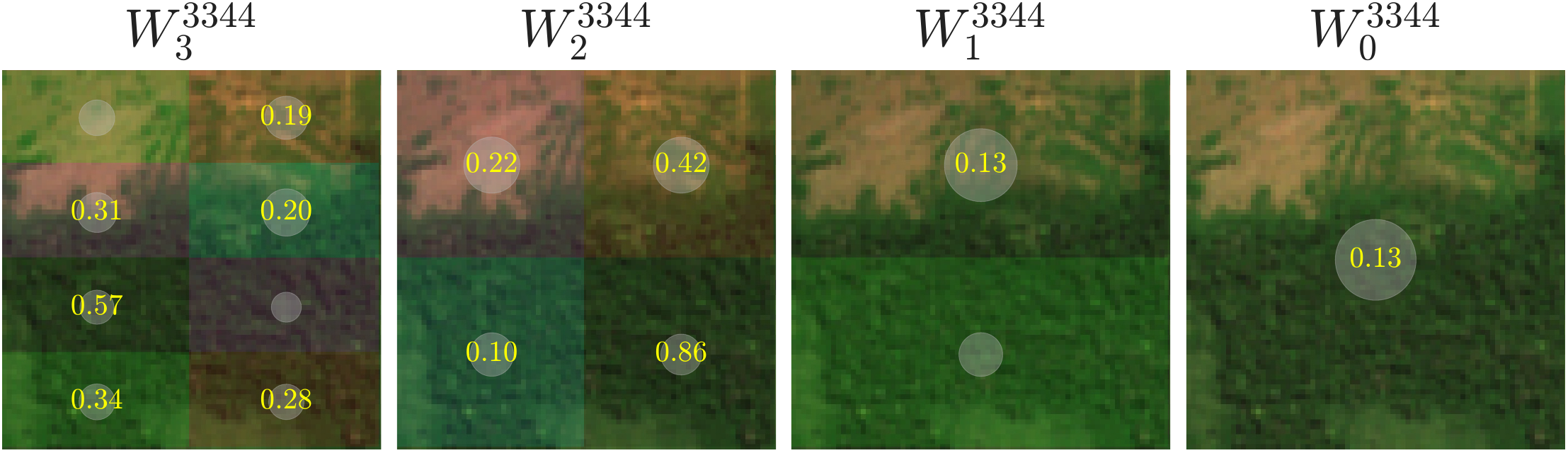}
              };
    \node at (0,3)
              {
                \includegraphics[scale = 0.25]{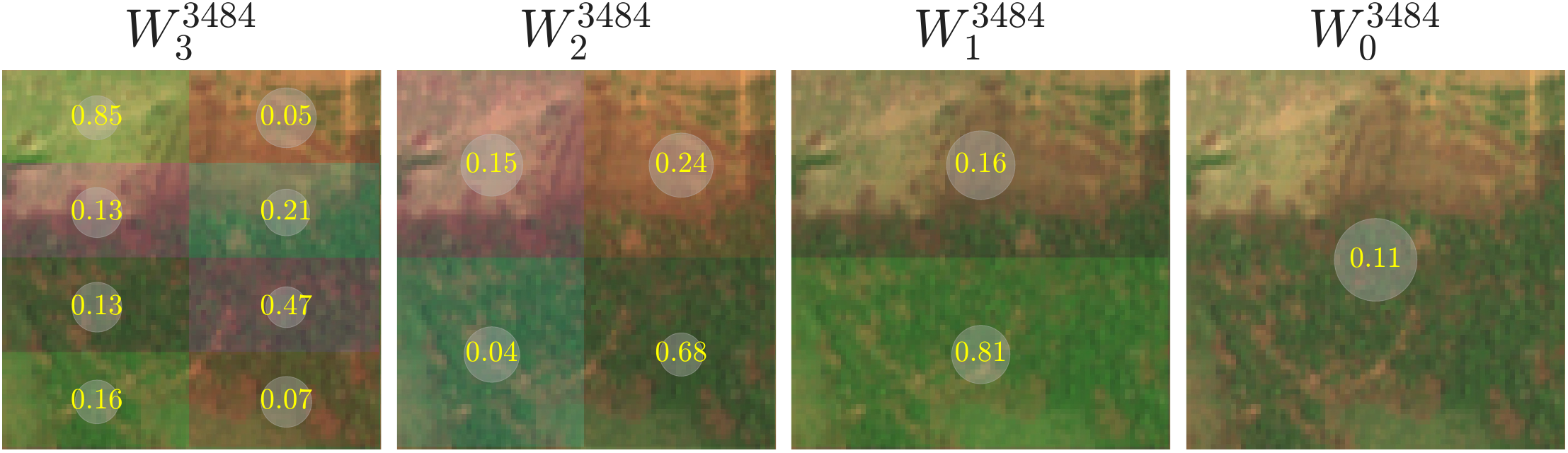}
               };
                  \node at (0,0)
              {
                \includegraphics[scale = 0.25]{./figures/AmazonForest/sentinel_vec_day3704_eig60.pdf}
              };
    \end{scope}  
  \end{tikzpicture}
\caption{Multilevel projection coefficient anomaly
  map for vector field Sentinel-2 data on a $75 \times 75$ land cover
  patch for days 3344, 3484 and 3704. Six spectral bands are used that
  include blue, green, red, near-infrared, shortwave infrared 1 and
  shortwave infrared 2. For visualization purposes the anomaly sizes $E^l_k$
  are divided by 1000.  For day 3344 the sizes of the anomalies $E^l_k$
  are shown for each cell $B^l_k$. On
day 3484 part of the trees in the forest are removed and thus
the sizes of the anomalies increase significantly. On day 3704
the forest has mostly recovered, but in contrast to the scalar EVI
data from Figure \ref{PR:Fig4} the sizes of the anomalies further
increase. This is due to the near-infrared (band 4) spectral data
which detects new chlorophyll. This is consistent with
Figure \ref{PR:Fig7} where the anomaly of band 4 increases
significantly at day 3704.}
\label{PR:Fig6}
\end{figure}

\begin{figure}[htb]
\centering
\begin{tikzpicture}[>=latex',scale = 0.98, every node/.style={scale=0.98}]
         \node at (0,0) {\includegraphics[scale = 0.49, trim = 0cm 11.4cm 0cm 0cm,
             clip]{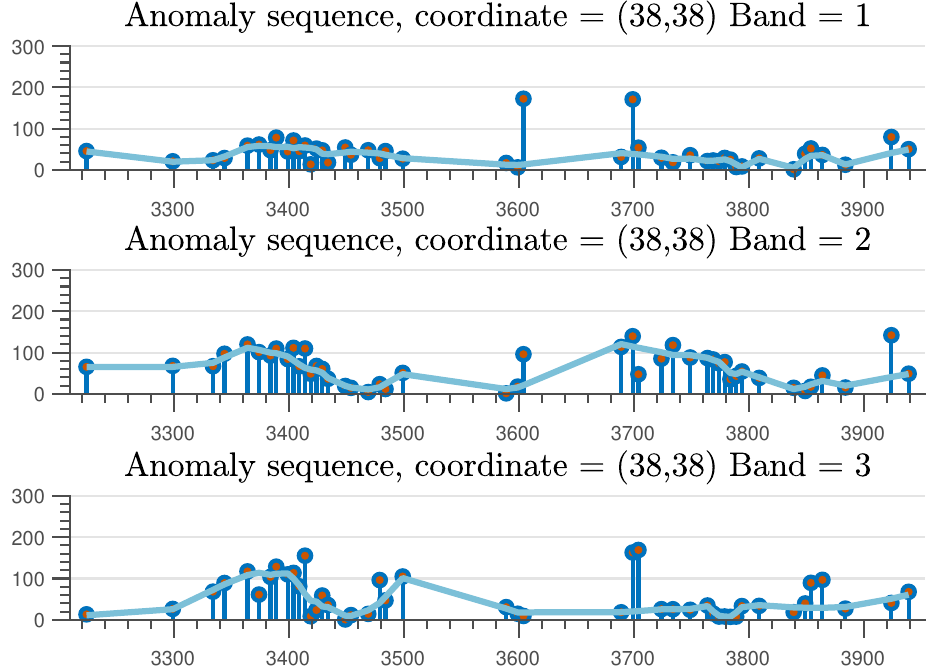}};
         \node at (7.6,0.01) {\includegraphics[scale = 0.49, trim = 0cm 0cm 0cm 11.35cm,
             clip]{./figures/AmazonForest/SequenceAnomalyFrameVec-38-38-crop.pdf}};
        %\begin{scope}[yshift = 0.84cm]
              %\fill [white] (1,1.35) rectangle (1.75,1.7);
              %\node at (1.37,1.52) {\tiny (96,70)};
        %\end{scope}
        %\begin{scope}[yshift = -0.9cm]
        %     \fill [white] (1,1.35) rectangle (1.75,1.7);
        %      \node at (1.37,1.52) {\tiny (96,70)};
        %\end{scope}
        %\begin{scope}[yshift = -2.6cm]
        %      \fill [white] (1,1.35) rectangle (1.75,1.7);
        %      \node at (1.37,1.52) {\tiny (96,70)};
        %\end{scope}

        %\begin{scope}[xshift = 8cm, yshift = 0.84cm]
        %      \fill [white] (1,1.35) rectangle (1.75,1.7);
        %      \node at (1.37,1.52) {\tiny (96,70)};
        %\end{scope}
        %\begin{scope}[xshift = 8cm,yshift = -0.9cm]
        %     \fill [white] (1,1.35) rectangle (1.75,1.7);
        %      \node at (1.37,1.52) {\tiny (96,70)};
        %\end{scope}
        %\begin{scope}[xshift = 8cm,yshift = -2.6cm]
        %      \fill [white] (1,1.35) rectangle (1.75,1.7);
        %      \node at (1.37,1.52) {\tiny (96,70)};
        %\end{scope}
        %\draw[color = blue, line width = 1pt] (8.94,2.15) circle (0.10cm);
\end{tikzpicture}
\caption{Deforestation anomaly sequence for multi modality Sentinel 2 data. Anomalies
are detected with respect to all bands. By day 3484 the
clearing of the trees is captured by the anomaly increase of bands 1
- 5. As the trees grow back the anomalies are slowly
reduced. The
clearing of the trees in parts of the forest eventually leads to
regrowth of the forest with new leaves. The near-infrared band is used
to detect the presence of new chlorophyll in the leaves. Notice that
the anomaly for this band peaks at day 3740.  This is consistent with regrowth.}
\label{PR:Fig7}
\end{figure}

It is advantageous to study forest degradation by incorporating all
correlations of the anomalies across spatial and band data.  In
Figure \ref{PR:Fig6} sizes of anomalies for days 3344, 3484 and 3704
are shown. For each cell $B^l_k \in \mcB^l$, $l = 0,1,2,3$, the size of the anomaly $E^l_k$ is calculated and the corresponding $p$-value. For day 3344
sizes of anomalies $E^l_k$ are shown for each cell. 
Anomalies are relatively small. However, for the cell $B^0_0$ we observe a small $p$-value, which indicates that changes have already occurred.

On day 3484 part of the tree cover is removed, and 
sizes of the anomalies increase significantly. This is consistent
with the scalar EVI anomaly map shown in Figure \ref{PR:Fig4}.
On day 3704 the forest has essentially recovered, but in contrast to
the scalar EVI, the anomalies increase further. This
is due to the near-infrared (band 4) spectral data that reflect the new
chlorophyll from forest regrowth. This is consistent with
Figure \ref{PR:Fig7} where the anomaly signal led by the near-infrared band
increases significantly by day 3740.

In Figure \ref{PR:Fig7} the anomaly time sequences for all 6 bands are
shown for pixel (38,38).  From the bands (1-5) observe the
relation of the clearing of trees by day 3484 to increases in the
anomaly levels for all 5 bands. In addition, particularly in band 4
(near infrared), by day 3740 the anomaly becomes more
notable. Near infrared as a modality detects new leaves
(chlorophyll). This is observable with the growth of vegetation after
the clearing of some of the trees around day 3484. By the last day the
anomaly level in near infrared slowly decreases, consistent with aging
of the leaves.

\section{Conclusions}

In this paper we have developed a multilevel basis that is adapted to
the truncated Karhunen-Lo\`{e}ve expansion of a nominal random vector
field.  We show a full rigorous proof of the existence and optimality
of the KL expansion for random vector fields. Numerical construction
of the multilevel basis for the vector field is achieved with an
efficient algorithm originating in the numerical solution of
PDEs. Detection is achieved by examining relative contributions to the
multilevel spaces. Anomalies are quantified using a suitable norm and
reliable hypothesis tests with weak assumptions. We apply our approach
to the problem of deforestation and degradation in the Amazon forest.
This approach does not replace, but rather augments current
statistical approaches.  For example, the multilevel filter amplitudes
can signal anomalous behaviour that can be tracked through identifiable
transitions using change point detection methods. Furthermore, these
anomalous behaviours can be identified at multiple change points,
further elucidating temporal behaviours. We are currently further
investigating these latter aspects.

Another problem of very general interest in this area is the reliable
identification of cloud cover. We have demonstrated a more effective
approach to this problem by examining variable temporal behaviour
anomalies. In \cite{Castrillon2026deforest} the authors explore 
detection of cloud cover.

%This has been shown in our results for machine
%learning \cite{Castrillon2025}.

\section*{Acknowledgments and Disclosure of Funding} 
We appreciate the help from Pontus Olofsson, Hanfeng Gu and Xiaojing
Tang in providing the Sentinel-2 dataset. Furthermore,
their expertise in remote sensing was useful in putting context to
this work and in interpreting the results.  We also appreciate the
feedback from Sucharita Gopal. Finally, we appreciate the insightful 
feedback from the JMLR reviewers. This material is based on work
supported by the National Science Foundation under Grants Nos. 1736392
and 2319011.

%\newpage

\appendix
\section{}
\label{appn}

% SIAM Supplemental File Template
%\documentclass[review,supplement,onefignum,onetabnum]{siamonline220329}

%\input{VFshared}

%\externaldocument[][nocite]{VFarticle}

% Optional PDF information
%\ifpdf
%\hypersetup{
%  pdftitle={Supplementary Materials},
%  pdfauthor={J. E. Castrillon-Candas, Mark Kon}
%}
%\fi

%\begin{document}

%\maketitle

% \section{Proofs}
In this appendix, we show the proof for the existence and optimality
of the vector field KL expansion by using tensor theory. However, we
first need to show what is meant by the tensor product spaces
$L^{2}(U;\R^d) \otimes L^{2}_{\bbP}(\Omega)$ and
$L^{2}(U;\R^d) \otimes L^{2}(U;\R^d)$.  We shall now prove a series of
useful results that are used to show existence and optimality of the
vector field KL expansion. The proofs for
Propositions \ref{intro:prop1}, \ref{intro:prop2}
and \ref{intro:prop3} are based on the arguments given in Chapter II.4
from \cite{Reed1972} and extended to the vector field case.

\begin{prop}
Let $\{\bv_k\}_{k \in \bbN}$ be a complete orthonormal basis of
$L^{2}(U;\R^{d})$ then $\{\bv_k(\bx)$ $\bv_l(\by) \T$
$\}_{k,l \in \bbN}$ is a complete orthonormal basis of $L^{2}(U \times
U;\R^{d \times d})$.
\label{intro:prop1}
\end{prop}
\begin{proof}
First we show that $\{\bv_k(\bx) \bv_l(\by) \T\}_{k,l \in \bbN}$ are
orthonormal.  Note
\begin{equation}
\begin{split}
&(\bv_i(\bx) \bv_j(\by) \T, \bv_k(\bx) \bv_l(\by) \T)_{L^2(U;\R^{d \times d})} =
\int_U \int_U
\oone\T (\bv_i(\bx) \bv_j(\by) \T
\bullet
\bv_k(\bx) \bv_l(\by) \T)
\oone\,\,\mbox{d}\bx \mbox{d}\by
\\
&=
\int_U \int_U
\oone \T
\begin{bmatrix}
v^i_1(\bx) v^k_1(\bx) \\
v^i_2(\bx) v^k_2(\bx) \\
\vdots \\
v^i_d(\bx) v^k_d(\bx)
\end{bmatrix}
\begin{bmatrix}
v^j_1(\by) v^l_1(\by) & \dots & v^j_d(\by) v^l_d(\by) 
\end{bmatrix}
\oone \,\,\mbox{d}\bx \mbox{d}\by
\\
&=
\int_U 
\bv_i(\bx) \T \bv_k(\bx) \,\,\mbox{d}\bx 
\int_U \bv_j(\by) \T \bv_l(\by) \,\, \mbox{d}\by 
=
\delta[i-k] \delta[j-l].
\end{split}
\end{equation}
Let $\bF \in L^{2}(U \times U;\R^{d \times d})$ and suppose that
\[
\int_U \int_U  \oone\T (\bF(\bx,\by) \bullet \bv_k(\bx) \bv_l(\by)) \oone \,\,\mbox{d}\bx \mbox{d}\by = 0
\]
for all $k,l \in \bbN$. Then
\begin{equation}
\begin{split}
&\int_U \int_U  \oone\T (\bF(\bx,\by) \bullet \bv_k(\bx) \bv_l(\by)) \oone \,\,\mbox{d}\bx \mbox{d}\by
=
\int_U \int_U \sum_{i,j} F_{i,j}(\bx,\by) v_i^k(\bx)  v_j^l(\by) \,\,\mbox{d}\bx \mbox{d}\by \\
&=
\int_U \sum_{j=1}^{d} 
v_j^l(\by) \left(
\int_U \sum_{i=1}^{d}
F_{i,j}(\bx,\by)
v_i^k(\bx) \,\,\mbox{d}\bx
\right) 
\mbox{d}\by = 0
\end{split}
\end{equation}
Since $\{\bv_l\}_{l \in \bbN}$ is a complete orthonormal basis of
$L^{2}(U;\R^{d})$, it follows
\begin{equation}
\int_U \sum_{i=1}^{d}
F_{i,j}(\bx,\by)
v_i^k(\bx) \,\,\mbox{d}\bx = 0
\label{back:eqn1}
\end{equation}
almost everywhere for all $k \in \bbN$. Now, for each $k \in \bbN$ let
$D_k$ be the set of measure zero on which
\eqref{back:eqn1} does not hold. Letting $D:= \cup_{k \in \bbN} D_k$,
if $\by \notin D$ then
\[
\int_U \sum_{i=1}^{d}
F_{i,j}(\bx,\by)
v_i^k(\bx) \,\,\mbox{d}\bx = 0
\]
for all $k \in \bbN$. Since $\{\bv_k\}_{k \in \bbN}$ is a complete
orthonormal basis for $L^{2}(U;\R^{d})$ it follows $\bF(\bx,\by) = \0$
almost everywhere with respect to the measure d$\bx$d$\by$.
\end{proof}

\begin{prop}
$L^{2}(U;\R^d) \otimes L^{2}(U;\R^d)$
is isomorphic to
$L^{2}(U \times U;\R^{d \times d})$.
\label{intro:prop2}
\end{prop}

\begin{proof}
From Proposition 1 we have that $\{\bv_k \bv_l \T
\}_{k,l \in \bbN}$
is a complete orthonormal basis for $L^{2}(U \times U;\R^{d \times d})$. Conversely, from Proposition 2 in \cite{Reed1972} we have that
$\{\bv_k \otimes \bv_l \}_{k,l \in \bbN}$ is a complete orthonormal
basis for $L^{2}(U;\R^d) \otimes L^{2}(U;\R^d)$. 
Now, define the map
$\mcM: L^{2}(U;\R^d) \otimes L^{2}(U;\R^d)
\rightarrow L^{2}(U \times U;\R^{d \times d})
$ such that 
\[
\mcM(\bv_k \otimes \bv_l ) = \bv_k \bv_l \T 
\]
for all $k,l \in \bbN$. Let $f \in L^{2}(U;\R^d) \otimes
L^{2}(U;\R^d)$, so that we can write $f = \sum_{k,l} \alpha_{k,l} \bv_k \otimes \bv_l$.
Define the map $\mcM': L^{2}(U;\R^d) \otimes L^{2}(U;\R^d)
\rightarrow L^{2}(U \times U;\R^{d \times d})$
as
\[
\mcM' f = \sum_{k,l} \alpha_{k,l} \mcM(\bv_k \otimes \bv_l)
= \sum_{k,l} \alpha_{k,l} \bv_k \bv_l \T.
\]
We have that
\[
\begin{split}
(\mcM'f, \mcM'f)_{L^{2}(U \times U;\R^{d \times d})}
%&=
%\sum_{k,l} \sum_{k,l}
%\alpha_{k,l} \alpha_{k',l'} (\mcM(\bv_k \otimes \bv_l),
%\mcM(\bv_{k'} \otimes \bv_{l'}))
%_{L^{2}(U;\R^d) \otimes L^{2}(U;\R^d)} \\
&=
\sum_{k,l} \sum_{k',l'}
\alpha_{k,l} \alpha_{k',l'} (\bv_k \bv_l \T,
\bv_{k'} \bv_{l'} \T)
_{L^{2}(U \times U;\R^{d \times d})}
= \sum_{k,l} \alpha^2_{k,l}.
\end{split}
\]
Conversely, we have 
\[
(f,f)_{L^{2}(U;\R^d) \otimes L^{2}(U;\R^d)} =
\sum_{k,l} \sum_{k',l'}
\alpha_{k,l} \alpha_{k',l'} (\bv_k \otimes \bv_l, \bv_{k'} \otimes \bv_{l'})_{L^{2}(U;\R^d) \otimes L^{2}(U;\R^d)}
= \sum_{k,l} \alpha^2_{k,l}.
\]
\end{proof}

\begin{prop}
Let $\{\bv_k\}_{k \in \bbN}$ is a complete orthonormal basis (cob) of
$L^{2}(U;\R^{d})$  and $\{s_k\}_{k \in \bbN}$ is a cob in $L^{2}_{\bbP}(\Omega)$. Then $L^{2}(U;\R^d) \otimes L^{2}_{\bbP}(\Omega)$
is isomorphic to $L^{2}_{\bbP}(\Omega;L^{2}(U;\R^d))$ and 
$\{\bv_k(\bx)$ $s_l(\omega)\}_{k,l \in \bbN}$ is a cob of 
$L^{2}(U;\R^d) \otimes L^{2}_{\bbP}(\Omega)$.
\label{intro:prop3}
\end{prop}

\begin{proof}
%However, for this case the tensor map $\mcM$ has to be constructed.
Let $\{\bv_k\}_{k \in \bbN}$ be a complete orthonormal basis of
$L^{2}(U;\R^{d})$ and $\{s_l\}_{l \in \bbN}$ be a complete orthonormal basis of
$L^{2}_{\bbP}(\Omega)$. Now, define  the map
$\mcM: L^{2}(U;\R^d) \otimes L^{2}_{\bbP}(\Omega)
\rightarrow L^{2}_{\bbP}(\Omega;L^{2}(U;\R^d))$ as
\[
\mcM(\bv_k \otimes s_l(\omega) ) = \bv_k s_l(\omega).
\]
The result follows from an argument  similar to that of the proofs of
Proposition $\ref{intro:prop1}$ and $\ref{intro:prop2}$. Thus, the tensor product is well defined.
\end{proof}

Consider the operator
\[
  \mcC_{\bv} (\bu)(\bx) 
  := 
  \int_{U} {\rm
      Cov}(\bv(\bx,\omega),\bv(\by,\omega)) 
      \bu(\by)\,\mbox{d} \by
\]
for all $\bu \in L^{2}(U;\R^d)$.  From Lemma 2 in \cite{Harbrecht2016}
operator $\mcC_{\bv}:L^{2}(U;\R^d) \rightarrow L^2(U;\R^d)$ is a
non-negative symmetric, trace class operator. From Theorem 1 in
\cite{Harbrecht2016} there exists an orthonormal set of eigenfunctions
$\{\bphi_k\}_{k \in \bbN}$, where $\bphi_{k} \in L^{2}(U;\R^d)$, and
eigenvalues $\lambda_1 \geq \lambda_2 \geq \dots\geq 0$ such that
$\mcC_{\bv} \bphi_k$ = $\lambda_k \bphi_k$ for all $k \in
\bbN$. Furthermore, for any $\bu \in L^{2}(U;\R^d)$ we have that
$\mcC_{\bv}(\bu) = \sum_{k \in \bbN}
\lambda_{k}(\bu,\bphi_k)_{L^{2}(U;\R^d)} \bphi_{k}$.

\begin{lemma}[Vector field Mercer theorem]
\[
{\Cov}(\bv(\bx,\omega),\bv(\by,\omega)) 
= \sum_{k \in \bbN} 
\lambda_k \bphi_k(\bx)  \bphi_k(\by) \T.
\]  
\label{KL:Mercer}
\end{lemma}
\begin{proof}
%Since $\mcC_{\bv}$ is a trace class operator (e.g. $\sum_{k \in
 %   \bbN} \lambda_k < \infty$), from the argument given in Theorem
%2.3 in \cite{Schwab2006} the series $\sum_{k \in \bbN}
%\lambda_k^{\frac{1}{2}} \bphi_k \otimes s_k$ converges to an element
%$f \in L^{2}(U;\R^d) \otimes L^{2}_{\bbP}(\Omega)$, where $\{s_k\}_{k \in
 % \bbN}$ is a complete orthonormal basis for $L^{2}_{\bbP}(\Omega)$,
%and $C_f = \sum_{k \in \bbN} \lambda_k \bphi_k \otimes \bphi_k$.
%Let $C_{\bv}$ be the correlation function for $\bv(\bx,\omega)
%- \eset{\bv(\bx,\omega)}$
Let
\[
C_{\bv} := \sum_{k \in \bbN} \lambda_k \bphi_k(\bx)  \bphi_k(\by) \T.
\]
Since $\{\bphi_k\}_{k \in \bbN}$ forms a complete orthonormal basis (cob) of $L^{2}(U)$ then from Proposition \ref{intro:prop1} $\{\bphi_k \bphi_l\T\}_{k,l \in \bbN}$ forms a cob in $L^{2}(U \times U;\R^{d \times d})$. Since $\sum_{k \in
    \bbN} \lambda_k < \infty$, we have that 
    $\|C_{\bv}\|_{L^{2}(U \times U;\R^{d \times d})} =  \sum_{k \in \bbN} \lambda_k \|\bphi_k(\bx)  \bphi_k(\by) \T\|_{L^{2}(U \times U;\R^{d \times d})} = 
    \sum_{k\in \bbN} \lambda_k < \infty$ and thus    
    $C_{\bv} \in L^{2}(U \times U;\R^{d \times d})$.
    
%In Theorem 2.3 in \cite{Schwab2006} we have the identity $(\mcC_{\bv} \bu,\bw)_{H} = (C_f, \bu \otimes \bw)_{H \otimes H}$, for all $\bu, \bw \in H$ for any separable Hilbert space.  We choose $H = L^{2}(U;\R^d)$ and from Proposition \ref{intro:prop2} we have that $H \otimes H$ is isomorphic to $L^{2}(U \times U;\R^{d \times d})$.
Similarly to Theorem 2.3 in \cite{Schwab2006} we can show the identity
\[
(\mcC_{\bv} \bu,\bw)_{L^{2}(U;\R^d)} = (C_{\bv}, \bu \bw \T)_{L^2(U \times U;\R^{d \times d})} \forall \bu, \bw \in L^{2}(U;\R^d).
\]
Since $\{\bv_k \bv_l \T\}_{k,l \in \bbN}$ is a complete orthonormal
basis for $L^2(U \times U; \R^{d \times d})$ we conclude that
$\Cov(\cdot,\cdot) = C_{\bv}$ almost everywhere. 
\end{proof}
We can now show the KL expansion for a random vector field.
%From Proposition \ref{intro:prop3} and Theorem 2.5 \cite{Schwab2006} the KL
%expansion for the vector field data is obtained: 
%Conversely, since  from
%Theorem 2.3 \cite{Schwab2006} there exists a one-to-one map between $\bv$ and %$\mcC_{\bv}$
%then from the property that $\Cov(\cdot,\cdot) = C_{\bv}$ the KL
%expansion follows.

\begin{theorem}
  Suppose  $\bv \in L^{2}_{\bbP}(\Omega;L^{2}(U;\R^d))$ then
  \[
  \bv(\bx,\omega) = \eset{\bv(\bx,\omega)} + \sum_{k \in \bbN} \lambda^{\frac{1}{2}}_{k} \bphi_k(\bx) Y_{k}(\omega),
  \]
where $ Y_k(\omega) = \frac{1}{\sqrt{\lambda_k}} \int_U
(\bv(\bx,\omega) - \eset{\bv(\bx,\omega)}) \T \bphi_{k}(\bx)
\,\mbox{\emph{d}} \bx$, $\eset{Y_k Y_l} = \delta_{kl}$ and $\eset{Y_k}
= 0$ for all $k,l \in \bbN$.
\label{appendix:kl}
\end{theorem}
\begin{proof}Without loss of generality (wlog), assume that $\eset{v(\bx,\by)} = 0$. From Proposition \ref{intro:prop3} we have that $\bv \in L^{2}_{\bbP}(\Omega;L^{2}(U;\R^d)) \cong L^{2}(U;\R^d) \otimes L^{2}_{\bbP}(\Omega)$. Now,  since $\{\bphi_k\}_{k \in \bbN}$  forms a cob of $L^{2}(U;\R^d)$ and $\{s_k\}_{k \in \bbN}$ forms a cob in $L^{2}_{\bbP}(\Omega)$ then from Proposition \ref{intro:prop3} we have that  $\{\bphi_k(\bx)s_l(\omega) \T \}_{k,l \in \bbN}$ forms a cob of $L^{2}(U;\R^d) \otimes L^{2}_{\bbP}(\Omega)$ and
\[
\bv(\bx,\omega) = 
\sum_{k,l \in \bbN} \alpha_{k,l} \phi_{k}(\bx) s_l(\omega)
\]
for some coefficients $\alpha_{k,l}$.  Now, taking the sum over $l$ we have that
\[
\bv(\bx,\omega) = 
\sum_{k \in \bbN} \phi_{k}(\bx) Z_k(\omega)
\]
for some  random variables $Z_k(\omega)$.  Note that this is equivalent of projecting $\bv(\bx,\omega)$ into $L^{2}(U;\R^{d})$. Now,
\[
\begin{split}
\Cov(\bv(\bx,\omega),\bv(\by,\omega)) &= 
\eset{\sum_{k \in \bbN} \phi_{k}(\bx) Z_k(\omega) \sum_{l \in \bbN} \phi_{l}(\by) Z_l(\omega)} \\
&=  \sum_{k,l \in \bbN} \eset{Z_k(\omega)Z_l(\omega)} 
\bphi_l(\bx)  \bphi_l(\by) \T.
\end{split}
\]
Alternatively,  from Lemma \ref{KL:Mercer} 
\[
\Cov(\bv(\bx,\omega),\bv(\by,\omega))
      = 
      \sum_{k \in \bbN} \lambda_k 
      \bphi_k(\bx)  \bphi_k(\by) \T.
\]
Since $\{\bphi_k(\bx) \bphi_l(\by)\T\}_{k,l \in \bbN}$ is an orthonormal set in $L^{2}(U \times U;\R^{d \times d})$, then if $k=l$ 
$\eset{Z^{2}_k(\omega)} = \lambda_k$, 
and if $k \neq l$ $\eset{Z_k(\omega)Z_l(\omega)} = 0$. It is not hard to show that $\eset{Z_k} = 0$. Now, let $Y_k := Z_k(\omega)/\sqrt{\lambda_k}$ and the result follows.
\end{proof}

We can now show optimality by using Theorem 2.7 in \cite{Schwab2006}. 
Suppose that $H,S$ are separable  Hilbert spaces with inner products $(\cdot,\cdot)_{H}$ and $(\cdot,\cdot)_{S}$ respectively. Any element $f \in H \otimes S$ can be written as a convergent series of the form $f = \sum_{k \in \bbN} f_k \otimes s_k$.

%with the associated correlation $C_f := \sum_{k \in
%  \bbN} f_k \otimes f_k$.  

To apply Theorem 2.7 in \cite{Schwab2006}, we set $H:=L^{2}(U;\R^d) $ and
$S:=L^{2}_{\bbP}(\Omega)$. However, to apply this result, we need to explicitly construct the tensor product for these spaces. In Proposition \ref{intro:prop3} the tensor products constructions are well defined.

Suppose that $H_M \subset L^{2}(U)$ is a
finite dimensional subspace of $L^{2}(U)$ such that $\dim H_M = M$ and
$P_{H_M \otimes L^{2}_{\bbP}(\Omega)}: L^{2}(U;\R^{d}) \otimes
L^{2}_{\bbP}(\Omega) \rightarrow H_M \otimes L^{2}_{\bbP}(\Omega)$ is
an orthogonal projection operator. Suppose \corb{$\bv \in
  L^{2}(U;\R^{d}) \otimes L^{2}_{\bbP}(\Omega)$}, and  with wlog $\eset{\bv} =
0$. From Theorem \ref{appendix:kl} and Theorem 2.7 in \cite{Schwab2006} we conclude that
\[
\inf_{\begin{array}{c}
H_M \subset L^{2}(U;\R^d) \\
\mbox{dim}\, H_M = M
\end{array}
} \|\bv - P_{H_M \otimes
L^{2}_{\bbP}(\Omega)} \bv \|_{L^{2}_{\bbP}(\Omega) \otimes L^{2}(U)} =
\left( \sum_{k \geq M+1} \lambda_k \right)^{\frac{1}{2}},
\]
where the infimum is achieved \corb{only} when $H_M =
\mbox{span}\{\bphi_1,\dots,\bphi_{M}\}$.

%\bibliographystyle{siamplain}
%\bibliography{juliorefs,KonRefs,SuchiRefs,PontusRefs,citations,multilevel,changedetectionreferences}

%\end{document}

%% if your bibliography is in bibtex format, uncomment commands:
\bibliographystyle{abbrv} % Style BST file (imsart-number.bst or imsart-nameyear.bst)
\bibliography{juliorefs,KonRefs,SuchiRefs,PontusRefs,citations,multilevel,changedetectionreferences,FDA}

@article{Gao2019,
title = {High-dimensional functional time series forecasting: An application to age-specific mortality rates},
journal = {Journal of Multivariate Analysis},
volume = {170},
pages = {232-243},
year = {2019},
issn = {0047-259X},
author = {Yuan Gao and Han Lin Shang and Yanrong Yang},
}

@article{Happ2018,
author = {Clara Happ and Sonja Greven},
title = {Multivariate Functional Principal Component Analysis for Data Observed on Different (Dimensional) Domains},
journal = {Journal of the American Statistical Association},
volume = {113},
number = {522},
pages = {649-659},
year  = {2018},
publisher = {Taylor & Francis}
}

@article{skole1994,
  title={Data on global land-cover change: acquisition, assessment and analysis},
  author={Skole, David L},
  journal={Changes in land use and land cover: a global perspective},
  pages={437--471},
  year={1994},
  publisher={Cambridge University Press Cambridge}
}

@article{chapin2000,
  title={Consequences of changing biodiversity},
  author={Chapin III, F Stuart and Zavaleta, Erika S and Eviner, Valerie T and Naylor, Rosamond L and Vitousek, Peter M and Reynolds, Heather L and Hooper, David U and Lavorel, Sandra and Sala, Osvaldo E and Hobbie, Sarah E and others},
  journal={Nature},
  volume={405},
  number={6783},
  pages={234--242},
  year={2000}
}

@article{sala2000,
  title={Global biodiversity scenarios for the year 2100},
  author={Sala, Osvaldo E and Stuart Chapin, FIII and Armesto, Juan J and Berlow, Eric and Bloomfield, Janine and Dirzo, Rodolfo and Huber-Sanwald, Elisabeth and Huenneke, Laura F and Jackson, Robert B and Kinzig, Ann and others},
  journal={science},
  volume={287},
  number={5459},
  pages={1770--1774},
  year={2000}
}

@article{Skakun2022,
title = {Cloud Mask Intercomparison eXercise (CMIX): An evaluation of cloud masking algorithms for Landsat 8 and Sentinel-2},
journal = {Remote Sensing of Environment},
volume = {274},
pages = {112990},
year = {2022},
issn = {0034-4257},
author = {Sergii Skakun and Jan Wevers and Carsten Brockmann and Georgia Doxani and Matej Aleksandrov and Matej Batič and David Frantz and Ferran Gascon and Luis Gómez-Chova and Olivier Hagolle and Dan López-Puigdollers and Jérôme Louis and Matic Lubej and Gonzalo Mateo-García and Julien Osman and Devis Peressutti and Bringfried Pflug and Jernej Puc and Rudolf Richter and Jean-Claude Roger and Pat Scaramuzza and Eric Vermote and Nejc Vesel and Anže Zupanc and Lojze Žust}
}

@article{barboza2020identifying,
  title={Identifying Human Mobility Patterns in the Rio de Janeiro Metropolitan Area using Call Detail Records},
  author={Barboza, Matheus HC and Alencar, Ricardo de S and Chaves, Julio C and Silva, Moacyr AHB and Orrico, Romulo D and Evsukoff, Alexandre G},
  journal={Transportation Research Record},
  pages={0361198120977655},
  year={2020},
  publisher={SAGE Publications Sage CA: Los Angeles, CA}
}

@article{rotem2015role,
  title={The role of local land-use on the urban heat island effect of Tel Aviv as assessed from satellite remote sensing},
  author={Rotem-Mindali, Orit and Michael, Yaron and Helman, David and Lensky, Itamar M},
  journal={Applied Geography},
  volume={56},
  pages={145--153},
  year={2015},
  publisher={Elsevier}
}

@article{peng2020exploring,
  title={Exploring urban spatial features of COVID-19 transmission in Wuhan based on social media data},
  author={Peng, Zhenghong and Wang, Ru and Liu, Lingbo and Wu, Hao},
  journal={ISPRS International Journal of Geo-Information},
  volume={9},
  number={6},
  pages={402},
  year={2020},
  publisher={Multidisciplinary Digital Publishing Institute}
}

@article{chen2017delineating,
  title={Delineating urban functional areas with building-level social media data: A dynamic time warping (DTW) distance based k-medoids method},
  author={Chen, Yimin and Liu, Xiaoping and Li, Xia and Liu, Xingjian and Yao, Yao and Hu, Guohua and Xu, Xiaocong and Pei, Fengsong},
  journal={Landscape and Urban Planning},
  volume={160},
  pages={48--60},
  year={2017},
  publisher={Elsevier}
}

@article{hasan2013spatiotemporal,
  title={Spatiotemporal patterns of urban human mobility},
  author={Hasan, Samiul and Schneider, Christian M and Ukkusuri, Satish V and Gonz{\'a}lez, Marta C},
  journal={Journal of Statistical Physics},
  volume={151},
  number={1},
  pages={304--318},
  year={2013},
  publisher={Springer}
}

@article{paul2016smartbuddy,
  title={Smartbuddy: defining human behaviors using big data analytics in social internet of things},
  author={Paul, Anand and Ahmad, Awais and Rathore, M Mazhar and Jabbar, Sohail},
  journal={IEEE Wireless communications},
  volume={23},
  number={5},
  pages={68--74},
  year={2016},
  publisher={IEEE}
}

@inproceedings{zhang2019predicting,
  title={Predicting 3d human dynamics from video},
  author={Zhang, Jason Y and Felsen, Panna and Kanazawa, Angjoo and Malik, Jitendra},
  booktitle={Proceedings of the IEEE/CVF International Conference on Computer Vision},
  pages={7114--7123},
  year={2019}
}

@incollection{shaw2018introduction,
  title={Introduction: Human dynamics in perspective},
  author={Shaw, Shih-Lung and Sui, Daniel},
  booktitle={Human dynamics research in smart and connected communities},
  pages={1--11},
  year={2018},
  publisher={Springer}
}

@inproceedings{petrovska2015traffic,
  title={Traffic congestion analysis visualisation tool},
  author={Petrovska, Natasha and Stevanovic, Aleksandar},
  booktitle={2015 IEEE 18th International Conference on Intelligent Transportation Systems},
  pages={1489--1494},
  year={2015},
  organization={IEEE}
}

@article{guo2020gps,
  title={GPS-based citywide traffic congestion forecasting using CNN-RNN and C3D hybrid model},
  author={Guo, Jingqiu and Liu, Yangzexi and Yang, Qingyan and Wang, Yibing and Fang, Shouen},
  journal={Transportmetrica A: Transport Science},
  pages={1--22},
  year={2020},
  publisher={Taylor \& Francis}
}

@article{Arias-Castro2018,
author = {Ery Arias-Castro and Rui M. Castro and Ervin Tánczos and Meng Wang},
title = {Distribution-Free Detection of Structured Anomalies: Permutation and Rank-Based Scans},
journal = {Journal of the American Statistical Association},
volume = {113},
number = {522},
pages = {789-801},
year  = {2018},
publisher = {Taylor & Francis},
doi = {10.1080/01621459.2017.1286240}
}

@article {Arias-Castro2005,
    author = {E. Arias-Castro and D. Donoho and X. Huo},
    title = {Near-Optimal Detection of Geometric Objects by Fast Multiscale Methods},
    journal = {IEEE Transactions on Information Theory},
    volume = {51},
    pages = {2402},
    year = {2005}
    }

@article {Cheung2013,
    author = {Y. T. D.  Cheung and M. J. Spittal and M. K. Williamson and S. J. Tung and J. Pirkis},
    title = {Application of Scan Statistics to Detect Suicide Clusters in Australia},
    journal = {PLoS ONE},
    volume = {8},
    pages = {e54168},
    year = {2013}
    }

@article {Guerriero2009,
    author = {M. Guerriero  and P. Willett  and J. Glaz}, 
    title = {Distributed Target Detection in Sensor Networks Using Scan Statistics},
    journal = {IEEE Transactions on Signal Processing},
    volume = {57},
    pages = {2629},
    year = {2009}
    }

@article {Kulldorff1997,
    author = {M. Kulldorff},
    title = {A Spatial Scan Statistic},
    journal = {Communications in Statistics: Theory and Methods},
    volume = {26},
    pages = {1481},
    year = {1997}
    }

@article{Neill2012,
    author = {D. B. Neill},
    title = {Fast Subset Scan for Spatial Pattern Detection},
    journal = {Journal of the Royal Statistical Society},
    volume = {74},
    pages = {337},
    year = {2012}
    }

@article{Neill2004,
    author = {D. B. Neil and A. W. Moore},
    title = {Rapid Detection of Significant Spatial Clusters},
    journal = {Proceedings of the Tenth ACM SIGKDD International Conference on Knowledge Discovery and Data Mining},
    volume = {},
    pages = {256},
    year = {2004}
    }

@article {Walther2010,
    author = {Walther G.},
    title = {Optimal and Fast Detection of Spatial Clusters With Scan Statistics},
    journal = {The Annals of Statistics},
    volume = {38},
    pages = {1010},
    year = {2010}
    }

@article{Castrillon2016a,
title = "Multi-level restricted maximum likelihood covariance estimation and {Kriging} for large non-gridded spatial datasets ",
journal = "Spatial Statistics ",
volume = "18, Part A",
number = "",
pages = "105 - 124",
year = "2016",
note = "Spatial Statistics Avignon: Emerging Patterns ",
issn = "2211-6753",
doi = "http://dx.doi.org/10.1016/j.spasta.2015.10.006",
url = "http://www.sciencedirect.com/science/article/pii/S2211675315000834",
author = "J. E. Castrill\'{o}n-Cand\'{a}s and M. G. Genton and R. Yokota"
}

@article{Castrillon2025,
      title={Stochastic tensor space feature theory with applications to robust machine learning}, 
      author={Julio Enrique Castrillon-Candas and Kaili Shi and Dingning Liu and Sicheng Yang and Xiaoling Zhang and Mark Kon},
      year={2025},
      note={2110.01729},
      journal={arXiv}
}

@inproceedings{Dasgupta2008,
 author = {Dasgupta, S. and Freund, Y.},
 title = {Random Projection Trees and Low Dimensional Manifolds},
 booktitle = {Proceedings of the Fortieth Annual ACM Symposium on Theory of Computing},
 series = {STOC '08},
 year = {2008},
 isbn = {978-1-60558-047-0},
 location = {Victoria, British Columbia, Canada},
 pages = {537--546},
 numpages = {10},
 url = {http://doi.acm.org/10.1145/1374376.1374452},
 doi = {10.1145/1374376.1374452},
 acmid = {1374452},
 publisher = {ACM},
 address = {New York, NY, USA},
}

@book{Matlab2025b,
year = {2025},
author = {MATLAB},
title = {version 25.2 (R2025b)},
publisher = {The MathWorks Inc.},
address = {Natick, Massachusetts}
}

@article{Harbrecht2016,
author={Harbrecht, H.
and Peters, M.
and Siebenmorgen, M.},
title={Analysis of the domain mapping method for elliptic diffusion problems on random domains},
journal={Numerische Mathematik},
year={2016},
volume={134},
number={4},
pages={823-856},
issn={0945-3245},
doi={10.1007/s00211-016-0791-4},
}

@article{Schwab2006,
author = "C. Schwab and R. A. Todor",
title = "{Karhunen–Lo\`{e}ve} approximation of random fields by generalized fast multipole methods",
journal = "Journal of Computational Physics",
volume = "217",
number = "1",
pages = "100 - 122",
year = "2006",
note = "Uncertainty Quantification in Simulation Science",
issn = "0021-9991",
doi = "https://doi.org/10.1016/j.jcp.2006.01.048",
}

@Book{Light1985,
   title = {Approximation theory in tensor product spaces. },
   year = {1985},
   volume = {1169},
   publisher = {Springer},
   address = {New York},
   series = {Lecture notes in mathematics},
   author = {Light, W.A. and Cheney, E.W.}
}

@ARTICLE{Castrillon2002,
  author={J. E. {Castrill\'{o}n-Cand\'{a}s} and K. {Amaratunga}},
  journal={IEEE Transactions on Signal Processing}, 
  title={Fast estimation of continuous Karhunen-Loeve eigenfunctions using wavelets}, 
  year={2002},
  volume={50},
  number={1},
  pages={78-86},
  doi={10.1109/78.972484}}

@article{Castrillon2022,
title = {Anomaly detection: A functional analysis perspective},
journal = {Journal of Multivariate Analysis},
volume = {189},
pages = {104885},
year = {2022},
issn = {0047-259X},
author = {Julio E. Castrill\'{o}n-Cand\'{a}s and Mark Kon}
}

@book{Reed1972,
author = {Reed, M and Simon, B},
title = {Methods of modern mathematical physics 1: Functional Analysis},
publisher = {Academic Press},
   address = {New York, New York},
year = {1972},
ISBN = {978-0-12-585001-8},
}

@article{Huete2002,
title = {Overview of the radiometric and biophysical performance of the MODIS vegetation indices},
journal = {Remote Sensing of Environment},
volume = {83},
number = {1},
pages = {195-213},
year = {2002},
note = {The Moderate Resolution Imaging Spectroradiometer (MODIS): a new generation of Land Surface Monitoring},
issn = {0034-4257},
author = {A. Huete and K. Didan and T. Miura and E.P. Rodriguez and X. Gao and L.G. Ferreira},
}

@article{Drusch2012,
title = {Sentinel-2: ESA's Optical High-Resolution Mission for GMES Operational Services},
journal = {Remote Sensing of Environment},
volume = {120},
pages = {25-36},
year = {2012},
note = {The Sentinel Missions - New Opportunities for Science},
issn = {0034-4257},
author = {M. Drusch and U. {Del Bello} and S. Carlier and O. Colin and
                  V. Fernandez and F. Gascon and B. Hoersch and
                  C. Isola and P. Laberinti and P. Martimort and
                  A. Meygret and F. Spoto and O. Sy and F. Marchese
                  and P. Bargellini},
}

@article{Zhu2012,
title = {Object-based cloud and cloud shadow detection in Landsat imagery},
journal = {Remote Sensing of Environment},
volume = {118},
pages = {83-94},
year = {2012},
issn = {0034-4257},
author = {Zhe Zhu and Curtis E. Woodcock},
}

@article{Lakhina2004,
author = {Lakhina, Anukool and Crovella, Mark and Diot, Christophe},
title = {Diagnosing Network-Wide Traffic Anomalies},
year = {2004},
issue_date = {October 2004},
publisher = {Association for Computing Machinery},
address = {New York, NY, USA},
volume = {34},
number = {4},
issn = {0146-4833},
journal = {SIGCOMM Comput. Commun. Rev.},
month = {aug},
pages = {219–230},
numpages = {12}
}

@article{Castrillon2026deforest,
      title={deFOREST: Fusing Optical and Radar satellite data for Enhanced Sensing of Tree-loss}, 
      author={Castrillon-Candas, Julio Enrique  and Gu, Hanfeng and Meredith, Caleb and Li, Yulin and Tang, Xiaojing and
      Olofsson, Pontus and Kon, Mark},
      year={2026},
      note={In Press},
      journal ={IEEE Transactions in Geosciences and Remote sensing}
}

@article{Cleveland1979,
author = {William S. Cleveland},
title = {Robust Locally Weighted Regression and Smoothing Scatterplots},
journal = {Journal of the American Statistical Association},
volume = {74},
number = {368},
pages = {829--836},
year = {1979},
publisher = {Taylor \& Francis},
doi = {10.1080/01621459.1979.10481038}
}

@article{Jackson1979,
author = {J. Edward Jackson and Govind S. Mudholkar},
title = {Control Procedures for Residuals Associated With Principal Component Analysis},
journal = {Technometrics},
volume = {21},
number = {3},
pages = {341-349},
year = {1979},
publisher = {Taylor \& Francis},
doi = {10.1080/00401706.1979.10489779},
}

@article{Castrillon2003,
author = {J. E. Castrill\'{o}n-Cand\'{a}s and Kevin Amaratunga},
title = {Spatially Adapted Multiwavelets and Sparse Representation of
                  Integral Equations on General Geometries},
journal = {SIAM Journal on Scientific Computing},
volume = {24},
number = {5},
pages = {1530-1566},
year = {2003},
doi = {10.1137/S1064827501371238},
}

@article{Tausch2001,
author = {Tausch, J. and White, J.},
title = {Multiscale Bases for the Sparse Representation of Boundary Integral Operators on Complex Geometry},
journal = {SIAM Journal on Scientific Computing},
volume = {24},
number = {5},
pages = {1610-1629},
year = {2003},
doi = {10.1137/S1064827500369451},

URL = { 
        https://doi.org/10.1137/S1064827500369451
    
},
eprint = { 
        https://doi.org/10.1137/S1064827500369451
    
}

}

\end{document}